\documentclass[10pt,journal,compsoc]{IEEEtran}

\ifCLASSOPTIONcompsoc
  \usepackage[nocompress]{cite}
\else
  \usepackage{cite}
\fi

\ifCLASSINFOpdf
\else
\fi

\usepackage[dvipsnames,table,xcdraw]{xcolor}
\usepackage[pagebackref=true,breaklinks=true,colorlinks,citecolor=ForestGreen]{hyperref}
\usepackage{ragged2e}

\usepackage{amssymb}
\usepackage{makecell}
\usepackage{multirow}
\usepackage{rotating}
\usepackage{array}
\usepackage{times}
\usepackage{graphicx}
\usepackage{psfrag}
\usepackage{subfigure}
\usepackage{enumitem}
\usepackage{url}
\usepackage{amsmath}
\usepackage{booktabs}
\usepackage{lscape}
\usepackage{verbatim}

\usepackage{bbding}

\usepackage{amsthm}

\usepackage{multirow}
\usepackage{array}
\usepackage{makecell}

\usepackage{overpic}
\usepackage{wrapfig}

\newtheorem{problem}{Problem}
\newtheorem{theorem}{Theorem}[section]
\newtheorem{definition}{Definition}[section]
\newtheorem{proposition}{Proposition}[section]
\newtheorem{lemma}{Lemma}

\begin{document}

\title{Generalized Semantic Contrastive Learning via Embedding Side Information for Few-Shot Object Detection}

\author{Ruoyu Chen,
    Hua Zhang*,
    Jingzhi Li,
    Li Liu*,~\IEEEmembership{Senior Member,~IEEE},
    Zhen Huang, \\
    and Xiaochun Cao*,~\IEEEmembership{Senior Member,~IEEE}
    \thanks{This work was supported by the National Key R\&D Program of China (Grant No.2022ZD0119200), National Natural Science Foundation of China (No. 62372448, 62306308, 62132006, 62025604), and Shenzhen Science and Technology Program (No. KQTD20221101093559018).}
    \IEEEcompsocitemizethanks{
        \IEEEcompsocthanksitem Ruoyu Chen, Hua Zhang, and Jingzhi Li are with the Institute of Information Engineering, Chinese Academy of Sciences, Beijing 100093, China, and also with the School of Cyber Security, University of Chinese Academy of Sciences, Beijing 100049, China \\(Email: \href{mailto:chenruoyu@iie.ac.cn}{chenruoyu@iie.ac.cn}, \href{mailto:zhanghua@iie.ac.cn}{zhanghua@iie.ac.cn}, \href{mailto:lijingzhi@iie.ac.cn}{lijingzhi@iie.ac.cn}).
        \IEEEcompsocthanksitem Li Liu is with the College of Electronic Science and Technology, National University of Defense Technology, Changsha 430074, China \\(Email: \href{mailto:liuli_nudt@nudt.edu.cn}{liuli\_nudt@nudt.edu.cn}).
        \IEEEcompsocthanksitem Zhen Huang is with the College of Computer, National University of Defense Technology, Changsha 430074, China \\(Email: \href{mailto:huangzhen@nudt.edu.cn}{huangzhen@nudt.edu.cn}).
        \IEEEcompsocthanksitem Xiaochun Cao is with the School of Cyber Science and Technology, Shenzhen Campus of Sun Yat-sen University, Shenzhen 518107, China \\(Email: \href{mailto:caoxiaochun@mail.sysu.edu.cn}{caoxiaochun@mail.sysu.edu.cn}). 
        \IEEEcompsocthanksitem * Corresponding authors
    }
}

\IEEEtitleabstractindextext{%
\begin{abstract}
    \justifying
    The objective of few-shot object detection (FSOD) is to detect novel objects with few training samples. The core challenge of this task is how to construct a generalized feature space for novel categories with limited data on the basis of the base category space, which could adapt the learned detection model to unknown scenarios. Most existing fine-tuning-based approaches tackle the challenge via pre-training a feature extractor based on the base categories and then fine-tuning the detector through the novel categories. However, limited by insufficient samples for novel categories, two issues still exist: (1) the features of the novel category are easily implicitly represented by the features of the base category, leading to inseparable classifier boundaries, (2) novel categories with fewer data are not enough to fully represent the distribution, where the model fine-tuning is prone to overfitting. To address these issues, we introduce the side information to alleviate the negative influences derived from the feature space and sample viewpoints and formulate a novel generalized feature representation learning method for FSOD. Specifically, we first utilize embedding side information to construct a knowledge matrix to quantify the semantic relationship between the base and novel categories. Then, to strengthen the discrimination between semantically similar categories, we further develop contextual semantic supervised contrastive learning which embeds side information. Furthermore, to prevent overfitting problems caused by sparse samples, a side-information guided region-aware masked module is introduced to augment the diversity of samples, which finds and abandons biased information that discriminates between similar categories via counterfactual explanation, and refines the discriminative representation space further. Finally, we theoretically analyze the generalization bound for introducing our proposed module and demonstrate that our proposed model can effectively reduce the upper bound of the generalization error. Extensive experiments using ResNet and ViT backbones on PASCAL VOC, MS COCO, LVIS V1, FSOD-1K, and FSVOD-500 benchmarks demonstrate that our model outperforms the previous state-of-the-art methods, significantly improving the ability of FSOD in most shots/splits. The code is released at \url{https://github.com/RuoyuChen10/CCL-FSOD}.
\end{abstract}

}

\maketitle

\IEEEdisplaynontitleabstractindextext

%
\IEEEpeerreviewmaketitle

\IEEEraisesectionheading{\section{Introduction}\label{sec:introduction}}

%
%
%
%
\IEEEPARstart{O}{bject} detection methods \cite{Li_2022_CVPR, ren2016faster,liu2020deep,han2021context} have achieved substantial progress in recent years. However, their impressive performance is heavily dependent on a large scale of annotated training instances, which limits the scalability of the object detector in data scarcity scenarios. To solve this limitation, few-shot object detection (FSOD) \cite{vu2025multi,zhang2025learning,majee2024smile,han2024few,fan2024fsodv2,wang2024snida} is proposed to detect the novel categories with few labeled instances.

There exist various types of approaches to tackle the problem, including meta-learning \cite{han2024few,fan2024fsodv2,xiao2022few, han2022meta, zhang2022meta}, and transfer learning \cite{zhang2025learning,majee2024smile,wang2020frustratingly, sun2021fsce, qiao2021defrcn, lu2022decoupled}. One of the dominating paradigms for FSOD is the fine-tuning-based method \cite{zhang2025learning,majee2024smile,zhang2022kernelized, qiao2021defrcn}, where the model is first trained on all base categories with abundant instances to learn a feature extractor, and then the novel categories are fine-tuned via freezing the feature extractor.

\begin{figure*}[!t]
    \centering
    \setlength{\abovecaptionskip}{0.cm}
    \includegraphics[width = \textwidth]{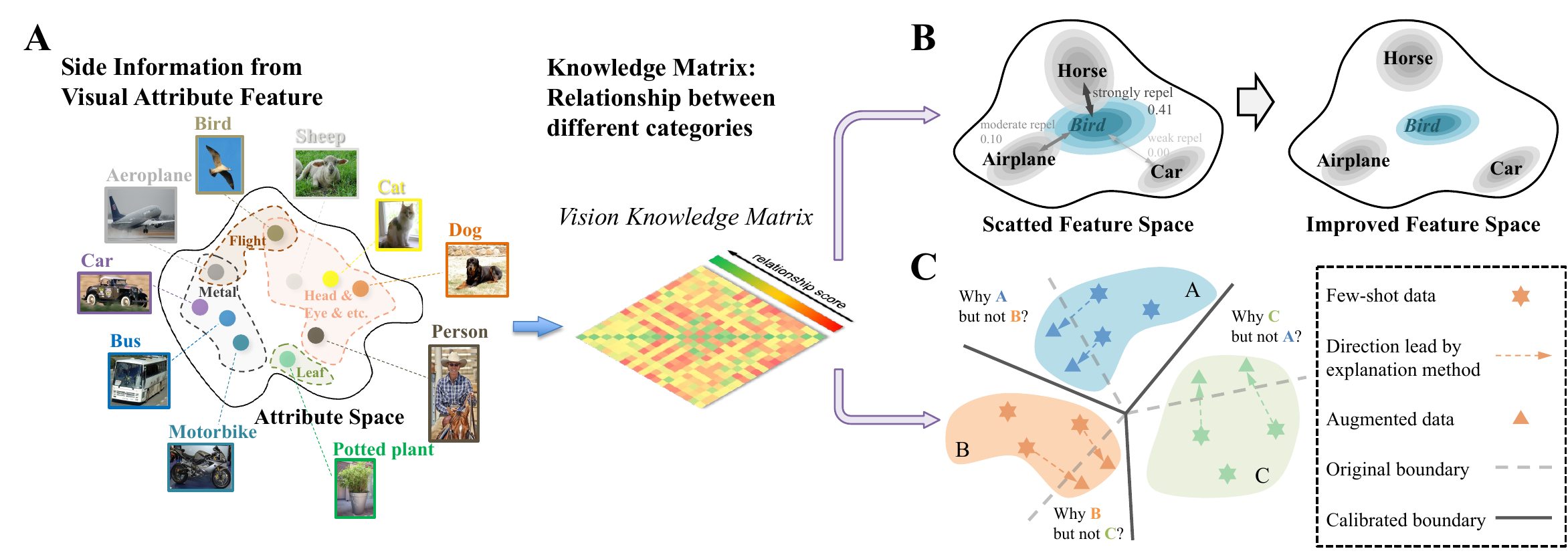}
    \caption{
        \textbf{A.} Motivation of our method. 
        Traditional methods on FSOD only consider the feature representation between distinct categories, which would make the detection model sensitive to the training data distribution. Our model measures the differences between categories with the visual attributes (e.g., head, eye), which could learn the generalizable and discriminative feature representations.
        \textbf{B.} In the fine-tuning stage, the novel category may implicitly utilize the features of multiple base categories for representation, leading to a scatted feature space. Based on the knowledge matrix, contextual semantic supervised contrastive learning is developed to strengthen the space discrimination between semantically similar categories.
        \textbf{C.} Due to the scarcity of few-shot data, the distribution of the novel category cannot be fully represented, resulting in data bias and overfitting. We use the counterfactual explanation method and the masking mechanism to augment the few-shot data so that the mined sample features are closer to the decision boundary and jointly trained to improve the generalization of the model.}
    \label{Knowledge_relationship}
    \vspace{-12 pt}
\end{figure*}

Despite the evident progress of the state-of-the-art models, there still exist several challenges in developing the FSOD models: First, in the fine-tuning stage of the existing method, the feature space is determined by the base categories, the features of the novel category are easily implicitly represented by the features of the base category, leading to biased and inseparable classifier boundaries~\cite{majee2024smile,cao2021few, sun2021fsce}.
The fewer samples of novel categories, the more biased and less uniformly distributed in the feature space.
Second, since the novel category has limited samples, it is insufficient to fully represent the distribution of the novel category. As a result, when iteratively trained on these few-shot data, the model tends to overfit~\cite{yang2021bridging}.
The above issues can be alleviated by introducing prior knowledge~\cite{zhang2021prototype}. However, current few-shot object detectors are usually driven by the visual appearance of the data, where the models would learn the shortcut features via only comparing the visual differences between the novel and base categories \cite{sun2021fsce, lu2022decoupled, xiao2022few}. To solve these limitations, we introduce a useful clue, i.e., the side information, which can measure semantic relationships between categories via exploiting inherent semantic information of the category, e.g., visual attributes and textual description.
Side information is data orthogonal to the input space and output space of the model, and it's helpful for model learning \cite{jonschkowski2015patterns}, which has been successfully applied in many fields, e.g., few-shot classification~\cite{zhang2021prototype, xu2022attribute} and zero-shot object detection~\cite{mao2020zero, rahman2020improved, yan2022semantics}. However, existing work simply transfers the semantic alignment method to the FSOD task~\cite{zhu2021semantic}, which makes it extremely easy to overfit in the few-shot case~\cite{han2022multimodal}. Thus, how to make full use of reasonable side information adapting for the FSOD task needs to be further explored.

To that end, in this paper, we propose a novel generalized feature representation learning method for FSOD, which introduces embedding side information mitigating the feature and sample bias via a few object instances of novel categories. Inspired by that humans can quickly associate distinct categories with semantic attributes, we define the semantic similarity between objects based on their visual attribute representations.
Specifically, we first select representative visual attributes (e.g., head, eye, metal) based on category information and employ visual attributes to obtain semantic representation for each category.
Then, we construct the semantic similarity between different categories, which is represented by a knowledge matrix, as shown in Fig.~\ref{Knowledge_relationship} A. Upon the knowledge matrix, we develop a contextual semantic supervised contrastive learning (CCL) model, which fuses the semantic diversity similarities between categories. In that way, the model will be guided to distinguish semantically similar novel categories from base categories, improving the separability of the feature space, as shown in Fig.~\ref{Knowledge_relationship} B.
Furthermore, we propose a counterfactual data augmentation method to prevent the model from overfitting with few-shot data. By introducing a counterfactual saliency map, it explains why is the ground truth category but not the counter category, where the counter category is selected by the knowledge matrix, i.e., the most semantically similar category. Through the explainable saliency map, the current discriminative region of the model can be found and masked, and the masked feature is joint training to improve the generalization ability of the FSOD model, and will enable CCL to refine the discriminative representation space further, enhancing distinction between similar base and novel categories, as shown in Fig.~\ref{Knowledge_relationship} C.
We also provide a theoretical analysis of the proposed CCL module and counterfactual data augmentation methods on the generalization error bounds and theoretically prove that our method can reduce model generalization errors, providing a theoretical guarantee for the good performance and robustness of these methods.
Extensive experiments conducted on PASCAL VOC~\cite{everingham2010pascal}, MS COCO~\cite{lin2014microsoft}, LVIS V1~\cite{gupta2019lvis}, FSOD-1K~\cite{fan2020few}, and FSVOD-500~\cite{fan2022few_video} benchmarks demonstrate that our model achieves SOTA performance, significantly improving the ability of FSOD in most of the shots/splits. To sum up, our contributions can be summarized as follows:
\begin{itemize}
    \item [1)]
    We introduce prior visual side information into the FSOD task, which can learn the generalized and discriminative representation for novel categories.
    \item [2)]
    We propose a contextual semantic supervised contrastive learning method that encodes the semantic relationship between base and novel categories. The detector can be guided to strengthen the discrimination between semantically similar categories and improve the separability of the feature space.
    \item [3)]
    We propose a side-information guided counterfactual data augmentation mechanism. By utilizing the interpretable saliency map to discover the discriminative region of the model and masking it, the model is forced to learn from other regional features, which prevents overfitting and further refines the discriminative space between similar base and novel categories.
    \item [4)]
    Our model has achieved state-of-the-art performance on the PASCAL VOC, MS COCO, LVIS V1, FSOD-1K, and FSVOD-500 benchmarks, which demonstrates the effectiveness of the proposed framework. We also theoretically prove that our method can reduce model generalization errors.
\end{itemize}

The rest of the paper is organized as follows. We first review related works in Sec.~\ref{related_work}. In Sec.~\ref{sec:method}, we elaborate our method and design the FSOD model through theoretical analysis. The experimental results and visualization analysis are presented in Sec.~\ref{experiment}, which demonstrate the effectiveness of the proposed FSOD method. Finally, we conclude the paper in Sec.~\ref{conclusion}.

\section{Related Work} \label{related_work}

\subsection{Few-Shot Object Detection}

Few-shot object detection (FSOD)~\cite{antonelli2022few, fan2020few, perez2020incremental,bar2022detreg} not only needs to learn the categories information of novel objects from a few instances but also needs to locate the position in the image, which is a difficult task compared with few-shot classification. Existing FSOD methods can be categorized as meta-learning-based \cite{du2023adaptive,fan2024fsodv2,han2024few} and transfer-learning-based \cite{wang2020frustratingly, sun2021fsce, qiao2021defrcn, lu2022decoupled,guirguis2023niff}. Meta-learning-based methods typically train a category-independent detector that facilitates inference using a small number of novel category samples, all without the need for retraining on a novel set~\cite{fan2020few,han2022few,han2022meta,zhang2023detect,han2024few}. Transfer-learning-based methods, primarily through fine-tuning the novel set, demonstrate significant potential in FSOD~\cite{qiao2021defrcn,wu2022multi,xu2023generating,guirguis2023niff}. Meta-learning-based methods also show promise for further improvements with additional fine-tuning~\cite{han2022meta,han2022few}.
However, the fine-tuning process is prone to overfitting, and difficult to construct realistic relationships between categories through learned representations. Furthermore, most of these methods would learn the shortcut features by only comparing the visual differences between the novel and base categories, without explicitly considering the side information.

Several studies employ side information to enhance the performance of few-shot object detection, broadly categorizing these efforts into the following groups. \textbf{Alignment-based:} Several methods maintain semantic consistency between visual features and text throughout the model learning process by incorporating text as side information.
SRR-FSD \cite{zhu2021semantic} uses word embeddings (Word2Vec~\cite{mikolov2013distributed}) as side information to align visual features with word embeddings. KD~\cite{pei2022few} introduces Bag-of-Visual-Words (PPC~\cite{xie2021propagate}) and employs knowledge distillation to preserve the consistency between text category relationships and visual category relationships. MM-FSOD~\cite{han2022multimodal} integrates the language features derived from CLIP~\cite{radford2021learning} with FSOD's visual features and applies prompt tuning to the language encoder to enhance alignment.
However, text-based semantic spaces usually possess inherent noise~\cite{rahman2020improved}, which may affect reasoning the visual relationships between categories. Furthermore, for novel categories, the iteration of visual-text alignment is limited, thus it is difficult to learn discriminative representations between the base and novel categories.
\textbf{Association-based:}
FADI~\cite{cao2021few} identifies the base category most similar to each novel category using WordNet~\cite{miller1995wordnet} and Lin Similarity~\cite{lin1998information}, establishes a one-to-one association between them, and subsequently implements intra-class discrimination. However, this approach encounters limitations when multiple novel categories heavily rely on the same base class. Moreover, it is not suitable for fine-tuning datasets composed solely of novel categories.
\textbf{Calibration-based:}
UA-RPN~\cite{fan2022few} leverages the ImageNet classification dataset~\cite{deng2009imagenet} to de-bias the detector. It effectively solves the problem of reduced novel class detection capability during the fine-tuning stage, arising from the absence of novel class annotations in the base training stage. However, incorporating additional datasets incurs higher training costs.
\textbf{Generation-based:}
Norm-VAE~\cite{xu2023generating} proposes a VAE-based feature generation model. By training this VAE model on the base set, conditioned on the Word2Vec semantic information of each category, and then generating features for novel categories. However, the domain gap between text and visual spaces, coupled with biases in the distributions of novel features, may lead to deviations in the generated features for novel classes from their actual features.
SNIDA~\cite{wang2024snida} employs saliency detection to extract and fuse foreground objects into diverse backgrounds. It then utilizes CLIP’s text semantics to guide a sparse encoder in reconstructing randomly masked images, enhancing sample diversity. However, this data generation method is cumbersome, and the data augmentation process lacks interpretability.
\textbf{Metric-based:}
DeFRCN~\cite{qiao2021defrcn} employs a model pre-trained on ImageNet to create class prototypes. These prototypes are then compared to the features extracted by the detector, which are integrated into the class prediction scores using a weighting method during inference. However, there may be a domain gap between the features of offline prototypes and the detector model.
\textbf{Additional pretrained model-based:} FM-FSOD~\cite{han2024few} leveraged a pre-trained large language model for few-shot object detection, carefully designing language instructions to prompt the LLM to classify each proposal in context. However, it can be sensitive to parameters, especially when data is scarce.

In this paper, we develop a novel few-shot object detector, which first learns the debiased feature representation by introducing object prototypes, and then fuses the diverse semantic similarities between categories based on visual attribute side information to achieve robust and discriminative representations. During training, the detector strengthens the discrimination among semantically similar categories. We also propose a counterfactual data augmentation method to mitigate over-fitting problems in the discrimination process. The counterfactual saliency map can identify and mask the current discriminative region of the model, guiding it to learn to discriminate novel categories using other regional features, thus preventing potential bias in the model.

Our method has the following advantages compared with other methods that introduce side information. Unlike the alignment-based methods, our approach concentrates on distinguishing features between categories that are easily confused. This strategy circumvents the issue of inference errors stemming from the domain gap between visual and semantic spaces. In contrast with the association-based method, our approach overcomes the limitation of being unable to effectively manage multiple novel categories closely related to a single base category. Furthermore, it is applicable for fine-tuning scenarios that involve exclusively novel categories. Distinct from calibration-based methods, our approach depends solely on the relationship between categories, eliminating the need for additional datasets to facilitate calibration. Compared with generation-based methods, our approach interprets the learning bias present in the current training process and leverages it to generate reliable data that correct model bias. Differing from metric-based methods, our approach is free from the domain gap issue and eliminates the necessity for side information in the reasoning process.

\subsection{Contrastive Learning}

Contrastive learning has been effectively studied in self-supervised tasks, and its core idea is to attract positive samples and repel negative samples \cite{khosla2020supervised, chen2020simple, he2020momentum, li2022twin}. Some supervised contrastive learning methods select positive and negative samples by labels \cite{khosla2020supervised}. Due to the richness of the samples, they achieved good feature representation. In some unsupervised contrastive learning, the discovered semantic structures are encoded into the learned embedding space through prototypes derived from the clustering method~\cite{li2022twin}. Since the labels are known in the few-shot problem, we believe that clustering each class feature leads to a more representative prototype. In addition, how to select negative samples is also an important problem in contrastive learning. 
Li \textit{et al.} \cite{li2022twin} generate pseudo-labels for unlabeled data by clustering and select negative samples by clustering confidence. However, these methods are mainly for unsupervised tasks.

We propose a novel contrastive learning with side information for FSOD, which first discovers a set of discriminative prototypes for each category. After that, we introduce a novel contextual semantic supervised contrastive learning, which fuses the semantic diversity similarities between categories to strengthen the space discrimination between semantically similar categories.

\begin{figure*}[]
    \centering
    \setlength{\abovecaptionskip}{0.cm}
    \includegraphics[width = \textwidth]{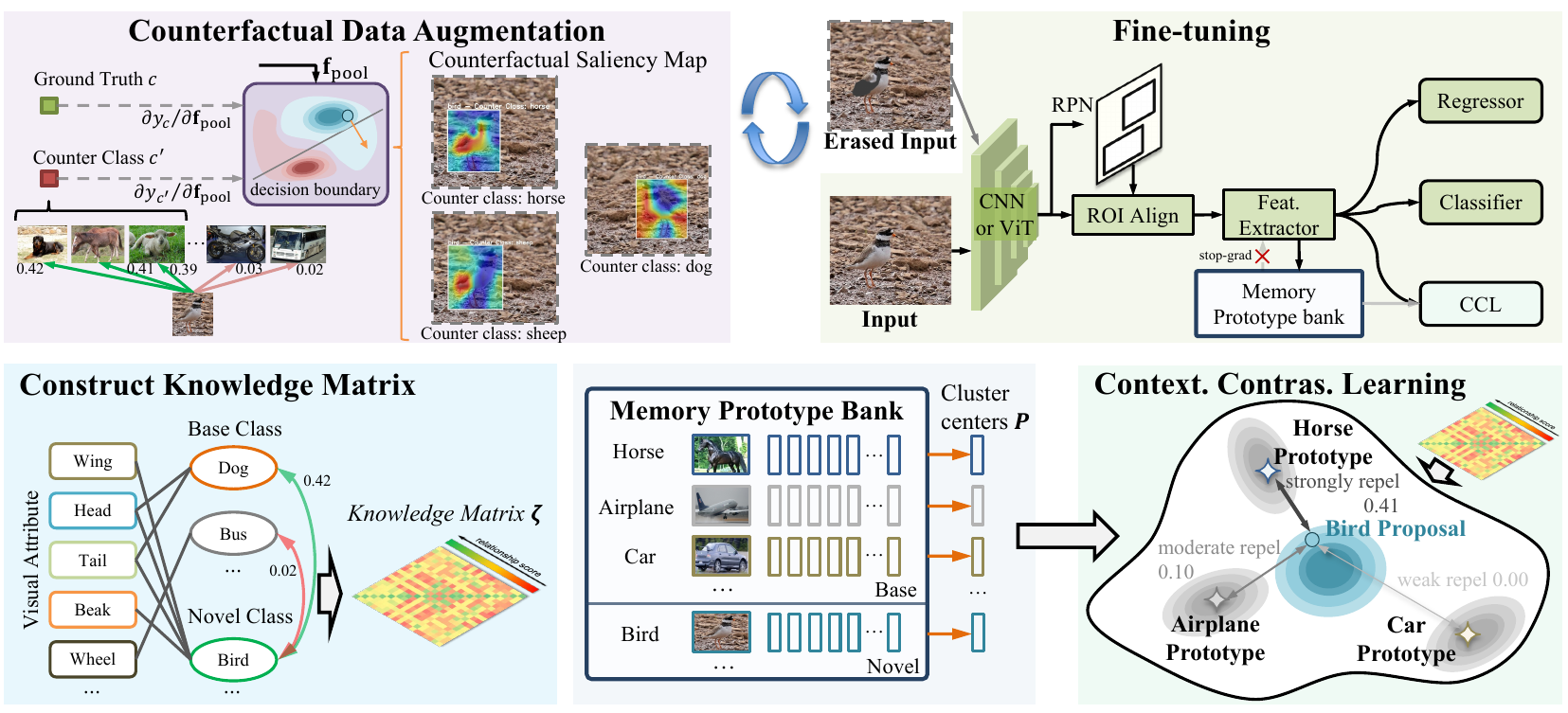}
    \caption{
        An overview of our few-shot object detection fine-tuning method. We first measure the similarity between the base category and the novel category using visual attributes, and represent it by constructing a knowledge matrix.
        During fine-tuning, the memory prototype bank continuously stores the complete features of all the categories. The proposed model leverages the Contextual Semantic Supervised Contrastive Learning (CCL) module and a knowledge matrix to learn generalized representations and improve discriminativeness. Specifically, the CCL module strengthens the distinction between proposal features and specific prototype categories, while the knowledge matrix enables the model to incorporate semantic relations between categories into its representations. The input image has a certain probability of data augmentation. Partially semantically similar counter categories are selected for the current category via the knowledge matrix, and saliency maps of region images are computed via counterfactual explanation. Under a threshold, the original image is erased according to the saliency map. The erased images aid in training the detector to improve its generalization capability and reduce the learning bias. The memory prototype bank is not updated with erased features.
    }
    \label{Framework}
    \vspace{-10pt}
\end{figure*}

\subsection{Object Recognition with Side Information}

Side information is data orthogonal to the input space and output space of the model, but it's helpful for model learning \cite{jonschkowski2015patterns}. Introducing side information in visual object recognition tasks has been well studied, such as few-shot classification \cite{zhang2021prototype, xu2022attribute}, zero-shot classification \cite{wan2021visual, han2022semantic}, and zero-shot object detection \cite{mao2020zero, rahman2020improved, yan2022semantics}. The side information is usually the semantic representation of objects in terms of visual attributes \cite{zhang2021prototype, zhu2019zero, mao2020zero} or word embeddings \cite{zheng2021visual}, 
word embeddings usually have unavoidable noise, resulting in often unstable visual-semantic relationships \cite{rahman2020improved}.
Visual attributes are visual components shared across categories \cite{feris2017visual}, and some works attempt to introduce visual attributes as side information. Some low-shot learning works extract attribute representations of objects through pre-trained attribute classifiers, and then infer categories based on the attribute representations \cite{lampert2013attribute, jayaraman2014zero}. 
In this paper, we introduce visual attributes as side information for the FSOD task. We obtain prior relationships between categories through visual attributes and strengthen the model's discrimination of similar objects during training. To the best of our knowledge, little work has been done to introduce semantic information in the FSOD task, and we believe our method will suggest a new and interesting direction for FSOD.

\subsection{Visual Counterfactual Explanation}

Visual counterfactual explanation explains why images belong to category A but not category B~\cite{goyal2019counterfactual}. Some methods explain the model by finding the smallest perturbation that changes the model's decision boundary~\cite{goyal2019counterfactual, dhurandhar2018explanations,chen2023less,chen2024interpreting}. However, such methods are time-consuming, and the generated perturbations are outside the boundaries of the natural image space. Some methods explain models through counterfactual attribution maps~\cite{wang2020scout, chen2022sim2word}, however, they all require additional uncertainty or attribute encoders, which can introduce biases in the estimation of novel categories in few-shot data.
In this paper, we propose a counterfactual explanation-based data augmentation method for object detection models. Our method can not only be used to explain the model but also can be used to improve the performance of FSOD. We use side information to select semantically similar counter classes, the current discriminative region of the model is discovered in the form of a saliency map and masked, guiding the model to learn discriminating novel categories through other regional features, improving the generalization of the FSOD model.

\subsection{Explanation-guided Data Augmentation}

Geirhos \textit{et al.}~\cite{geirhos2020shortcut} introduced the concept of shortcut learning, highlighting that machines may rely on shortcuts rather than capturing the underlying essence. This tendency is largely attributed to biases in data distribution and other factors. Enhancing model performance through explanation-guided data augmentation offers a way to mitigate shortcut learning without requiring additional annotation~\cite{gao2024going}. Ismail \textit{et al.} \cite{ismail2021improving} augments data that emphasizes important regions in the attribution map of the ground truth category while suppressing irrelevant features during model training. In contrast, Xiao \textit{et al.} \cite{xiao2023masked} aim to fine-tune the model by masking out highlighted regions in the attribution map of the ground truth category, encouraging it to learn a broader range of features. These methods are all aimed at conventional training task scenarios~\cite{ismail2021improving,xiao2023masked}. In this paper, we address the potential overfitting problem among similar categories in few-shot object detection. To mitigate this issue, we design a counterfactual explanation to identify fine-grained dependency regions that distinguish specific category pairs. Utilizing a masking mechanism, we aim to prevent the model from relying on specific features, encouraging it to learn more generalized representations. This, in turn, enhances both model performance and generalization ability.

\section{Method}
\label{sec:method}

In this section, we provide an overview of the proposed method for generalized few-shot object detection. We begin by introducing the problem setting and traditional two-stage fine-tuning framework (Sec.~\ref{problem_def}). Next, we introduce the proposed contextual semantic supervised contrastive learning method with theoretical analysis (Sec.~\ref{thero_analysis}). Then, we discuss some methods for constructing semantic relations between categories via side information (Sec.~\ref{ex_know}). After that, we introduce how to estimate and update prototypes (Sec.~\ref{bank}) and incorporate semantic relations with contrastive learning into FSOD detectors (Sec.~\ref{KGC}). Additionally, we introduce a counterfactual data augmentation method (Sec.~\ref{counterfactual}). Finally, the overall learning objective of our model is introduced (Sec.~\ref{object}).
Fig. \ref{Framework} shows an overview of our model.

\subsection{Preliminaries}
\label{problem_def}

A classical method of FSOD is the two-stage training strategy \cite{wang2020frustratingly, sun2021fsce}. Specifically, it first trains the backbone with abundant labeled instances of base set $\mathcal{D}_{{\rm base}}$, which consists of the base categories $\mathcal{C}_{{\rm base}}$. The goal of this step is to learn a series of transferable features from the base set. Then, the model is adapted to the novel set $\mathcal{D}_{{\rm novel}}$, which consists of novel categories $\mathcal{C}_{{\rm novel}}$ with very few annotated instances, where $\mathcal{C}_{{\rm base}} \cap \mathcal{C}_{{\rm novel}} = \varnothing$. Moreover, a few-shot object detector is trained on $\mathcal{D}_{{\rm base}} \cup \mathcal{D}_{{\rm novel}}$, where there are $k$ annotated instances for each category in $\mathcal{C}_{{\rm base}} \cup \mathcal{C}_{{\rm novel}}$ named $k$-shot detection. This is also called generalized few-shot object detection.
To prevent over-fitting in the fine-tuning stage, there would be components of the model that will be frozen in some modules of the framework. E.g., TFA \cite{wang2020frustratingly} freezes all parameters except regression and classification heads, and FSCE \cite{sun2021fsce} only freezes the parameters of the backbone and RoI pooling layer.

\subsection{Contextual Semantic Supervised Contrastive Learning}\label{thero_analysis}

We propose a simple and effective contrastive learning method for few-shot object detection. In detail, we build relationships between categories with side information for weighing negative samples in contrastive learning.


\subsubsection{Revisiting FSCE}

Given a novel dataset $\mathcal{D}_{\text{novel}}=\{(\mathcal{I}, \mathbf{c}, \mathbf{x}) \}$, where $\mathbf{c}$ and $\mathbf{x}$ represent the object category and location in the image, respectively. We desire that the detector can learn generalized and discriminative feature representations, $\phi_{\theta} : \mathcal{I} \rightarrow \mathbf{f}$, using mapping parameters $\theta$. In Faster R-CNN, $\phi$ is mainly composed of FPN, backbone, RPN, and ROI pooling layer, and $\theta$ is the parameter of $\phi$.
FSCE \cite{sun2021fsce} is the first method to introduce contrastive learning into the fine-tuning stage for few-shot object detection. Specifically, FSCE includes one more contrastive loss function than the standard Faster R-CNN. FSCE utilizes a contrastive head $h$ to encode RoI
feature of region proposals $\mathbf{f}$ to a feature space, $h : \mathbf{f} \rightarrow \boldsymbol{F} \in \mathbb{R}^{d}$, with $C$ categories.
Through CPE (Contrastive Proposal Encoding) loss function, make the features of the same category are compacted, and the features of different categories are contrasted, and it is wished to address the following problem.

\begin{problem}\label{cpe}
    Find the parameters $\theta$ of the mapping function $\phi_{\theta} : \mathcal{I} \rightarrow \mathbf{f}$, using a novel set $\mathcal{D}_{\text{novel}}$, to make the feature representations of the same category more similar, and the feature representations of different categories are farther apart. Such that, for any feature proposals $\mathbf{f}$ and a contrastive head, $h : \mathbf{f} \rightarrow \boldsymbol{F} \in \mathbb{R}^{d}$,
    \begin{equation}
        \mathcal{L}_{\boldsymbol{F}_i} = \sum_{j=1, j\ne i}^{N}\mathbb{I}_{y_i^f=y_j^f} \log \! \left ( \frac{\exp{(\boldsymbol{F}_i \cdot \boldsymbol{F}_j}/ \tau)}{\sum_{k=1,k \ne i}^{N} \exp{(\boldsymbol{F}_i \cdot \boldsymbol{F}_k}/ \tau)} \! \right ),
    \end{equation}
    \begin{equation}
        \min_{\theta} \mathcal{L}_{CPE}, \quad \mathcal{L}_{CPE} = -\frac{1}{N}\sum_{i=1}^{N} u_i  \frac{\mathcal{L}_{\boldsymbol{F}_i}}{N_{y_i^f}-1},
    \end{equation}
where $y_i^f$ and $u_i$ are the category label and IoU score of feature $\boldsymbol{F}_i$, $N_{y_i^f}$ denotes the number of region proposals with the same category label as $y^f_i$, and $\tau$ is the hyper-parameter temperature.
\end{problem}

Some issues arise when using the formula in Problem \ref{cpe} based on few-shot object detection.

\textbf{Issue of feature representation:} Since the proposals of FSCE are all regions predicted by the network, there may be cases where the proposals do not contain complete features, which may lead to biased learning. Although FSCE introduces IoU scores $u$ for weighting positive samples, it is still possible for the model to learn biased features. Therefore, it is very important to choose an unbiased prototype for contrastive learning.

\textbf{Issue of negative pairs weight:} In FSCE, the weight of the CPE loss generated by each negative pair to the model is the same. In fact, since a specific novel category is quite different from some base categories (e.g., motorcycles and cats), it is difficult to confuse them in the fine-tuning stage. We try to leverage the side information to enhance the discrimination between similar base categories and novel categories in the fine-tuning stage.

\textbf{Issue of contrasted category limits:} Since FSCE performs contrastive learning on the proposals of each category, the total number of categories in one iteration of the fine-tuning stage is limited by the batch size. This may result in some similar base categories and novel categories not both appearing in an iteration. We try to build a representative prototype for each category in the fine-tuning stage, ensuring that prototypes of all categories are accessible at each iteration, and strengthening the distinction between the current proposal features and all categories.

\subsubsection{Learning with Side Information}

In order to solve the above issues, we propose a novel prototypical-based contrastive learning method. We utilize a complete prototype feature representation for each object category, denoted as $\mathcal{P} = \{\boldsymbol{P}\}$, with $C$ categories. We learn by contrasting the feature $\boldsymbol{F} \in \mathbb{R}^{d}$ with the unbiased prototype feature $\boldsymbol{P} \in \mathbb{R}^{d}$. In this setup, we can address the following problem.

\begin{problem}\label{p2}
    Find the parameters $\theta$ of the mapping function $\phi_{\theta} : \mathcal{I} \rightarrow \mathbf{f}$, using a novel set $\mathcal{D}_{\text{novel}}$. Given each category prototype feature representation $\mathcal{P} = \{\boldsymbol{P}\}$,
    \begin{equation}
        \min_{\theta} {-\sum_{i=1}^{C} \sum_{j=1}^{N} \mathbb{I}_{y_i^p=y_j^f} \log \left( \frac{\exp{(\boldsymbol{P}_i \cdot \boldsymbol{F}_j / \tau)}}{\sum_{k=1}^{N} \exp{(\boldsymbol{P}_i \cdot \boldsymbol{F}_k / \tau)}} \right)},
    \end{equation}
    where $y^{p}$ is the category label of prototype feature $\boldsymbol{P}$, $y^{f}$ is the category label of feature $\boldsymbol{F}$. This first address the concerns about feature representation and contrasted category limits.
\end{problem}

Next, we consider how to weigh negative samples. We consider the weight of negative samples from the perspective of category relationships. Let $\boldsymbol{\zeta} \in \mathbb{R}^{C \times C}$ be a knowledge matrix constructed from side information, which is used to represent the visual similarity between categories. We proposed contextual supervised contrastive learning with side information, which can strengthen the model's discrimination of similar categories.

\begin{definition}[Contextual semantic supervised contrastive learning with side information]
    Given prototype representations $\mathcal{P} = \{\mathbf{P}\}$ of each category and a prior knowledge matrix $\boldsymbol{\zeta} \in \mathbb{R}^{C \times C}$. The contextual semantic supervised contrastive learning is,
    \begin{equation}
        \mathcal{L} = -\sum_{i=1}^{C} \sum_{j=1}^{N} \mathbb{I}_{y_i^p=y_j^f} \log \left( \frac{\exp{(\boldsymbol{P}_i \cdot \boldsymbol{F}_j / \tau)}}{\sum_{k=1}^{N} \exp{(\boldsymbol{\zeta}_{y_i^p, y_k^f}  \boldsymbol{P}_i \cdot \boldsymbol{F}_k / \tau)}} \right),
    \end{equation}
    when the categories $y^p$ and $y^f$ of  $\boldsymbol{P}$ and $\boldsymbol{F}$ are visually similar, the value of the corresponding position of the knowledge matrix $\boldsymbol{\zeta}_{y^p, y^f}$ is larger. This can strengthen the weight of negative samples between similar categories.
\end{definition}

Now, we address the issue of weighing negative samples in Problem \ref{p2}. Next, we describe in Sec.~\ref{ex_know} how to build category relations as a prior knowledge matrix $\boldsymbol{\zeta}$ for negative sample weighing. Then, in Sec.~\ref{bank}, we introduce how to estimate and update the prototype representations $\mathcal{P} = \{\boldsymbol{P}\}$ of categories in our FSOD model. Finally, in Sec.~\ref{KGC}, we describe the contextual semantic supervised contrastive learning branch in FSOD.

\begin{figure}[!t]
    \centering
    \setlength{\abovecaptionskip}{0.cm}
    \includegraphics[width = 0.48 \textwidth]{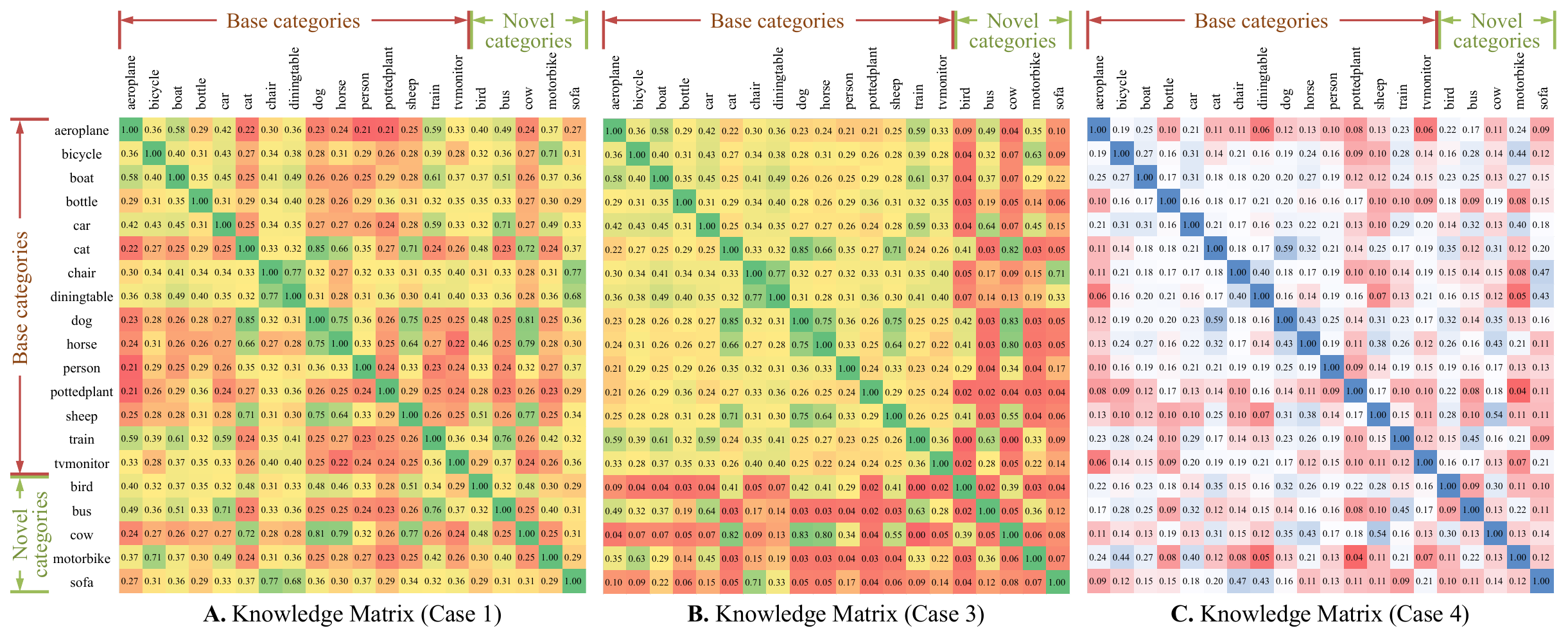}
    \caption{Knowledge matrix for different cases. Relationships between categories were built through semantic information from different side information.}
    \label{knowledge}
\end{figure}

\subsection{Build Semantic Relationship through Side Information}\label{ex_know}

During the FSOD training process, it is easy to confuse similar base categories with novel categories. We expect to know which base categories are likely to be more similar to the novel categories through side information, so as to strengthen their distinction during FSOD training.
To capture the semantic relationship between the base and novel categories, we build a knowledge matrix to measure their visual semantic similarities.
\begin{definition}[Knowledge Matrix] For $C$ categories, the knowledge matrix is denoted as $\boldsymbol{\zeta} \in \mathbb{R}^{C \times C}$. For any categories $c_1$ and $c_2$, where $0<c1\le C, 0<c_2 \le C$, the value of $\boldsymbol{\zeta}_{(c_1,c_2)}$ measures the semantic similarity of two categories, and the value range is [0,1]. The knowledge matrix $\boldsymbol{\zeta}$ is a symmetric matrix, i.e., $\boldsymbol{\zeta}_{(c_1,c_2)} = \boldsymbol{\zeta}_{(c_2,c_1)}$. The value of the main diagonal of the knowledge matrix is $1$, i.e., $\boldsymbol{\zeta}_{(c_1,c_1)} = 1$.
\end{definition}
Through visual attributes or text embeddings, we can obtain the prior relationship between the novel categories and base categories. We discuss several ways to construct category relationships through visual attributes or text embeddings.

\subsubsection{Visual Attributes}
Attributes can be represented as a set of distinctive visual features of object components~\cite{banik2018multi}, which have been effectively applied across various domains. To achieve the visual attribute representation, we can consider a traditional visual attribute prediction model, which is trained with the annotated instances at the image level~\cite{farhadi2009describing, everingham2010pascal}. Following the traditional multi-label model, we can obtain our visual attribute prediction model, which can predict with predefined attributes.
For each category, we select images in the attribute dataset and utilize the attribute model to extract their feature embedding, and then perform $K$-means clustering operations on these feature embeddings. The cluster center of category $c$ is denoted as $\boldsymbol{F}^{\text{Cluster}}_c \in \mathbb{R}^{K \times d}$, and the similarity $S^{embedding}_{(c_1, c_2)}$ between category $c_1$ and category $c_2$ is as follow,
\begin{equation}\label{embedding_similarity}
    S^{embedding}_{(c_1, c_2)} = \frac{1}{K^2} \sum_{i=1}^{K}\sum_{j=1}^{K}\frac{\boldsymbol{F}^{\text{Cluster}}_{c_1, i} \cdot \boldsymbol{F}^{\text{Cluster}}_{c_2, j}}{\left \| \boldsymbol{F}^{\text{Cluster}}_{c_1, i} \right \| \left \| \boldsymbol{F}^{\text{Cluster}}_{c_2, j} \right \|}.
\end{equation}

However, the attribute prediction model is trained on the base set, there may be attributes in the novel category that are unknown in the base category, and the features of the attribute encoder of the base category may not be applicable to the novel category~\cite{xu2022attribute}. Therefore, we use attribute labels for estimating the semantic similarity between novel categories and other categories. This involves specifying a set of attributes and determining beforehand whether the target category includes these attributes.

The attribute labels are represented as a binary vector $\boldsymbol{L}^{\text{attribute}} \in \mathbb{R}^{N_A}$, where $N_A$ is the number of visual attributes, and the value of each dimension is either 1 or 0, indicating the presence or absence of the corresponding attribute in the object. For categories $c_1$ and $c_2$, we compute their visual semantic similarity $S^{label}_{(c_1, c_2)}$:
\begin{equation}\label{label_similarity}
    S^{label}_{(c_1, c_2)} = \frac{\boldsymbol{L}^{\text{attribute}}_{c_1} \cdot \boldsymbol{L}^{\text{attribute}}_{c_2}}{\left \| \boldsymbol{L}^{\text{attribute}}_{c_1} \right \| \left \| \boldsymbol{L}^{\text{attribute}}_{c_2} \right \|}.
\end{equation}

Next, we discuss how to construct knowledge matrix $\boldsymbol{\zeta}$, which can be divided into three cases:
\begin{itemize}
    \item Case 1:
    For category $c_1$ and category $c_2$, $\boldsymbol{\zeta}_{(c_1,c_2)} = S^{embedding}_{(c_1, c_2)}$.
    \item Case 2:
    For category $c_1$ and category $c_2$, $\boldsymbol{\zeta}_{(c_1,c_2)} = S^{label}_{(c_1, c_2)}$.
    \item Case 3:
    When both categories $c_1$ and $c_2$ belong to the base category, $\boldsymbol{\zeta}_{(c_1,c_2)} = S^{embedding}_{(c_1, c_2)}$, otherwise, $\boldsymbol{\zeta}_{(c_1,c_2)} = S^{label}_{(c_1, c_2)}$.
\end{itemize}

The attribute feature can capture additional information beyond the given attributes, leading to a better evaluation of the semantic relationship between base classes. However, it may not represent the novel categories accurately. This paper uses Case 3 to construct the knowledge matrix, this approach considers all the attributes of the categories, including those that may not be known or discovered yet, and thus it avoids the potential bias introduced by predicting unknown attributes.

\textbf{Discussion:} Existing large language models can also serve as a source of prior attribute discrimination for specific categories~\cite{menon2023visual}. Since the motivation of this paper is to leverage prior knowledge for representation enhancement, we assume that obtaining prior attributes for specific categories is feasible. In practical applications, even in the absence of specific dataset annotations, representative attributes can still be identified using large language models to infer category relationships and enhance the quality of the learned representation space.

\subsubsection{Text Embeddings} We also tried to construct the knowledge matrix through text embeddings. The similarity of two words usually reflects the frequency of their co-occurrence in a sentence. To achieve the similarity between categories by text embedding. We adopt the word embedding method GloVe \cite{pennington2014glove} to map the word into a $300$-dimension vector denoted as $\boldsymbol{F}^{\text{GloVe}} \in \mathbb{R}^{300}$. Then, for categories $c_1$ and $c_2$, the similarity $S^{cword}_{(c_1, c_2)}$ between them is,
\begin{equation}\label{cword_similarity}
    S^{cword}_{(c_1, c_2)} = \frac{\boldsymbol{F}^{\text{GloVe}}_{c_1} \cdot \boldsymbol{F}^{\text{GloVe}}_{c_2}}{\left \| \boldsymbol{F}^{\text{GloVe}}_{c_1} \right \| \left \| \boldsymbol{F}^{\text{GloVe}}_{c_2} \right \|}.
\end{equation}

Moreover, the similarity between categories can also be measured by attribute words. The vector $\boldsymbol{F}^\text{AW}_{c} \in \mathbb{R}^{(64 \times 300)}$ indicates the feature of category $c$.
If the category $c$ contains the $i$-th attribute ${\rm attr}_i$, then $\boldsymbol{F}^\text{AW}_{c, i} = \boldsymbol{F}^{\text{GloVe}}_{{\rm attr}_i}$, otherwise, we use "null" to represent the $i$-th embedding, $\boldsymbol{F}^\text{AW}_{c, i} = \boldsymbol{F}^{\text{GloVe}}_{{\rm null}}$.
For categories $c_1$ and $c_2$, the similarity $S^{aword}_{(c_1, c_2)}$ between them is,
\begin{equation}\label{aword_similarity}
    S^{aword}_{(c_1, c_2)} = \frac{\boldsymbol{F}^\text{AW}_{c_1} \cdot \boldsymbol{F}^\text{AW}_{c_2}}{\left \| \boldsymbol{F}^\text{AW}_{c_1} \right \| \left \| \boldsymbol{F}^\text{AW}_{c_2} \right \|}.
\end{equation}

Then, we have two cases of obtaining knowledge matrix $\boldsymbol{\zeta}$ through text embedding:

\begin{itemize}
    \item Case 4:
    Based on category words, for category $c_1$ and category $c_2$, $\boldsymbol{\zeta}_{(c_1,c_2)} = S^{cword}_{(c_1, c_2)}$.
    \item Case 5:
    Based on attribute words, for category $c_1$ and category $c_2$, $\boldsymbol{\zeta}_{(c_1,c_2)} = S^{aword}_{(c_1, c_2)}$.
\end{itemize}

\textbf{Discussion:} Text embeddings are widely used in zero-shot object detection~\cite{yan2022semantics}, where misalignment between textual and visual spaces is a common challenge. Our method does not rely on enforcing this alignment. Instead, it prioritizes enhancing representation discrimination among visually similar novel classes. This design ensures that even if certain textual relationships are imperfect, they still contribute positively to model performance.

Fig. \ref{knowledge} shows the constructed knowledge matrix for the three cases. The knowledge matrix is used to represent the visual-semantic similarity between categories, which is established and fixed before FSOD training.


\subsection{Memory Prototype Bank} \label{bank}
To estimate the unbiased prototypical representation in Problem \ref{p2}, we construct representative prototypes for each category.
To build a representative prototype for each category, we use a memory prototype bank to store the feature embeddings extracted by the current detector parameters as shown in the lower part of the middle of Fig. \ref{Framework}.
The memory prototype bank stores prototypes for all the categories $\boldsymbol{P}^{M} = \left \{ \boldsymbol{P}^{M}_1,  \boldsymbol{P}^{M}_2, \cdots,  \boldsymbol{P}^{M}_C \right \}$, where $C$ is the number of categories in $\mathcal{C}_{{\rm base}} \cup \mathcal{C}_{{\rm novel}}$. For category $c$, the memory prototype bank is denoted as $\boldsymbol{P}^{M}_c = \left \{\boldsymbol{P}^{M}_{c,1}, \boldsymbol{P}^{M}_{c,2}, \cdots, \boldsymbol{P}^{M}_{c,N_c^{M}} \right \}$, where $N_c^{M}$ denotes the number of current proposal embeddings. Note that for the $n$-th prototype embedding $\boldsymbol{P}^{M}_{c,n} \in \mathbb{R}^{d}$ of category $c$, we use the ground truth box location to crop the global feature extracted by the backbone layer, and then feed it into the RoI Pooling layer to obtain the feature representation.

For each category $c$, the memory prototype bank $\boldsymbol{P}^{M}_c$ accumulates from scratch ($N_c^M=0$) and continuously stores the ground truth prototypes. We set the upper limit of the memory prototype bank as $N_c^M=2k$ (related to $k$-shot). When reaching the upper limitation, the oldest prototype feature will be replaced by a new prototype feature.

\subsection{Contextual Semantic Supervised Contrastive Learning Branch} \label{KGC}
Our work uses the knowledge matrix built from side information to weight the coefficients of negative pairs in contrastive learning, as shown in the bottom right corner of Fig. \ref{Framework}. We use a learnable linear transformation layer to map the feature representation of the object proposals and prototypes into a common feature space. The common feature space is denoted as $\boldsymbol{F} \in \mathbb{R}^{N \times 128}$, and their corresponding category labels are denoted as $y^f$.
For the prototype feature in the memory bank, we use $K$-means clustering to obtain cluster centers of each category as the new prototype feature,  which is denoted as $\boldsymbol{P} \in \mathbb{R}^{N^k C \times 128}$, where $N^k$ denotes the number of cluster centers, and their corresponding category labels are denoted as $y^p$. In practice, we set the number of cluster centers to 1, so that we only need to calculate the mean of the prototype embeddings of each category.
Thus, the contextual semantic supervised contrastive learning loss $\mathcal{L}_{CCL}$ is defined as:
\begin{equation}\label{cr_loss}
    \mathcal{L}_{CCL} = -\frac{1}{N}\sum_{i=1}^{N}\log \left ( \frac{\sum_{j=1}^{N_k \times C}\mathbb{I}_{y_i^f=y_j^p} \exp\left(\boldsymbol{F}_i \cdot \boldsymbol{P}_j/\tau \right)}{\sum_{j=1}^{N_k \times C} \exp \left(\boldsymbol{\zeta}_{y_i^f,y_j^p} \boldsymbol{F}_i \cdot \boldsymbol{P}_j/\tau\right)} \right ),
\end{equation}
where $N$ denotes the number of proposal features, $\boldsymbol{F}_i$ is the $i$-th proposal embedding after mapping, $\boldsymbol{P}_j$ denotes the $j$-th prototype embedding of one category after the $K$-means operation, $\tau$ is the hyper-parameter temperature. Note that the gradient stops at the prototype feature $\boldsymbol{P}$ and only goes back through $\boldsymbol{F}$.

\subsection{Counterfactual Data Augmentation}\label{counterfactual}

To address the issue of bias learning in few-shot object detection training, we propose a novel method that employs a knowledge matrix and counterfactual explanation to guide data augmentation, as shown in the upper left of Fig. \ref{Framework}.
Given an input image with annotation $(\mathcal{I}, \mathbf{c}, \mathbf{x})$, where $\mathbf{c}$ and $\mathbf{x}$ represent the object category and location in the image. First, to calculate the forward propagation of the detection network, we do not use the RPN network to calculate the prediction location, but directly use the annotated location in the ROI pooling layer, and calculate the feature map $\mathbf{f}_{\mathrm{pool}} \in \mathbb{R}^{w_p \times h_p \times c_p}$, where $w_p$, $h_p$, and $c_p$ denote the feature's width, height, and number of channels. Then, we compute the gradient of backpropagation from the prediction score $y_c$ (for the ground truth category $c$) to the feature map $\mathbf{f}_{\mathrm{pool}}$, then the gradients are global average pooled,
\begin{equation}
    \alpha_k = \frac{1}{w_p \times h_p} \sum_{i}^{w_p} \sum_{j}^{h_p} \frac{\partial y_c}{\partial \mathbf{f}_{\mathrm{pool}}^{ k}},
\end{equation}
where $k$ denotes the $k$-th channel. We can then compute the attribution map $\mathcal{A}_c \in \mathbb{R}^{w_p \times h_p}$ for category $c$:
\begin{equation}
    \mathcal{A}_c = \mathrm{ReLU} \left ( \sum_{k}\alpha_k \mathbf{f}_{\mathrm{pool}}^{ k} \right).
\end{equation}

Similarly, given a counter category $c^\prime$, we can compute its attribute map $\mathcal{A}_{c^\prime} \in \mathbb{R}^{w_p \times h_p}$. The counter category is suggested by the knowledge matrix, which is selected from the $k_e$ categories most similar to category $c$. In this paper, we set $k_e$ to 3. Next, we compute the counterfactual attribution map $\mathcal{A}$, i.e. why is category $c$ but not category $c^\prime$:
\begin{equation}
    \mathcal{A} =\mathrm{norm} \left ( \mathcal{A}_c \odot \left ( \max{\mathcal{A}_{c^\prime}} - \mathcal{A}_{c^\prime} \right ) \right),
\end{equation}
where $\mathrm{norm}(\cdot)$ means to normalize the value domain of the attribution map. The activation strength of the counterfactual attribution map is higher only if the region is strongly activated for category $c$ and not strongly activated for category $c^\prime$. The activation region of the counterfactual attribution map will be more fine-grained than that of the one-class attribution map. This region is the potential bias of the model in distinguishing highly confounding categories on the few-shot dataset.

We augment the data with counterfactual attribution maps. We expect to mine some samples close to the classification boundary and jointly train to improve the model's generalization, which can be achieved by erasing the discriminative region the model currently focuses on. A threshold $t$ is given, then the mask $\boldsymbol{H}$ for local region erasure is:
\begin{equation}
    \boldsymbol{H}_{i, j} = \left\{\begin{matrix}
        0, & \mathrm{if} \, \mathcal{A}_{i,j} \ge  t  \\
        1, & \mathrm{if} \, \mathcal{A}_{i,j} < t
    \end{matrix}\right. ,
\end{equation}
where $(i, j)$ denotes the position. Then we get an augmented image $\tilde{\mathbf{x}}$,
\begin{equation}
    \tilde{\mathbf{x}} = \mathbf{x} \odot \boldsymbol{H} + \boldsymbol{E} \odot (\mathbf{1} - \boldsymbol{H}),
\end{equation}
where $\boldsymbol{E}$ is a mask template with the same size as $\mathbf{x}$. Here we fill each pixel with a random integer from $[0, 255]$,
\begin{equation}
    \boldsymbol{E}_{i,j,k} = \mathrm{Rand}(0, 255).
\end{equation}

Finally, we feed the augmented data into the few-shot object detection model to train the model parameters. During each iteration, there is a probability of $\varepsilon$ that requires data augmentation operations. It should be noted that the augmented data will not be used to update the memory prototype bank, because the prototype features at this time are not complete, which will affect the performance of model learning. The erased object features will be closer to the classification boundary because the salient regions that the model currently uses to distinguish categories are erased. By aligning erased features with prototype features, the model can be forced to learn additional features to recognize objects, thus will refine the discriminative representation space further, enhancing the distinction between similar base and novel categories and improving the generalization of the model.

\subsection{Our Objective} \label{object}
In the first training stage, we follow the standard training paradigm of Faster R-CNN \cite{ren2016faster} on the base set $\mathcal{D}_{{\rm base}}$. In the second training stage, we transfer the parameters of the first training stage and add the contextual semantic supervised contrastive learning loss into the total loss, where we train the model based on  $\mathcal{D}_{{\rm base}} \cup \mathcal{D}_{{\rm novel}}$:
\begin{equation}
    \mathcal{L} = \mathcal{L}_{\rm rpn} + \lambda_1 \mathcal{L}_{\rm cls} + \lambda_2 \mathcal{L}_{\rm reg} + \lambda_3 \mathcal{L}_{CCL},
\end{equation}
where $\mathcal{L}_{\rm rpn}$ is used to optimize foreground object prediction, $\mathcal{L}_{\rm cls}$ is used for the proposal category prediction, and $\mathcal{L}_{\rm reg}$ is used to modify the final bounding box location, $\lambda_1, \lambda_2$, and $\lambda_3$ are balance hyper-parameters, which are all set to 1.

We also present a theoretical analysis of our method in terms of generalization error bounds in the \textit{Supplementary Materials}, aiming to provide insights into how the introduction of the knowledge matrix enhances model performance and why counterfactual data augmentation is effective.


\begin{table*}
    \begin{minipage}[t]{0.67\textwidth}
        \centering
        \makeatletter\def\@captype{table}\makeatother
        \caption{Effectiveness of different components of our model. We conduct experiments on the 1, 5, and 10-shot of the PASCAL VOC Split 1-3. * represents that the result was run directly from the release code.}
        \resizebox{ \textwidth}{!}{
            \begin{tabular}{cccc|ccccccccc}
                \toprule
                \multirow{2}{*}{Method}  & Knowledge            & Prototype/Feature    & Counterfactual & \multicolumn{3}{c}{Novel Set 1 $AP_{50}$}          & \multicolumn{3}{c}{Novel Set 2 $AP_{50}$} & \multicolumn{3}{c}{Novel Set 3 $AP_{50}$} \\ 
                & Matrix               & Cluster              &   Augmentation                           & 1             & 5             & 10            & 1    & 5             & 10            & 1    & 5             & 10            \\ \midrule
                TFA w/cos                & \XSolidBrush     &  \XSolidBrush    &   \XSolidBrush       & 39.8          & 55.7          & 57.0          & 23.5 & 35.1          & 39.1          & 30.8 & 49.5          & 49.8          \\ \midrule
                \multicolumn{1}{l}{FSCE* \cite{sun2021fsce}} &   \XSolidBrush       &    \XSolidBrush    &  \XSolidBrush    & \textbackslash{}   & $\mathbf{60.5}$ & $\mathbf{63.4}$ &  \textbackslash{}   & 43.4          & 48.7          &   \textbackslash{}   & 53.3          & 55.8          \\
                \multicolumn{1}{l}{FSCE+Cluster}     & \XSolidBrush &   \Checkmark   &     \XSolidBrush        &   \textbackslash{}    & 60.3          & 62.1          & \textbackslash{}     & 43.1          & 48.4          & \textbackslash{}     & 53.1          & 55.3          \\
                \multicolumn{1}{l}{FSCE+Knowledge Matrix}     & \Checkmark & \XSolidBrush & \XSolidBrush         &  \textbackslash{}   & 60.4          & 62.5          & \textbackslash{} & $\mathbf{45.8}$ & $\mathbf{51.1}$ &  \textbackslash{} & $\mathbf{55.2}$ & $\mathbf{57.2}$ \\ \midrule
                \multirow{5}{*}{Ours}    &  \XSolidBrush   &   \XSolidBrush  &   \XSolidBrush   & 47.0          & 61.6          & 63.8          & 27.0 & 44.1          & 49.0          & 38.0 & 53.7          & 57.1          \\
                &  \XSolidBrush  &  \Checkmark   &  \XSolidBrush  & 47.1          & 62.1          & 64.0          & 27.1 & 44.8          & 50.2          & 38.4 & 55.1          & 57.5          \\
                &   \Checkmark   &  \XSolidBrush   &  \XSolidBrush   & 47.8          & 62.2          & 64.4          & 27.9 & 46.7          & 51.3          & 38.7 & 56.3          & 58.2          \\
                &  \Checkmark  & \Checkmark  &    \XSolidBrush   & 48.4 & 62.6 & 64.8 & 28.6 & 46.8          & 52.2          & 39.5 & 56.6          & 58.5          \\
                &   \Checkmark    &   \Checkmark    &   \Checkmark    &  $\mathbf{49.6}$  &  $\mathbf{63.3}$   &  $\mathbf{65.2}$   &  $\mathbf{30.0}$    &  $\mathbf{47.7}$   &  $\mathbf{53.2}$   &  $\mathbf{40.2}$  & $\mathbf{56.9}$     &   $\mathbf{59.0}$   \\ \bottomrule
            \end{tabular}
        }\label{add}
    \end{minipage} 
    \begin{minipage}[t]{0.33\textwidth}
        \centering
        \makeatletter\def\@captype{table}\makeatother\caption{Effectiveness of the different knowledge matrix on PASCAL VOC Split 1 1-shot and 10-shot.}
        \vspace{8pt}
        \resizebox{\textwidth}{!}{
            \begin{tabular}{ll|cc}
            \toprule
            \multicolumn{2}{c|}{\multirow{2}{*}{Method}}                       & \multicolumn{2}{c}{Novel Set 1} \\ 
            \multicolumn{2}{c|}{}                                              & 1-shot            & 10-shot          \\ \midrule
            \multicolumn{2}{l|}{Baseline (Without Knowledge Matrix)}                                      & 47.07             & 63.97            \\ \midrule
            \multirow{7}{*}{Visual Attributes} & Case 1 w/ clustering num. $K$=1 & 47.46          &  64.52         \\
            & Case 1 w/ clustering num. $K$=5 &   47.56      & 64.71        \\
            & Case 1 w/ clustering num. $K$=9 &   47.64        &  64.66                \\
            & Case 2                       &   48.04      & 64.69            \\
            & Case 3 w/ clustering num. $K$=1 &  47.91        &  64.81                 \\
            & Case 3 w/ clustering num. $K$=5 & $\mathbf{48.42}$    & $\mathbf{64.84}$   \\
            & Case 3 w/ clustering num. $K$=9 &  47.86        & 64.44        \\ \midrule
            \multirow{2}{*}{Text Embedding}    & Case 4                       &  47.77      & 64.54            \\
            & Case 5                       &   47.85       & 63.98            \\ \bottomrule
        \end{tabular}
        }
        \label{knowledge_matrix}
    \end{minipage}
\end{table*}

\section{Experiments} \label{experiment}

We perform our method on PASCAL VOC \cite{everingham2010pascal}, MS COCO \cite{lin2014microsoft}, LVIS V1~\cite{gupta2019lvis}, FSOD-1K~\cite{fan2020few}, and FVSOD-500~\cite{fan2022few_video} benchmarks. We conduct a series of ablation studies and visualizations to validate each module of our model. We compare with methods based in part on semantic alignment to demonstrate the optimality of our method for incorporating side information in the FSOD setting. Furthermore, we compare the performance with the SOTA methods and demonstrate the advantages of our method. We also demonstrate that our method can serve as a plug-and-play module, enhancing existing FSOD methods across various backbones.

\subsection{Implementation Details}
If not otherwise stated, our experiments follow the default settings described below. TFA++\cite{sun2021fsce}, DeFRCN\cite{qiao2021defrcn}, MFDC~\cite{wu2022multi}, and DE-ViT\footnote{Referring to \url{https://arxiv.org/abs/2309.12969v2}}~\cite{zhang2023detect} are utilized as baselines, with Faster R-CNN \cite{ren2016faster} serving as the base detector. Among these, TFA++ is distinct for integrating an FPN layer \cite{lin2017feature}, a feature not shared by the other baselines. Both TFA++, DeFRCN, and MFDC employ the ResNet-101~\cite{he2016deep} network, pre-trained on ImageNet-1K, as their backbone. In contrast, DE-ViT is built upon the ViT-L/14~\cite{dosovitskiy2020image}, pre-trained on LVD-142M~\cite{oquab2023dinov2}, as its backbone. The learning rate and other hyperparameters are aligned with those of the baseline methods. Due to limitations in video memory, the Top-k parameter for DE-ViT is set to 6.
We use case 3 mentioned in Section \ref{ex_know} (visual attribute) to obtain the knowledge matrix, where the number of cluster numbers $K$ used to construct the knowledge matrix is set to 5. The number of cluster centers $N^k$ used for prototype clustering is set to 1, and the hyper-parameter temperature $\tau$ is set to 0.2.
For the counterfactual data augmentation method, the number of counter categories $k_e$ assigned to each category as $3$, the probability $\varepsilon$ of performing interpretable data augmentation during training as $0.05$, and the threshold $t$ for erasing as $0.8$. The experiments are conducted on 2 RTX 3090 graphics cards.

\begin{figure}[!t]
    \centering
    \setlength{\abovecaptionskip}{0.cm}
    \includegraphics[width = 0.5 \textwidth]{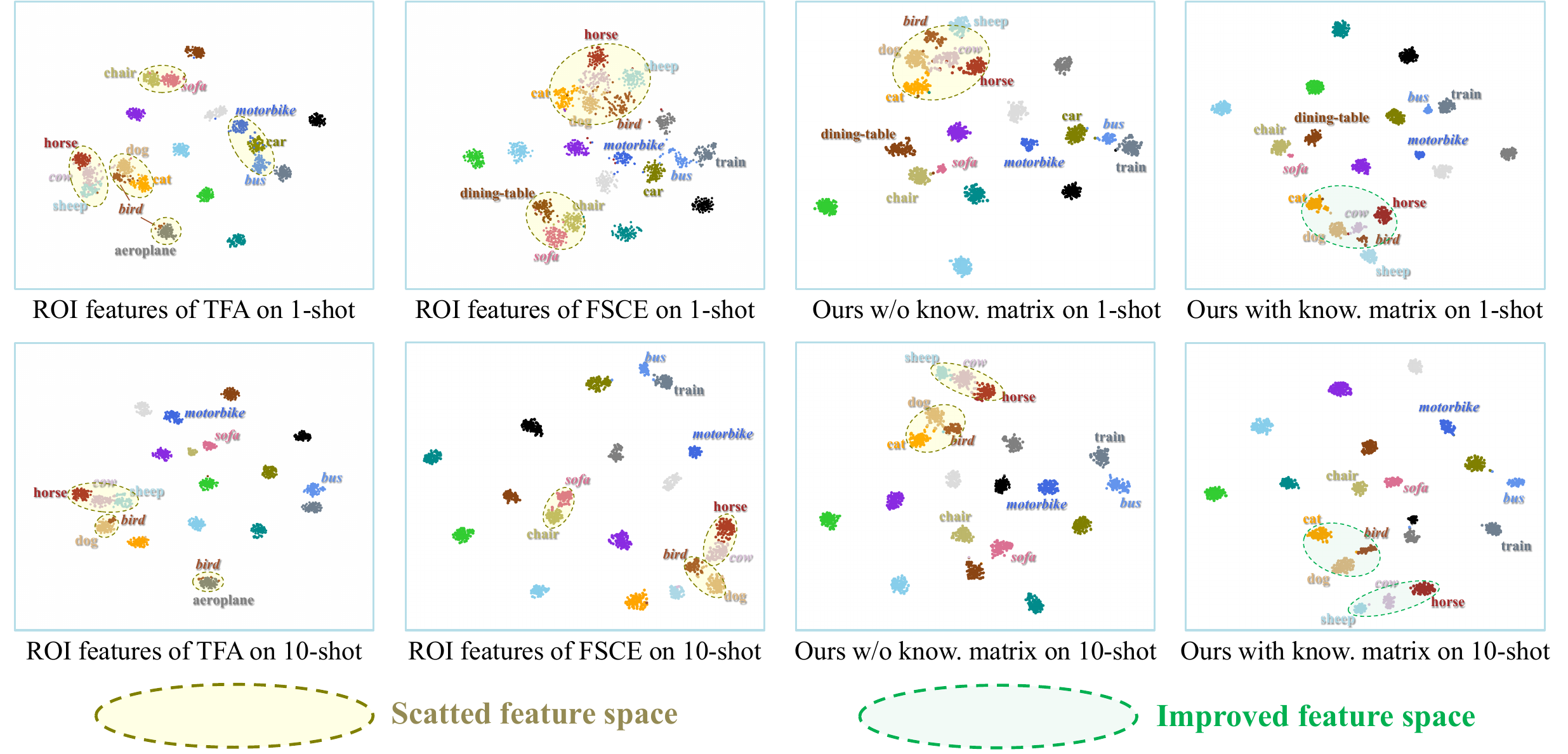}
    \caption{T-SNE visualization of the object proposal leaned by TFA, FSCE, and ours, we randomly select 100 proposals of each category. The novel categories are presented in italic font. Please zoom in for better visualization.}
    \label{TSNE}
\end{figure}

\subsection{Few-Shot Detection Benchmarks}

We adhere to the well-established benchmarks for few-shot object detection to validate our method, incorporating evaluations on five popular datasets in particular.

\textbf{PASCAL VOC}~\cite{everingham2010pascal} dataset contains 20 categories. For the few-shot object detection task, the VOC 2007 and VOC 2012 training and validation sets are used for training, and the VOC 2007 test set is used for evaluation. Following the previous work \cite{wang2020frustratingly}, $15$ categories are used as base categories, and the remaining 5 are used as novel categories. There are 3 different division schemes. We denote them as novel split 1, 2, and 3. For $k$-shot, we will randomly instance $k$ unseen instances of the novel categories, where $k=1,2,3,5,10$, we use the same instance as TFA's \cite{wang2020frustratingly} for training. $AP_{50}$ is used as an evaluation metric.

\textbf{MS COCO}~\cite{lin2014microsoft} contains 80 categories, where 20 are the same as the categories in the PASCAL VOC dataset as the novel categories, and the other 60 categories as the base categories. 5,000 images are selected from the COCO 2014 validation set for evaluation, the images in the training set and the rest images in the validation set are used for training.
For $k$-shot, each unseen novel category will randomly sample $k$ instances, where $k=10, 30$, and we use the same instance as TFA's~\cite{wang2020frustratingly}. $AP$ and $AP_{75}$ are used as the evaluation metrics.

\textbf{LVIS V1}~\cite{gupta2019lvis}, derived from MS COCO 2017, is a large-scale dataset, incorporating 1,203 categories in its Version 1.0, which we utilize to validate our method. This dataset is categorized into 337 rare (1-10 images), 461 common (10-100 images), and 405 frequent (more than 100 images) categories. Adhering to the TFA~\cite{wang2020frustratingly} setup, the common and frequent categories are treated as base categories, while rare categories are considered novel categories. For fine-tuning, up to 10 instances from each category in the training set are sampled. The detection performance of the rare categories, denoted as $AP_{r}$, serves as the evaluation metric.

\textbf{FSOD-1K}~\cite{fan2020few} is a dataset specifically crafted for the few-shot detection task, featuring 1,000 categories endowed with high-quality annotations. 800 categories are designated for the training set as base categories, while the remaining 200 categories are allocated to the test set as novel categories. 5 instances for each category in the test set are randomly sampled for fine-tuning, with the remaining samples utilized for evaluation. The $AP_{50}$ serves as the evaluation metric. The experiment is conducted five times, and the average value of these trials is reported as the final result.

\textbf{FSVOD-500}~\cite{fan2022few_video} is a large-scale video object detection dataset, encompassing 500 categories, each with class-balanced videos. 320 categories are designated for the training set as base categories, and 100 categories are allocated to the test set as novel categories. The offline test support set, containing 5 instances for each novel category, is utilized for fine-tuning. $AP$, $AP_{50}$, and $AP_{75}$ are used as the evaluation metrics.

\begin{table*}
    \begin{minipage}[t]{0.66\textwidth}
        \centering
        \makeatletter\def\@captype{table}\makeatother
        \caption{Comparison with the self-constructing knowledge matrix method. The results are on the PASCAL VOC benchmark.}
        \resizebox{ \textwidth}{!}{
            \begin{tabular}{cc|ccccc|ccccc|ccccc}
            \toprule
            \multirow{2}{*}{Method} & Counterfactual & \multicolumn{5}{c|}{Novel Split 1} & \multicolumn{5}{c|}{Novel Split 2} & \multicolumn{5}{c}{Novel Split 3} \\ 
            & Augmentation  & 1     & 2    & 3    & 5    & 10   & 1     & 2    & 3    & 5    & 10   & 1     & 2    & 3    & 5    & 10   \\ \midrule
            Self-const. knowl.    &  \XSolidBrush & 48.2  & 52.1 & 51.5  & 61.1 & 60.2 & 27.9 & 34.7 & 44.7 & 44.4 & 48.2 & 39.1  & 42.1 & 47.3 & 53.8 & 57.2 \\ 
            \rowcolor{gray!20}
            Visual knowl.             &  \XSolidBrush   & $\mathbf{48.4}$  & $\mathbf{52.2}$ & $\mathbf{53.5}$ & $\mathbf{62.6}$ & $\mathbf{64.8}$ & $\mathbf{28.6}$  & $\mathbf{34.5}$ & $\mathbf{46.6}$ & $\mathbf{46.8}$ & $\mathbf{52.7}$ & $\mathbf{39.5}$  & $\mathbf{43.3}$ & $\mathbf{49.0}$ & $\mathbf{56.6}$ & $\mathbf{58.5}$ \\
            \midrule
            Self-const. knowl.    &  \Checkmark & 48.0  & 52.4 & 52.7  & 62.0 & 62.4 & 28.1  & 35.1 & 45.7 & 45.3 & 52.0 & 39.2  & 42.9 & 49.1 & 56.4 & 58.3 \\ 
            \rowcolor{gray!20}
            Visual knowl.             &  \Checkmark    & $\mathbf{49.6}$  & $\mathbf{53.2}$ & $\mathbf{54.3}$ & $\mathbf{63.3}$ & $\mathbf{65.2}$ & $\mathbf{30.0}$  & $\mathbf{35.3}$ & $\mathbf{47.3}$ & $\mathbf{47.7}$ & $\mathbf{53.2}$ & $\mathbf{40.2}$  & $\mathbf{44.2}$ & $\mathbf{50.4}$ & $\mathbf{56.9}$ & $\mathbf{59.0}$ \\
            \bottomrule
        \end{tabular}
        }\label{self-construct}
    \end{minipage} 
    \begin{minipage}[t]{0.34\textwidth}
        \centering
        \makeatletter\def\@captype{table}\makeatother\caption{Effectiveness of the number of cluster centers. The results are on PASCAL VOC Split 1 10-shot.}\vspace{-4pt}
        \resizebox{\textwidth}{!}{
            \begin{tabular}{c|ccc}
            \toprule
            \multicolumn{1}{l|}{Number of the} & \multirow{2}{*}{Baseline} & \multicolumn{2}{c}{Knowledge Matrix}                                      \\ 
            \multicolumn{1}{l|}{Cluster Centers}  &                           & \multicolumn{1}{l}{Visual Attribute} & \multicolumn{1}{l}{Text Embedding} \\ \midrule
            w/o clustering                                   & 63.77                     & 64.42                                & 64.35                              \\
            1                                   & $\mathbf{63.97}$                     & $\mathbf{64.84}$                               & $\mathbf{64.54}$                              \\
            3                                   & 63.81                     & 64.21                                &  64.35                   \\
            5                                   & 63.80                     & 64.43                                &  64.36                    \\ \bottomrule
        \end{tabular}
        }
        \label{cluster_center}
    \end{minipage}
\end{table*}

\subsection{Ablation Study}
In this section, the method is integrated with TFA++, aiming to validate the effectiveness of various components and modules.
\subsubsection{Components of Our Method} We conduct experiments to validate the effectiveness of different components of our model. The experiments are conducted on 1, 5, and 10 shots of the PASCAL VOC split 1, 2, and 3. We analyze two important components in the CCL loss: the knowledge matrix and the clustering of the prototype in the memory prototype bank. We also analyze the effect of adding counterfactual data augmentation. Table \ref{add} shows the experiment results. When clustering is not used, we use the original prototype $\boldsymbol{P}^{M}$ in the memory prototype bank to replace the cluster center $\boldsymbol{P}$ in Equation \ref{cr_loss}. Adding the knowledge matrix or clustering alone will improve the results. In any split and shot, we observe a 0.6-2.6 (an average of 1.4) improvement without the prototype clustering component and a 0.5-2.0 (an average of 1.3) improvement with the prototype clustering component when adding the knowledge matrix. These demonstrated that using the knowledge matrix in our method can effectively improve the performance of few-shot object detection.
The clustering component helps to improve model performance, especially after introducing the knowledge matrix (an average of 0.5), which demonstrates that prototype clustering is beneficial to estimating unbiased features.
The model will reach the highest performance when both components are equipped. We also introduce the knowledge matrix into the FSCE \cite{sun2021fsce} model in a similar way and improve the performance by 2.4, 2.4, 1.9, and 1.4 on 5, 10 shots of split 2, and 5, 10 shots of split 3, respectively, but the performance drops by 0.1 and 0.9 on the 5 and 10 shots of split 1, respectively. FSCE cannot guarantee that the performance will be improved after adding the knowledge matrix, this is mainly because of the biased feature used for contrastive learning. How to introduce side information in FSOD is still an open and challenging problem, and we provide some ideas for this problem. Furthermore, our experiments have demonstrated that including clustering operations in one of the branches of FSCE contrastive learning can actually reduce model performance. This may be due to the presence of bias in the proposal process, which can subsequently result in biased cluster centers. After introducing counterfactual data augmentation, we observe an average performance improvement of $0.8$. This proves that the counterfactual data augmentation model can effectively reduce the generalization error of the model.

We also adopt the T-SNE \cite{van2008visualizing} method to inspect the proposal embedding representations, as shown in Fig.~\ref{TSNE}.
We compare 4 detectors, TFA~\cite{wang2020frustratingly}, FSCE \cite{sun2021fsce}, our model without knowledge matrix, and our model with knowledge matrix on 1-shot and 10-shot of PASCAL VOC split 1 to extract all proposal embeddings in the VOC test set. For a fair comparison, the counterfactual data augmentation model is not used. For each category, we randomly select 100 proposal embeddings with prediction scores greater than 0.5. In previous work, a novel category may implicitly use the visual composite attributes of all the base categories to develop its representation, which would generate a scatted feature space marked as shown in Fig.~\ref{TSNE}. For example, for TFA on 1 shot, the novel category bird exploits base category dog, cat, and aeroplane features to represent its feature, which is spatially dense and difficult to distinguish. Our method will effectively improve the feature space after introducing a prior visual knowledge matrix, the prior semantic relations enable the model to focus on indistinguishable categories and improve the discriminant between features. These results demonstrated that our method can effectively alleviate the scattered space problem and learn a better feature space after introducing a knowledge matrix.

\subsubsection{Performance of Different Knowledge Matrix}
Here we explore the effectiveness of the different knowledge matrices mentioned in Section \ref{ex_know}. The experiments are conducted on PASCAL VOC split 1 1-shot and 10-shot respectively. The experiments do not use counterfactual data augmentation. Table \ref{knowledge_matrix} shows the evaluation results. We observe that the knowledge matrix obtained through visual attributes and text embedding can improve the performance of the model.
In the knowledge matrix obtained by visual attributes, the results of case 3 are always better than the baseline. In the knowledge matrix obtained by text embedding, although case 5 has higher performance on 1-shot than case 1, it doesn't do well on 10-shot, only case 4 is always effective. Knowledge matrix built from side information of visual attribute has an advantage over text embedding, with $0.6$ and $0.3$ better performance on 1-shot and 10-shot. Therefore, we believe that visual attributes are more advantageous than text embeddings for building relationships of similarity.

We also tried to use the prototype cluster centers from the memory prototype bank to construct the knowledge matrix, that is, the self-constructed knowledge matrix. The self-constructed knowledge matrix is updated by continuously computing the cosine similarity between prototype centers during training.
As shown in Table~\ref{self-construct}, regardless of whether counterfactual data augmentation is used or not, the performance of utilizing a self-constructed knowledge matrix is not as strong as that of a knowledge matrix built with visual side information. This could potentially be attributed to overfitting during the training process. As a result, we suggest using side information to construct a static knowledge matrix for model training.

\subsubsection{Ablation for the Number of Cluster Centers $N^k$}
The effectiveness of the number of cluster centers $N^k$ on the performance of the model is explored.
We only experiment with 10-shot of PASCAL VOC split 1, because, in a low shot, such as 1-shot, too many cluster centers cannot be enforced.
The results are presented in Table \ref{cluster_center}. When "w/o clustering" is indicated in the table, it means that the original prototype embedding $\boldsymbol{P}^{M}$ is utilized to replace the cluster center $\boldsymbol{P}$, followed by the same operation.
We find that after adding the clustering operation with $N^k=1$, regardless of whether the knowledge matrix is used, the performance of the model has been improved, especially with the visual attribute knowledge matrix has $0.42$ improvements, but as $N^k$ increases, the performance of the model decreases.
Therefore, we perform the clustering operation with the number of cluster centers $N^k=1$, which is effective for any shots.

\subsubsection{Ablation for the Counterfactual Data Augmentation}

We examine the effects of various modules and parameters on counterfactual interpretability data augmentation. The experiments are conducted on PASCAL VOC split 1 3-shot.
We compare our method with the following approaches: 1) random mask augmentation~\cite{zhong2020random}, a common data augmentation trick for object detection, 2) random counter class augmentation without knowledge matrix selection of specified categories, 3) Ismail \textit{et al.}\cite{ismail2021improving}, an explanation-guided data augmentation method based on important region enhancement, and 4) Xiao \textit{et al.}\cite{xiao2023masked}, an explanation-guided data augmentation method based on key region masking. For random mask augmentation, we adopt parameters designed for object detection consistent with the original settings~\cite{zhong2020random}. For random counter-class augmentation, we randomly select a counter-class among categories other than the ground truth category. 
Compared with the common data augmentation tricks or existing explanation-guided data augmentation methods, our proposed counterfactual augmentation method is specifically tailored for FSOD tasks and intricately linked with CCL. It serves to prevent the overfitting of CCL during the process of distinguishing between base and novel categories.

As shown in Fig.~\ref{explanation_ablation} A, we find that random mask augmentation degraded the model performance, possibly due to overfitting in few-shot learning, which could mislead the model's learning. We observed only a slight improvement with counter-class augmentation, which may be because some classes are highly discriminative, and applying counterfactual interpretable data augmentation between these classes does not lead to further improvements in their discriminative power. The important region enhancement method (Ismail \textit{et al.} \cite{ismail2021improving}) degrades baseline performance, while the key region masking method (Xiao \textit{et al.} \cite{xiao2023masked}) provides only slight improvements. Therefore, these approaches may not be well-suited for few-shot scenarios. Our data augmentation method guided by a knowledge matrix can provide more effective augmentation data for few-shot data. Thus, our proposed counterfactual data augmentation method effectively enhances the distinction between base and novel categories.

\begin{figure}[!t]
    \centering
    \setlength{\abovecaptionskip}{0.cm}
    \includegraphics[width = 0.48 \textwidth]{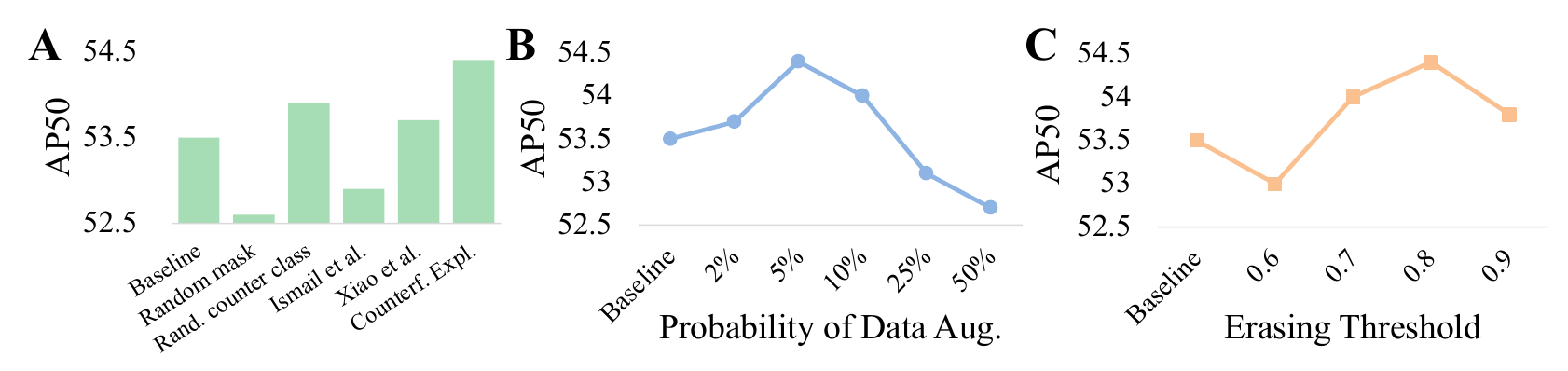}
    \caption{\textbf{A.} Comparison with different data augmentation strategies. Counterfactual data augmentation guided by a knowledge matrix achieves the best performance. \textbf{B.} Probability of data augmentation $\varepsilon$,  the best performance is achieved when set to $5\%$. \textbf{C.} Erasing threshold $t$, the best performance is achieved when set to $0.8$.}
    \label{explanation_ablation}
    \vspace{-10pt}
\end{figure}

We also discuss the impact of data augmentation probability $\varepsilon$ and erasing threshold $t$. As shown in Fig.~\ref{explanation_ablation} B, we observed that an excessively high enhancement frequency leads to decreased model performance. This is mainly because the high frequency erases the currently learned discriminative features of the model will interfere with the original feature learning. We found that the best performance was achieved when the frequency $\varepsilon$ was set to $5\%$. As shown in Fig.~\ref{explanation_ablation} C, a lower erasing threshold denotes erasing more regions. If the erasing region is too large, the enhanced image may be considered an out-of-distribution sample, which can reduce the performance of the model. On the other hand, if the erasing region is too small, it may affect the quality of the generated samples used for training. We found that setting the erasing threshold $t$ to $0.8$ achieved the best performance.

\subsection{Comparison with the Alignment-based Methods}

We explore additional techniques to incorporate knowledge into FSOD. Specifically, we drew inspiration from previous work on zero-shot object detection \cite{zheng2021visual, rahman2020improved, rahman2020zero}, which leverages semantic feature alignment. We adapted these methods for use in the FSOD setting. We assume that the novel categories are known beforehand and semantic embeddings of their visual attributes are available. Our baseline models for this study were TFA and FSCE, and we incorporated alignment modules into the classification layer branch. The alignment technique was introduced during the first stage of training, while the semantic embeddings of novel categories were included in the second stage of fine-tuning. The same freezing strategy was employed as TFA or FSCE. We verify the optimality of our method in integrating side information on the PASCAL VOC dataset, as shown in Table \ref{zero-shot}, our method achieved the best results on any split/shot. It shows that the integration of side information into FSOD in the form of category relationships is better than that based on semantic feature alignment. Furthermore, we discovered that although the semantic alignment method is beneficial for enhancing model performance (It improves the performance of TFA), it appears to be at odds with the contrastive learning module (It degrades the performance of FSCE). This is largely due to the fact that the semantic alignment module has a predetermined prototype, whereas contrastive learning aims to increase the distance between different classes, which can potentially disrupt the established semantic structure and result in reduced performance.

\begin{table}[!t]
    \centering
    \caption{Comparison with the alignment-based methods. * represents that the result was run directly from the release code.}
    \resizebox{0.48 \textwidth}{!}{
        \begin{tabular}{c|ccccc|ccccc|ccccc}
            \toprule
            \multirow{2}{*}{Method} & \multicolumn{5}{c|}{Novel Split 1}                                                   & \multicolumn{5}{c|}{Novel Split 2}                                                   & \multicolumn{5}{c}{Novel Split 3}                                                   \\ 
            & 1             & 2             & 3             & 5             & 10            & 1             & 2             & 3             & 5             & 10            & 1             & 2             & 3             & 5             & 10            \\ \midrule
            TFA w/fc                  \cite{wang2020frustratingly} & 36.8          & 29.1          & 43.6          & 55.7          & 57.0          & 18.2          & 29.0          & 33.4          & 35.5          & 39.0          & 27.7          & 33.6          & 42.5          & 48.7          & 50.2          \\
            TFA+VL-SZSD                 \cite{zheng2021visual}         & 40.1          & 35.3          & 47.3          & 54.6          & 55.1          & 22.0          & 27.9          & 37.6          & 37.1          & 39.5          & 30.8          & 37.4          & 35.7          & 45.8          & 50.0          \\
            TFA+PL-ZSD                   \cite{rahman2020improved}  & 40.8          & 40.7          & 47.5          & 53.2          & 58.4          & 20.1          & 26.5          & 41.0          & 41.2          & 42.6          & 36.5          & 39.0          & 43.4          & 49.4          & 48.9          \\
            TFA+Rahman \textit{et al.}           \cite{rahman2020zero}          & 45.9              & 44.8           & 47.5      &  57.8        & 61.2         &  24.6    & 25.9     & 38.7    & 37.1     & 42.2        & 29.8    & 38.3    & 41.7      & 48.9     &  47.7   \\ \midrule
            FSCE* \cite{sun2021fsce} & \textbackslash{} & \textbackslash{}  & 48.2  & 60.5  & 63.4  & \textbackslash{} & \textbackslash{}  & 44.6  & 44.6  & 50.7  & \textbackslash{} & \textbackslash{}  & 46.7  & 54.7  & 57.2\\
            FSCE+VL-SZSD \cite{zheng2021visual} & \textbackslash{} & \textbackslash{} & 48.3 & 53.8 & 61.0  & \textbackslash{} & \textbackslash{}  & 40.0  &37.5  & 45.9  & \textbackslash{} & \textbackslash{} & 41.9  & 50.6  & 53.8\\
            FSCE+PL-ZSD                   \cite{rahman2020improved} & \textbackslash{} & \textbackslash{} & 49.5  & 52.9  & 58.6  & \textbackslash{} & \textbackslash{} & 42.0 & 41.6  & 46.0  & \textbackslash{} & \textbackslash{}  & 46.2  & 53.8  & 53.9\\
            FSCE+Rahman \textit{et al.}           \cite{rahman2020zero} & \textbackslash{} & \textbackslash{}  & 44.9  & 54.5  & 60.7  & \textbackslash{} & \textbackslash{}  & 43.6  & 41.8  & 46.3  & \textbackslash{} & \textbackslash{}  & 43.5  & 51.4  & 55.1 \\ \midrule
            \rowcolor{gray!20}
            Ours                    & $\mathbf{48.4}$ & $\mathbf{52.5}$ & $\mathbf{53.5}$ & $\mathbf{62.9}$ & $\mathbf{64.8}$ & $\mathbf{28.6}$ & $\mathbf{34.5}$ & $\mathbf{46.6}$ & $\mathbf{46.8}$ & $\mathbf{52.2}$ & $\mathbf{39.5}$ & $\mathbf{43.3}$ & $\mathbf{49.0}$ & $\mathbf{56.6}$ & $\mathbf{58.5}$ \\ \bottomrule
        \end{tabular}
    }
    \label{zero-shot}
\end{table}

\begin{table*}[!t]
    \centering
    \caption{Few-shot object detection evaluation results on PASCAL VOC~\cite{everingham2010pascal}. The evaluation metric adopts the mean average precision (mAP@0.5). $\dagger$ denotes further fine-tuned on the novel categories.}
    \vspace{-10pt}
    \resizebox{\textwidth}{!}{
        \begin{tabular}{lcccc|ccccc|ccccc|ccccc}
        \toprule
        \multirow{2}{*}{Method} & \multirow{2}{*}{Paper Year} & \multirow{2}{*}{Backbone} & \multirow{2}{*}{Base Detector} & \multirow{2}{*}{Side Information} & \multicolumn{5}{c}{Novel Split 1} & \multicolumn{5}{c}{Novel Split 2} & \multicolumn{5}{c}{Novel Split 3} \\ \cmidrule(lr){6-10} \cmidrule(lr){11-15} \cmidrule(l){16-20}
         &     &                           &      &                     & 1    & 2    & 3    & 5    & 10    & 1    & 2    & 3    & 5    & 10    & 1    & 2    & 3    & 5    & 10    \\ \midrule
        LSTD~\cite{chen2018lstd} & AAAI 18 & VGGNet-16 & SSD & N/A & 8.2 & 1.0 & 12.4 & 29.1 & 38.5 & 11.4 & 3.8 & 5.0 & 15.7 & 31.0 & 12.6 & 8.5 & 15.0 & 27.3 & 36.3 \\
        FSRW~\cite{kang2019few} & ICCV 19 & DarkNet-19 & YOLOv2 & N/A & 14.8 & 15.5 & 26.7 & 33.9 & 47.2 & 15.7 & 15.3 & 22.7 & 30.1 & 40.5 & 21.3 & 25.6 & 28.4 & 42.8 & 45.9\\
        MetaDet-FRCN~\cite{wang2019meta} & ICCV 19 & VGGNet-16 & Faster R-CNN & N/A & 18.9 & 20.6 & 20.1 & 36.8 & 49.6 & 21.8 & 23.1 & 27.8 & 31.7 & 43.0 & 20.6 & 23.9 & 29.4 & 43.9 & 44.1 \\
        Meta R-CNN~\cite{yan2019meta} & ICCV 19 & ResNet-101 & Faster R-CNN & N/A & 19.9 & 25.5 & 35.0 & 45.7 & 51.5 & 10.4 & 19.4 & 29.6 & 34.8 & 45.4 & 14.3 & 18.2 & 27.5 & 41.2 & 48.1 \\
        RepMet~\cite{karlinsky2019repmet} & CVPR 19 & InceptionV3 & FPN+DCN & N/A & 26.1 & 32.9 & 34.4 & 38.6 & 41.3 & 17.2 & 22.1 & 23.4 & 28.3 & 35.8 & 27.5 & 31.1 & 31.5 & 34.4 & 37.2\\
        NP-RepMet~\cite{yang2020restoring} & NeurIPS 19 & InceptionV3 & FPN+DCN & N/A & 37.8 & 40.3 & 41.7 & 47.3 & 49.4 & $\mathbf{41.6}$ & 43.0 & 43.4 & 47.4 & 49.1 & 33.3 & 38.0 & 39.8 & 41.5 & 44.8 \\
        TFA w/cos~\cite{wang2020frustratingly} & ICML 20 & ResNet-101 & Faster R-CNN & N/A & 39.8 & 36.1 & 44.7 & 55.7 & 56.0 & 23.5 & 26.9 & 34.1 & 35.1 & 39.1 & 30.8 & 34.8 & 42.8 & 49.5 & 49.8 \\
        MPSR~\cite{wu2020multi} & ECCV 20 & ResNet-101 & Faster R-CNN & N/A & 41.7 & 42.5 & 51.4 & 55.2 & 61.8 & 24.4 & 29.3 & 39.2 & 39.9 & 47.8 & 35.6 & 41.8 & 42.3 & 48.0 & 49.7 \\
        Retentive R-CNN~\cite{fan2021generalized} & CVPR 21 & ResNet-101 & Retentive R-CNN & N/A& 42.4 & 45.8 & 45.9 & 53.7 & 56.1 & 21.7 & 27.8 & 35.2 & 37.0 & 40.3 & 30.2 & 37.6 & 43.0 & 49.7 & 50.1     \\
        CME~\cite{li2021beyond} & CVPR 21 & DarkNet-19 & YOLOv2 & N/A & 41.5 & 47.5 & 50.4 & 58.2 & 60.9 & 27.2 & 30.2 & 41.4 & 42.5 & 46.8 & 34.3 & 39.6 & 45.1 & 48.3 & 51.5 \\
        FSCE~\cite{sun2021fsce} & CVPR 21 & ResNet-101 & Faster R-CNN & N/A & 44.2 & 43.8 & 51.4 & 61.9 & 63.4 & 27.3 & 29.5 & 43.5 & 44.2 & 50.2 & 37.2 & 41.9 & 47.5 & 54.6 & 58.5 \\
        QA-FewDet~\cite{Han2021ICCV} & ICCV 21 & ResNet-101 & Faster R-CNN & N/A & 42.4 & 51.9 & 55.7 & 62.6 & 63.4 & 25.9 & 37.8 & 46.6 & 48.9 & 51.1 & 35.2 & 42.9 & 47.8 & 54.8 & 53.5 \\
        $FSOD^{up}$~\cite{Wu2021ICCV} & ICCV 21 & ResNet-101 & Faster R-CNN & N/A & 43.8 & 47.8 & 50.3 & 55.4 & 61.7 & 31.2 & 30.5 & 41.2 & 42.2 & 48.3 & 35.5 & 39.7 & 43.9 & 50.6 & 53.5 \\
        DMNet~\cite{lu2022decoupled} & T-Cyber. 22 & ResNet-101 & DMNet & N/A & 34.7 & 50.7 & 54.0 & 58.8 & 62.5 & 31.3 & 28.2 & 41.8 & 46.2 & 52.7 & 38.6 & 40.0 & 43.4 & 48.9 & 48.9 \\
        MRSN~\cite{ma2022mutually}  & ECCV 22 & ResNet-101 & Faster R-CNN & N/A & 47.6 & 48.6 & $\mathbf{57.8}$ & 61.9 & 62.6 & 31.2 & 38.3 & 46.7 & 47.1 & 50.6 & 35.5 & 30.9 & 45.6 & 54.4 & 57.4 \\
        Xiao \textit{et al.}~\cite{xiao2022few}& TPAMI 23 & ResNet-18 & Faster R-CNN & N/A & 26.9 & 35.7 & 42.3 & 48.9 & 57.8 & 21.2 & 26.7 & 30.6 & 37.7 & 45.1 & 24.3 & 30.4 & 36.3 & 41.6 & 50.1 \\
        CKPC \cite{chen2023category}  & TIP 23 & ResNet-101 & Faster R-CNN & N/A & 45.5 & 52.4 & 56.6 & 61.7 & 63.9 & 33.4 & $\mathbf{43.5}$ & $\mathbf{47.3}$ & $\mathbf{49.4}$ & 52.1    & 40.4 & 43.7 & 48.5 & 54.0    & 58.8 \\
        SRR-FSD~\cite{zhu2021semantic} & CVPR 21 & ResNet-101 & Faster R-CNN & Word2Vec~\cite{mikolov2013distributed} & 47.8 & 50.5 & 51.3 & 55.2 & 56.8 & 32.5 & 35.3 & 39.1 & 40.8 & 43.8 & 40.1 & 41.5 & 44.3 & 46.9 & 46.4 \\
        UA-RPN~\cite{fan2022few} & ECCV 22 & ResNet-50 & Faster R-CNN & ImageNet~\cite{deng2009imagenet} & 40.1 & 44.2 & 51.2 & 62.0 & 63.0 & 33.3 & 33.1 & 42.3 & 46.3 & 52.3 & 36.1 & 43.1 & 43.5 & 52.0 & 56.0 \\
        KD-TFA++~\cite{pei2022few} & ECCV 22 & ResNet-101 & Faster R-CNN & PPC~\cite{xie2021propagate} & 47.0 & 50.2 & 52.5 & 62.1 & 64.2 & 29.7 & 32.9 & 45.9 & 48.5 & 51.1 & $\mathbf{42.6}$ & $\mathbf{46.5}$ & 48.8 & 56.8 & 57.4 \\
        \rowcolor{gray!20}
        TFA++ w/ ours & Our Method & ResNet-101 & Faster R-CNN & Visual Attribute & $\mathbf{49.6}$ & $\mathbf{53.2}$ & 54.4 & $\mathbf{63.3}$ & $\mathbf{65.2}$ & 30.0 & 35.3 & $\mathbf{47.3}$ & 47.7 & $\mathbf{53.2}$ & 40.2 & 44.2 & $\mathbf{50.4}$ & $\mathbf{56.9}$ & $\mathbf{59.0}$ \\
        \midrule
        FADI~\cite{cao2021few} & NeurIPS 21 & ResNet-101 & Faster R-CNN & WordNet~\cite{miller1995wordnet} & 50.3 & 54.8 & 54.2 & 59.3 & 63.2 & 30.6 & 35.0 & 40.3 & 42.8 & 48.0 & 45.7 & 49.7 & 49.1 & 55.0 & 59.6 \\
        Meta Faster R-CNN~\cite{han2022meta} & AAAI 22 & ResNet-101 & Faster R-CNN & N/A & 43.0 & 54.5 & 60.6 & 66.1 & 65.4 & 27.7 & 35.5 & 46.1 & 47.8 & 51.4 & 40.6 & 46.4 & 53.4 & 59.9 & 58.6 \\
        Meta-DETR~\cite{zhang2022meta} & TPAMI 22 & ResNet-101 & Deformable DETR & N/A & 40.6 & 51.4 & 58.0 & 59.2 & 63.6 & 37.0 & 36.6 & 43.7 & 49.1 & 54.6 & 41.6 & 45.9 & 52.7 & 58.9 & 60.6 \\
        LVC~\cite{kaul2022label} & CVPR 22 & ResNet-101 & Faster R-CNN & N/A & 54.5 & 53.2 & 58.8 & 63.2 & 65.7 & 32.8 & 29.2 & 50.7 & 49.8 & 50.6 & 48.4 & 52.7 & 55.0 & 59.6 & 59.6 \\
        KFSOD~\cite{zhang2022kernelized} & CVPR 22 & ResNet-101 & Faster R-CNN & N/A & 44.6 & - & 54.5 & 60.9 & 65.8 & 37.8 & - & 43.1 & 48.1 & 50.4 & 34.8 & - & 44.1 & 52.7 & 53.9\\
        FCT~\cite{han2022few} & CVPR 22 & PVTv2-B2-Li & Faster R-CNN & N/A & 49.9 & 57.1 & 57.9 & 63.2 & 67.1 & 27.6 & 34.5 & 43.7 & 49.2 & 51.2 & 39.5 & 54.7 & 52.3 & 57.0 & 58.7 \\
        VFA~\cite{han2023vfa} & AAAI 23 & ResNet-101 & Meta R-CNN++ & N/A & 57.7 & 64.6 & 64.7 & 67.2 & 67.4 & 41.4 & 46.2 & 51.1 & 51.8 & 51.6 & 48.9 & 54.8 & 56.6 & 59.0 & 58.9 \\
        ICPE & AAAI 23 & ResNet-101 & Meta R-CNN & N/A & 54.3 & 59.5 & 62.4 & 65.7 & 66.2 & 33.5 & 40.1 & 48.7 & 51.7 & 52.5 & 50.9 & 53.1 & 55.3 & 60.6 & 60.1\\
        $\sigma$-ADP~\cite{du2023adaptive} & ICCV 23 & ResNet-101 & Faster R-CNN & N/A & 52.3 & 55.5 & 63.1 & 65.9 & 66.7 & 42.7 & 45.8 & 48.7 & 54.8 & $\mathbf{56.3}$ & 47.8 & 51.8 & 56.8 & 60.3 & 62.4 \\
        FS-DETR~\cite{bulat2023fs} & ICCV 23 & ResNet-50 & DETR & N/A & 45.0 & 48.5 & 51.5 & 52.7 & 56.1 & 37.3 & 41.3 & 43.4 & 46.6 & 49.0 & 43.8 & 47.1 & 50.6 & 52.1 & 56.9 \\
        DeFRCN~\cite{qiao2021defrcn} & ICCV 21 & ResNet-101 & Faster R-CNN & ImageNet~\cite{deng2009imagenet} & 57.0 & 58.6 & 64.3 & 67.8 & 67.0 & 35.8 & 42.7 & 51.0 & 54.5 & 52.9 & 52.5 & 56.6 & 55.8 & 60.7 & 62.5 \\
        PTF+KI~\cite{yang2022efficient} & TIP 22 & ResNet-101 & DeFRCN & ImageNet~\cite{deng2009imagenet} & 57.0 & 62.3 & 63.3 & 66.2 & 67.6 & 42.8 & 44.9 & 50.5 & 52.3 & 52.2 & 50.8 & 56.9 & 58.5 & 62.1 & 63.1 \\
        MFDC~\cite{wu2022multi} & ECCV 22 & ResNet-101 & DeFRCN & ImageNet~\cite{deng2009imagenet} & 63.4 & 66.3 & 67.7 & 69.4 & 68.1 & 42.1 & 46.5 & 53.4 & 55.3 & 53.8 & 56.1 & 58.3 & 59.0 & 62.2 & 63.7 \\
        NIFF-DeFRCN~\cite{guirguis2023niff} & CVPR 23 & ResNet-101 & DeFRCN & ImageNet~\cite{deng2009imagenet} & 63.5 & 67.2 & $\mathbf{68.3}$ & $\mathbf{71.1}$ & 69.3 & 37.8 & 41.9 & 53.4 & $\mathbf{56.0}$ & 53.5 & 55.3 & 60.5 & 61.1 & 63.7 & 63.9 \\
        KD-DeFRCN~\cite{pei2022few} & ECCV 22 & ResNet-101 & DeFRCN & ImageNet~\cite{deng2009imagenet}, PPC~\cite{xie2021propagate} & 58.2 & 62.5 & 65.1 & 68.2 & 67.4 & 37.6 & 45.6 & 52.0 & 54.6 & 53.2 & 53.8 & 57.7 & 58.0 & 62.4 & 62.2\\
        Norm-VAE~\cite{xu2023generating} & CVPR 23 & ResNet-101 & DeFRCN & ImageNet~\cite{deng2009imagenet}, Word2Vec~\cite{mikolov2013distributed} & 62.1 & 64.9 & 67.8 & 69.2 & 67.5 & 39.9 & 46.8 & $\mathbf{54.4}$ & 54.2 & 53.6 & 58.2 & 60.3 & 61.0 & 64.0 & 65.5 \\
        MM-FSOD~\cite{han2022multimodal} & ArXiv 22 & ResNet-101 & DeFRCN & ImageNet~\cite{deng2009imagenet}, CLIP~\cite{radford2021learning} & 59.4 & 59.5 & 64.6 & 68.7 & 68.4 & 36.0 & 45.5 & 51.5 & 55.0 & 55.2 & 54.2 & 53.7 & 57.5 & 60.8 & 62.5 \\
        FPD~\cite{wang2024fine} & AAAI 24 & ResNet-101 & Meta-RCNN & N/A & 46.5 & 62.3 & 65.4 & 68.2 & 69.3 & 32.2 & 43.6 & 50.3 & 52.5 & 56.1 & 43.2 & 53.3 & 56.7 & 62.1 & 64.1 \\
        SMILe~\cite{majee2024smile} & ECCV 24 & ResNet-101 & Faster R-CNN & N/A & 40.9 & - & - & 59.7 & 62.0 & 26.5 & - & - & 49.5 & 52.3 & 42.6 & - & - & 56.4 & 61.4 \\
        T-GSEL~\cite{zhang2025learning} & IJCV 25 & PVTv2-B2-Li & Faster R-CNN & N/A & 50.4 & 63.6 & 61.9 & 68.6 & 67.3 & 31.3 & 32.9 & 43.6 & 47.9 & 53.9 & 41.2 & 49.8 & 54.1 & 62.1 & 61.9\\
        SNIDA~\cite{wang2024snida} & CVPR 24 & ResNet-101 & DeFRCN & ImageNet~\cite{deng2009imagenet}, CLIP~\cite{radford2021learning} & 59.3 & 60.8 & 64.3 & 65.4 & 65.6 & 35.2 & 40.8 & 50.2 & 54.6 & 50.0 & 51.6 & 52.4 & 55.9 & 58.5 & 62.6 \\
        \rowcolor{gray!20}
        DeFRCN w/ ours & Our Method & ResNet-101 & DeFRCN & ImageNet~\cite{deng2009imagenet}, Visual Attribute & 58.6 & 61.9 & 65.2 & 68.8 & 67.7 & 38.8 & 46.7 & 52.8 & 55.1 & 54.1 & 56.5 & 58.1 & 59.6 & 61.0 & 63.1 \\ \rowcolor{gray!20}
        MFDC w/ ours & Our Method & ResNet-101 & DeFRCN & ImageNet~\cite{deng2009imagenet}, Visual Attribute & $\mathbf{64.9}$ & $\mathbf{67.3}$ & 67.8 & 70.5 & $\mathbf{70.3}$ & $\mathbf{42.9}$ & $\mathbf{48.4}$ & 53.9 & 55.5 & 53.9 & $\mathbf{59.4}$ & $\mathbf{62.0}$ & $\mathbf{61.2}$ & $\mathbf{64.8}$ & $\mathbf{65.8}$ \\
        \midrule \midrule
        DE-ViT$^{\dagger}$~\cite{zhang2023detect} & CoRL 24 & ViT-L/14 & Faster R-CNN & LVD-142M~\cite{oquab2023dinov2} & 43.3 & 52.7 & 56.9 & 65.5 & 68.4 & 27.9 & 34.4 & 51.6 & 60.2 &   65.2 & 49.7 & 60.5 & 61.8 & 64.1 & 64.8 \\ 
        FM-FSOD~\cite{han2024few} & CVPR 24 & ViT-L/14 & Large language model & LVD-142M~\cite{oquab2023dinov2}, Vicuna~\cite{chiang2023vicuna} & 40.1 & 53.5 & 57.0 & 68.6 & \textbf{72.0} & \textbf{33.1} & 36.3 & 48.8 & 54.8 & 66.2 & 39.2 & 50.2 & 55.7 & 63.4 & 68.1 \\
        \rowcolor{gray!20}
        DE-ViT w/ ours & Our Method & ViT-L/14 & Faster R-CNN & LVD-142M~\cite{oquab2023dinov2} & $\mathbf{46.9}$ & $\mathbf{55.7}$ & $\mathbf{57.6}$ & $\mathbf{69.4}$ & $70.8$ & $30.0$ & $\mathbf{36.6}$ & $\mathbf{54.6}$ & $\mathbf{63.9}$ & $\mathbf{66.2}$ & $\mathbf{51.4}$ & $\mathbf{62.1}$ & $\mathbf{63.5}$ & $\mathbf{69.3}$ & $\mathbf{70.9}$ \\ 
        \bottomrule
        \end{tabular}
            }
    \label{PascVoc}
\end{table*}

\subsection{Comparison with the State-of-the-art Methods}

In our experiments, we compared our method with state-of-the-art methods on the PASCAL VOC, MS COCO, LVIS V1, FSOD-1K, and FSVOD-500 datasets. Among these methods, SRR-FSD~\cite{zhu2021semantic} utilizes Word2Vec~\cite{mikolov2013distributed}, the DeFRCN's (and other methods based on it) PCB module~\cite{qiao2021defrcn} incorporates the ImageNet pre-training model, KD~\cite{pei2022few} uses visual words PPC~\cite{xie2021propagate}, FADI~\cite{cao2021few} leverages WordNet~\cite{miller1995wordnet}, MM-FSOD~\cite{han2022multimodal} and SNIDA~\cite{wang2024snida} use pre-trained CLIP text encoder, UA-RPN~\cite{fan2022few} leverages the ImageNet dataset, and Norm-VAE~\cite{xu2023generating} uses Word2Vec~\cite{mikolov2013distributed} as side information. FM-FSOD~\cite{han2024few} leveraged a pre-trained large language model with carefully designed instructions. Other methods in our comparison do not utilize side information.

\subsubsection{PASCAL VOC Results} Table \ref{PascVoc} shows the results of three different splits with distinct shots, our method achieves superior performance in most shots/splits.
Compared with the methods SRR-FSD~\cite{zhu2021semantic}, FADI~\cite{cao2021few}, and KD-TFA++~\cite{pei2022few} that introduce side information, our method based on TFA++ is $11.5\%$, $1.6\%$, and $1.9\%$ higher than them in the average results of 15 experiments. Since our method leverages visual attributes as side information to enhance the discrimination training between similar base categories and novel categories, resulting in a significant performance improvement compared to state-of-the-art methods. Moreover, we adopt a counterfactual data augmentation method to reduce potential bias in model learning. This combination of techniques helps our method achieve significant performance. We also combine with stronger baselines, including DeFRCN~\cite{qiao2021defrcn}, MFDC~\cite{wu2022multi}, and DE-ViT~\cite{zhang2023detect}, achieving SOTA results. Our method can improve the average results on 15 experiments of DeFRCN by $3.4\%$ and $2.7\%$ on MFDC. On the backbone based on ResNet-101, our method achieves SOTA results under 10 settings, which are on average $2.5\%$ and $2.2\%$ higher than the current SOTA methods NIFF-DeFRCN~\cite{guirguis2023niff} and Norm-VAE~\cite{xu2023generating} respectively. The method is also verified on the stronger ViT~\cite{dosovitskiy2020image} backbone which is pre-trained on LVD-142M~\cite{oquab2023dinov2}, using DE-ViT as the baseline. DE-ViT, a meta-learning based method, is combined with our method and fine-tuned on the novel set following base training to achieve enhanced results. Our method improves the average results of DE-ViT by $5.0\%$. 
We observe that using ViT as the backbone yields stronger results than ResNet-101 for 3-10 shots, but performance on 1 and 2 shots is less effective. Our analysis suggests that ViT is prone to overfitting, and since PASCAL VOC does not constitute a large-scale dataset, it results in a diminished ability to generalize well when data is extremely scarce. 
Compared to FM-FSOD~\cite{han2024few}, which leverages a large language model architecture and ViT as the vision backbone, our approach achieves the best results in 13 settings, surpassing it by an average of 7.6\%.
Fig.~\ref{Visualization} visualizes some detection results of TFA, FSCE, FADI, DE-ViT, and our method based on TFA++ and DE-ViT for novel categories on the PASCAL VOC dataset with different data splits and shots. 
TFA, FSCE, and FADI often have ambiguous decision boundaries for novel categories, as well as DE-ViT.
After incorporating a visual knowledge matrix, our method strengthens the discrimination between similar categories. In addition, the use of a counterfactual data augmentation module reduces the potential learning bias of the model and gets more accurate predictions.

\begin{figure}[!t]
    \centering
    \setlength{\abovecaptionskip}{0.cm}
    \begin{overpic}[width=0.48\textwidth,tics=8]{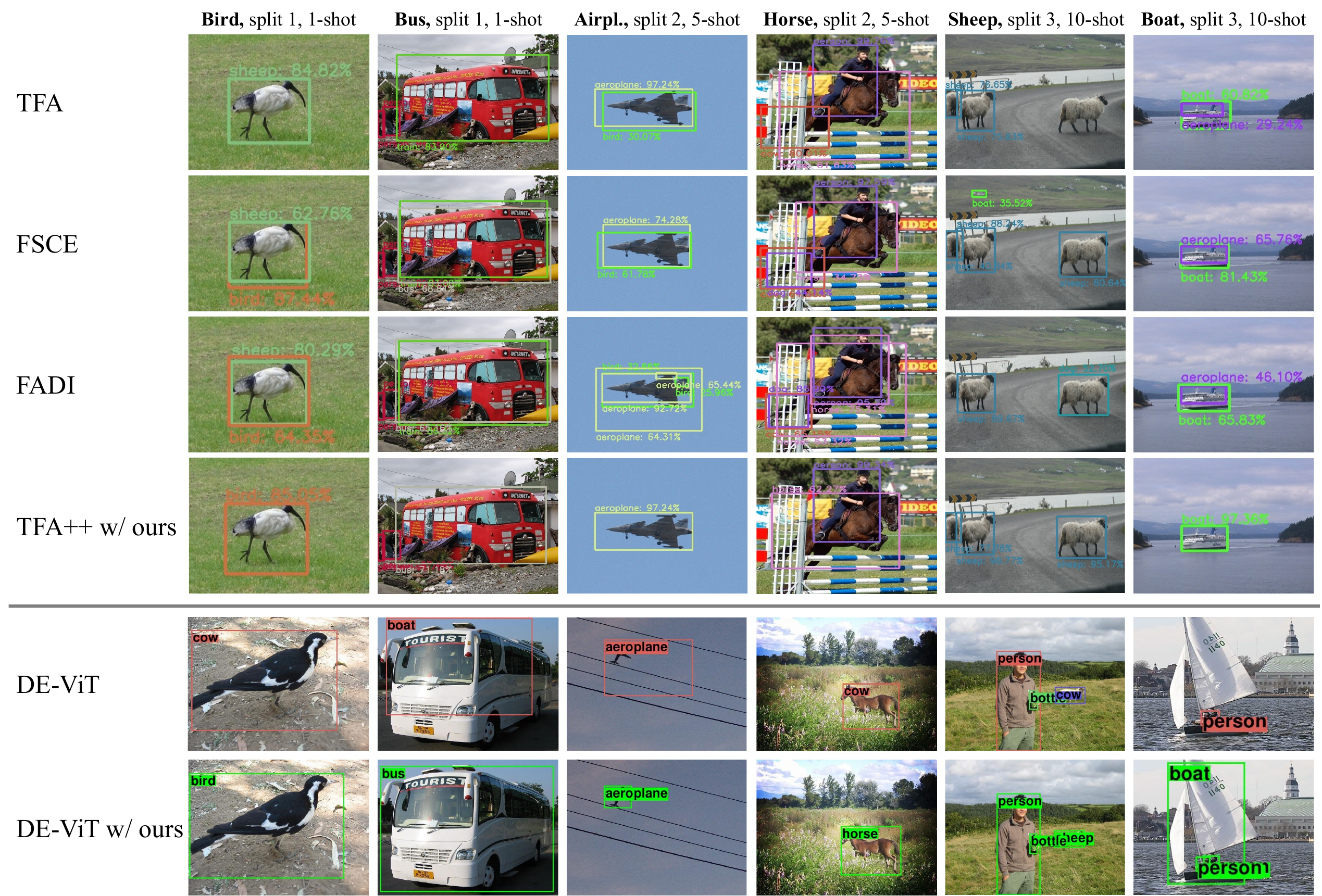}  
        \put (5,59.4) {\tiny \cite{wang2020frustratingly}}
        \put (6,48.8) {\tiny \cite{sun2021fsce}}
        \put (6,38.2) {\tiny \cite{cao2021few}}
        \put (7.8,15.5) {\tiny \cite{zhang2023detect}}
    \end{overpic} 
    \caption{Detection results of TFA, FSCE, FADI, DE-ViT, and our method in the PASCAL VOC dataset with different data splits and shots.}
    \vspace{-14pt}
    \label{Visualization}
\end{figure}
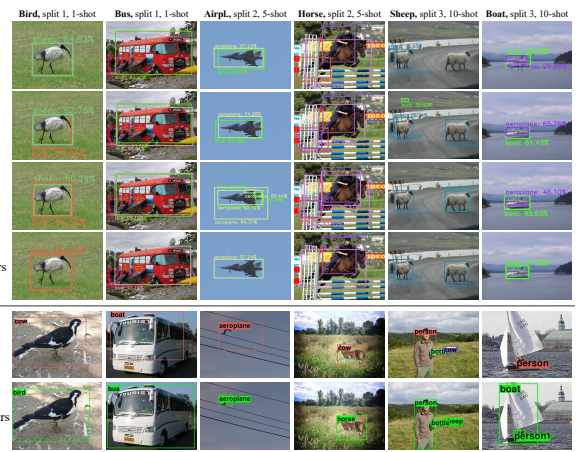

\subsubsection{MS COCO Results} The evaluation results of MS COCO are shown in Table \ref{COCO_benchmark}, we use DeFRCN~\cite{qiao2021defrcn}, MFDC~\cite{wu2022multi}, and DE-ViT~\cite{zhang2023detect} as the baselines. Given that the MS COCO dataset contains 80 categories, more than PASCAL VOC, we assigned the 10 most similar counter categories to each category in our counterfactual data augmentation method, i.e., $k_e=10$. We adopt COCO-style $AP$ and $AP_{75}$ as the evaluation metric. We compare our method with the SOTA methods, and we achieve competitive performance. From the experiment result, we can observe that our model has achieved better performance. Our method achieves improvement over all the baselines. On DeFRCN, there is a $1.6\%$ and $1.3\%$ increase in $AP$, along with a $1.7\%$ and $3.1\%$ increase in $AP_{75}$ for 10 shots and 30 shots, respectively. On MFDC, there is a $3.6\%$ and $0.9\%$ increase in $AP$, along with a $0.5\%$ and $0.9\%$ increase in $AP_{75}$ for 10 shots and 30 shots, respectively. On DE-ViT, the method improves the average $AP$ by $2.2\%$ and $AP_{75}$ by $2.0\%$. The method on the MS COCO dataset further demonstrates the effectiveness of our method. Given that MS COCO features multiple novel categories similar to the same base category, the association-based FADI method struggles to perform effectively. In contrast, our method is capable of learning a robust representation space despite these challenges. Furthermore, the large-scale MS COCO dataset creates a significant gap between its visual and textual semantic spaces, hindering the effectiveness of both alignment-based (KD and MM-FSOD) and generation-based (Norm-VAE) methods. However, the advantages of our method become particularly evident in improving the 10-shot $AP$ metric. In the 30-shot setting, our method performs worse than FM-FSOD~\cite{han2024few}, possibly due to the large language model benefiting from richer pre-training content and a relatively sufficient number of training samples. However, under the scarce 10-shot setting, our method demonstrates a significant advantage, outperforming FM-FSOD by 24.2\% and 24.6\% in $AP$ and $AP_{75}$, respectively. Fig.~\ref{coco_visualization} illustrates the detection results of DE-ViT and our method for novel categories on the MS COCO dataset, showing how the method contributes to more accurate object detection and classification.

\begin{table}[!t]
    \centering
    \caption{Few-shot object detection evaluation results on MS COCO. The evaluation metric adopts $AP$ and $AP_{75}$.}
    \vspace{-10pt}
    \resizebox{0.48 \textwidth}{!}{
        \begin{tabular}{cccc|cccc}
        \toprule
        \multirow{2}{*}{Method} & \multirow{2}{*}{Paper Year} & \multirow{2}{*}{Backbone} & \multirow{2}{*}{Side Information} & \multicolumn{2}{c}{10-shot} & \multicolumn{2}{c}{30-shot} \\ \cmidrule(lr){5-6} \cmidrule(l){7-8} 
         &   &   &  & $AP$ & $AP_{75}$ & $AP$ & $AP_{75}$ \\ \midrule
        SRR-FSD~\cite{zhu2021semantic} & CVPR 21 & ResNet-101 & Word2Vec~\cite{mikolov2013distributed} & 11.3 & 9.8 & 14.7 & 13.5 \\
        FADI~\cite{cao2021few} & NeurIPS 21 & ResNet-101 & WordNet~\cite{miller1995wordnet} & 12.2 & 11.9 & 16.1 & 15.8 \\ 
        Meta Faster R-CNN~\cite{han2022meta} & AAAI 22 & ResNet-101 & N/A & 12.7 & 10.8 & 16.6 & 15.8 \\
        Meta-DETR~\cite{zhang2022meta} & TPAMI 22 & ResNet-101 & N/A & 19.0 & 19.7 & 22.2 & 22.8 \\
        KFSOD~\cite{zhang2022kernelized} & CVPR 22 & ResNet-101 & N/A & 18.5 & 18.7 & - & - \\
        FCT~\cite{han2022few} & CVPR 22 & PVTv2-B2-Li & N/A & 17.1 & 17.0 & 21.4 & 22.1 \\
        VFA~\cite{han2023vfa} & AAAI 23 & ResNet-101 & N/A & 16.2 & - & 18.9 & - \\
        FS-DETR~\cite{bulat2023fs} & ICCV 23 & ResNet-50 & N/A & 11.3 & 11.1 & - & - \\ 
        DeFRCN~\cite{qiao2021defrcn} & ICCV 21 & ResNet-101 & ImageNet~\cite{deng2009imagenet} & 18.6 & 17.6 & 22.5 & 22.3 \\
        PTF+KI~\cite{yang2022efficient} & TIP 22 & ResNet-101 & ImageNet~\cite{deng2009imagenet} & 16.9 & 16.7 & 20.7 & 20.4 \\
        MFDC~\cite{wu2022multi} & ECCV 22 & ResNet-101 & ImageNet~\cite{deng2009imagenet} & 19.4 & 20.2 & 22.7 & 23.2 \\
        NIFF-DeFRCN~\cite{guirguis2023niff} & CVPR 23 & ResNet-101 & ImageNet~\cite{deng2009imagenet} & 18.0 & - & 20.0 & - \\ 
        KD~\cite{pei2022few} & ECCV 22 & ResNet-101 & ImageNet~\cite{deng2009imagenet}, PPC~\cite{xie2021propagate} & 18.9 & 17.8 & 22.6 & 22.6 \\
        Norm-VAE~\cite{xu2023generating} & CVPR 23 & ResNet-101 & ImageNet~\cite{deng2009imagenet}, Word2Vec~\cite{mikolov2013distributed} & 18.7 & 17.8 & 22.5 & 22.4 \\
        MM-FSOD~\cite{han2022multimodal} & ArXiv 23 & ResNet-101 & ImageNet~\cite{deng2009imagenet}, CLIP~\cite{radford2021learning} & 18.7 & 17.7 & 22.5 & 22.2 \\
        FPD~\cite{wang2024fine} & AAAI 24 & ResNet-101 & N/A & 16.5 & - & 20.1 & - \\
        T-GSEL~\cite{zhang2025learning} & IJCV 25 & PVTv2-B2-Li & N/A & 18.0 & 17.7 & 22.3 & $\mathbf{23.5}$ \\
        \rowcolor{gray!20} 
        DeFRCN w/ ours & Our Method & ResNet-101 & ImageNet~\cite{deng2009imagenet}, Visual Attribute & 18.9 & 17.9 & 22.8 & 23.0 \\ \rowcolor{gray!20}
        MFDC w/ ours & Our Method & ResNet-101 & ImageNet~\cite{deng2009imagenet}, Visual Attribute & $\mathbf{20.1}$ & $\mathbf{20.3}$ & $\mathbf{22.9}$ & 23.4 \\
        \midrule \midrule 
        DE-ViT~\cite{zhang2023detect} (Top-k=6) & CoRL 24 & ViT-L/14 & LVD-142M~\cite{oquab2023dinov2} & 33.6 & 36.6 & 33.8 & 36.8 \\ 
        SMILe~\cite{majee2024smile} & ECCV 24 & ViT-B & ImageNet~\cite{deng2009imagenet} & 25.8 & 26.1 & 31.0 & 33.6 \\
        FM-FSOD~\cite{han2024few} & CVPR 24 & ViT-L/14 & LVD-142M~\cite{oquab2023dinov2}, Vicuna~\cite{chiang2023vicuna} & 27.7 & 30.1 & $\mathbf{37.0}$ & $\mathbf{39.7}$ \\
        \rowcolor{gray!20}
        DE-ViT w/ ours & Our Method & ViT-L/14 & LVD-142M~\cite{oquab2023dinov2} & $\mathbf{34.4}$ & $\mathbf{37.5}$ & 34.5 & 37.4 \\
        \bottomrule
        \end{tabular}
    }
    \label{COCO_benchmark}
\end{table}


\begin{figure}[!t]
    \centering
    \setlength{\abovecaptionskip}{0.cm}
    \begin{overpic}[width=0.48\textwidth,tics=8]{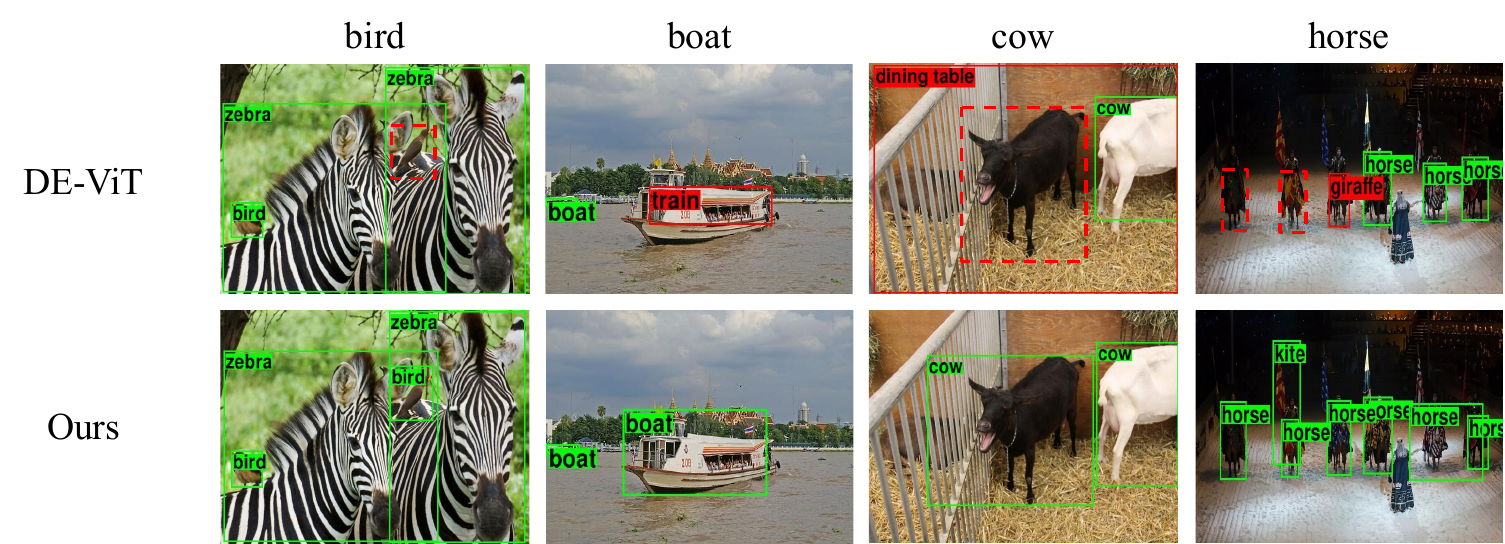}  
        \put (9.6,23.5) {\tiny \cite{zhang2023detect}}
    \end{overpic} 
    \caption{Detection results of DE-ViT and our method in MS COCO dataset under the 30-shot setup. The red bounding boxes indicate incorrect detection results, while the green bounding boxes indicate correct detection results.}
    \label{coco_visualization}
\end{figure}

\subsubsection{LVIS V1 Results} We integrate our method with the recent open-source baseline DE-ViT~\cite{zhang2023detect}, employing Faster R-CNN as the base detector. The evaluation results of LVIS V1 are shown in Table~\ref{LVIS_benchmark}, where the average precision $AP_{r}$ of the rare category of LVIS serves as the evaluation metric. Our method enhances the baseline by 1.6 $AP_{r}$, underscoring its efficacy. This further demonstrates that our method is capable of delivering strong results, even when faced with a large number of novel categories to detect. Fig.~\ref{lvis_visualization} shows detection outcomes of scenes featuring rare categories. While DE-ViT may encounter difficulties in distinguishing similar objects, such as mistaking a sawhorse for a cone or a mallet for a walking cane, our method effectively rectifies these detection errors. This demonstrates that, through fine-tuning and learning discriminative features, our method is capable of accurately distinguishing between closely similar objects in more challenging scenarios.

\begin{table}[!t]
    \centering
    \caption{Few-shot object detection evaluation results on LVIS V1~\cite{gupta2019lvis}. The evaluation metric adopts $AP_{r}$ on rare categories.}
    \vspace{-10pt}
    \resizebox{0.48\textwidth}{!}{
        \begin{tabular}{cccc|c}
        \toprule
        Method & Paper Year & Backbone & Base Detector & $AP_{r}$ \\ \midrule
        Baseline~\cite{gupta2019lvis} & CVPR 19 & ResNet-101 & Faster R-CNN & 13.1  \\
        TFA w/ fc~\cite{wang2020frustratingly} & ICML 20 & ResNet-101 & Faster R-CNN & 15.5 \\
        TFA w/ cos~\cite{wang2020frustratingly} & ICML 20 & ResNet-101 & Faster R-CNN & 17.3 \\
        DiGeo+TFA~\cite{ma2023digeo}  & CVPR 23 & ResNet-101 & Faster-RCNN & 18.5 \\
        EQL~\cite{tan2020equalization} & CVPR 20 & ResNet-101 & Mask R-CNN & 14.6\\
        BAGS~\cite{li2020overcoming} & CVPR 20 & ResNeXt-101 &  Faster R-CNN & 18.8 \\
        ACSL~\cite{wang2021adaptive}  & CVPR 21 & ResNet-101 & Faster R-CNN &  19.3 \\
        DiGeo+ACSL~\cite{ma2023digeo}  & CVPR 23 & ResNet-101 & Faster R-CNN &  19.5 \\
        Grounding DINO (Zero-Shot)~\cite{liu2023grounding} & ArXiv 23 & Swin-T & DINO & 18.1 \\ 
        DE-ViT~\cite{zhang2023detect} & CoRL 24 & ViT-L/14 & Faster R-CNN & 32.2 \\ \rowcolor{gray!20}
        DE-ViT w/ ours & Our Method & ViT-L/14 & Faster R-CNN & $\mathbf{33.8}$  \\
        \bottomrule
        \end{tabular}
    }
    \label{LVIS_benchmark}
\end{table}

\begin{figure}[!t]
    \centering
    \setlength{\abovecaptionskip}{0.cm}
    \begin{overpic}[width=0.48\textwidth,tics=8]{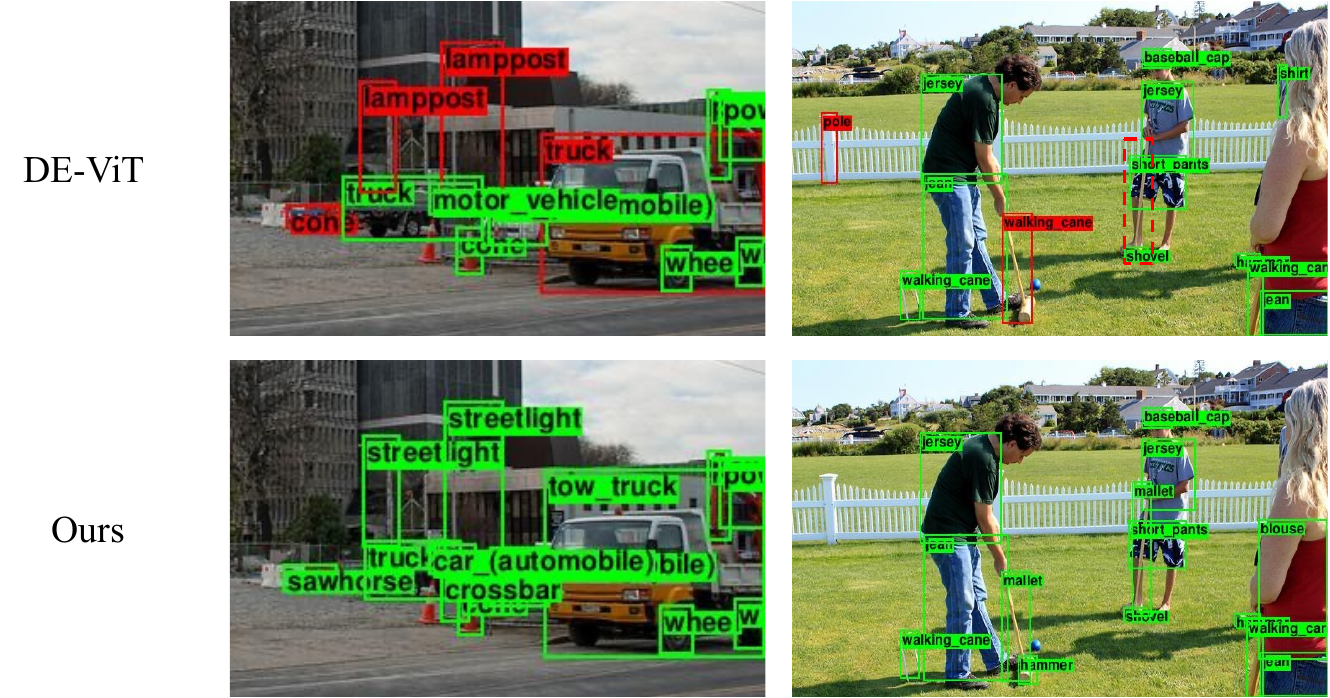}  
        \put (11,38.9) {\scriptsize \cite{zhang2023detect}}
    \end{overpic} 
    \caption{Detection results of DE-ViT and our method in LVIS V1 dataset. The red bounding boxes indicate incorrect detection results, while the green bounding boxes indicate correct detection results.}
    \label{lvis_visualization}
\end{figure}

\subsubsection{FSOD-1K Results} 
We employ DeFRCN~\cite{qiao2021defrcn} and DE-ViT~\cite{zhang2023detect} as baseline models. Due to limitations in video memory, the PCB module within DeFRCN is excluded from our implementation. Since there is no officially released support set for the FSOD-1K dataset, we adopt a 5-shot setup, consistent with other studies~\cite{fan2020few,zhang2022time}. In this setup, 5 instances for each category in the test set are randomly selected for fine-tuning, while the remaining samples are used for evaluation. This experiment is conducted five times to mitigate variability, and the mean of these trials is reported as the final result. It is crucial to highlight that the categories subjected to fine-tuning are strictly novel categories, excluding any base categories. Our evaluation metric is $AP_{50}$. The performance of our method on the FSOD-1K dataset, as shown in Table~\ref{FSOD1k_benchmark}, is compared against other meta-learning-based approaches, showcasing the SOTA results. Specifically, our method registers a $4.0\%$ improvement on DeFRCN and a $17.0\%$ improvement on DE-ViT. These enhancements underscore the efficacy of our approach within the FSOD-1K benchmark. Due to the high annotation quality of the FSOD-1K dataset, DE-ViT outperforms all methods that utilize a ResNet backbone without any fine-tuning. By subsequently applying fine-tuning in conjunction with our method, we achieved significant enhancements. This indicates our method's capability to learn more discriminative representations from data consisting exclusively of novel categories.
Fig.~\ref{fsod1k_visualization} illustrates the detection outcomes for novel categories on the FSOD-1K dataset using DE-ViT and our method, showcasing our method's superior precision in object classification.

\begin{table}[!t]
    \centering
    \caption{Few-shot object detection evaluation results on FSOD-1K~\cite{fan2020few} 5-shot test set. The evaluation metric adopts $AP_{50}$.}
    \vspace{-10pt}
    \resizebox{0.48 \textwidth}{!}{
        \begin{tabular}{cccc|c}
            \toprule
            Method & Paper Year & Backbone & Base Detector & $AP_{50}$ \\ \midrule
            FRCN~\cite{ren2016faster} & TPAMI 16 & ResNet & Faster R-CNN & 23.0 \\
            LSTD~\cite{chen2018lstd} & AAAI 18 & ResNet-50 & Faster R-CNN & 24.2 \\
            FSOD~\cite{fan2020few} & CVPR 20  & ResNet-50 & Faster R-CNN & 27.5 \\
            PNSD~\cite{zhang2020few}  & ACCV 20 & ResNet-50 & Faster R-CNN & 29.8 \\
            QSAM~\cite{lee2022few} & WACV 22 & ResNet-101 & Faster R-CNN & 30.7 \\
            KFSOD~\cite{zhang2022kernelized} & CVPR 22 & ResNet-50 & Faster R-CNN & 33.4 \\
            TENET~\cite{zhang2022time} & ECCV 22 & ResNet-50 & Faster R-CNN & 35.4 \\
            $\sigma$-ADP~\cite{du2023adaptive} & ICCV 23 & ResNet-101 & Faster R-CNN & 36.9 \\ 
            FSODv2~\cite{fan2024fsodv2} & IJCV 24 & ResNet-50 & Faster R-CNN & 30.9 \\
            DeFRCN~\cite{qiao2021defrcn} & ICCV 21 & ResNet-101 & Faster R-CNN & 37.8 \\ 
            \rowcolor{gray!20}
            DeFRCN w/ ours & Our Method & ResNet-101 & Faster R-CNN & $\mathbf{39.3}_{\pm 0.4}$ \\ 
            \midrule \midrule 
            DE-ViT~\cite{zhang2023detect} & CoRL 24 & ViT-L/14 & Faster R-CNN & 40.0 \\ \rowcolor{gray!20}
            DE-ViT w/ ours & Our Method & ViT-L/14 & Faster R-CNN & $\mathbf{46.8}_{\pm 0.3}$ \\
                   \bottomrule
            \end{tabular}
    }
    \label{FSOD1k_benchmark}
\end{table}

\begin{figure}[!t]
    \centering
    \setlength{\abovecaptionskip}{0.cm}
    \begin{overpic}[width=0.48\textwidth,tics=8]{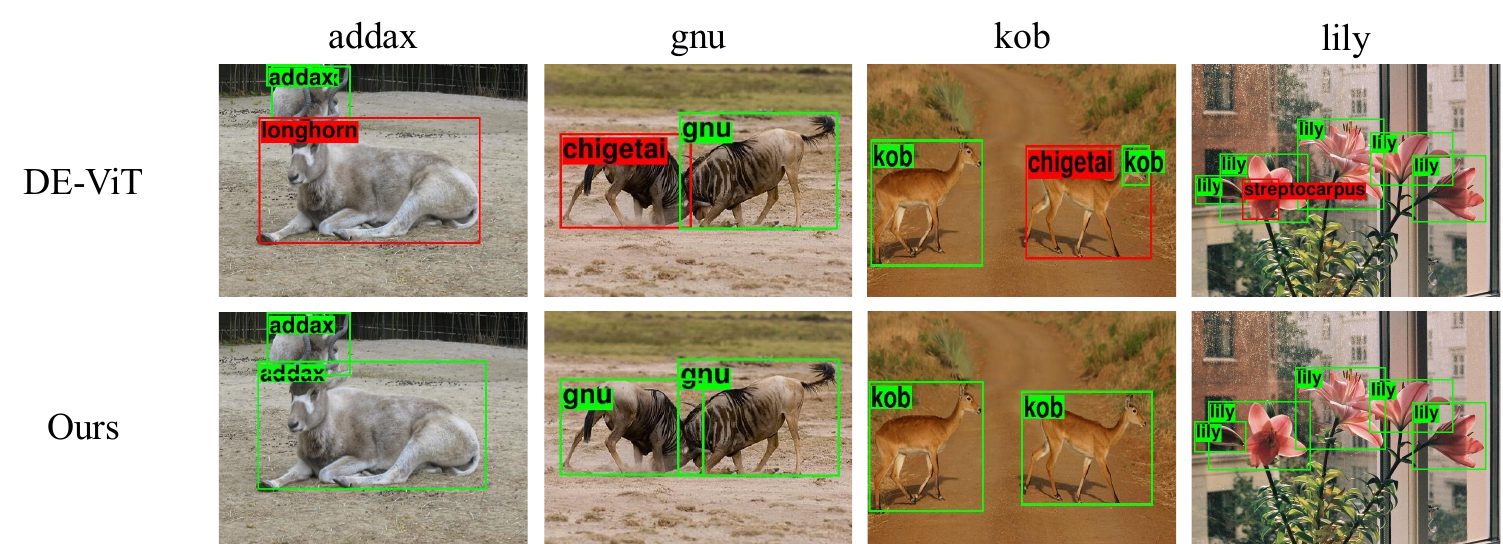}  
        \put (9.6,23.5) {\tiny \cite{zhang2023detect}}
    \end{overpic} 
    \caption{Detection results of DE-ViT and our method in FSOD-1K dataset. The red bounding boxes indicate incorrect detection results, while the green bounding boxes indicate correct detection results.}
    \label{fsod1k_visualization}
\end{figure}

\subsubsection{FSVOD-500 Results} We employ DeFRCN~\cite{qiao2021defrcn} (excluding the PCB module) and DE-ViT~\cite{zhang2023detect} as baseline models, and use $AP$, $AP_{50}$, and $AP_{75}$ as the evaluation metrics. Similar to the FSOD-1K benchmark, the fine-tuning phase only involves novel categories, with no inclusion of base categories. The evaluation results of the FSVOD dataset are shown in Table~\ref{FSVOD_benchmark}, indicating that our method achieves state-of-the-art (SOTA) performance. Specifically, with DeFRCN, our methodology realizes improvements of 1.7 in $AP$, 0.9 in $AP_{50}$, and 1.4 in $AP_{75}$. DE-ViT underperforms compared to DeFRCN without fine-tuning, due to the FSVOD dataset's base set consisting of video clips. This composition offers less sample diversity than seen in datasets like MS COCO, LVIS V1, and FSOD-1K, which in turn limits the effectiveness of novel category inference. However, by incorporating our method to refine the discriminative feature space, we significantly improve the performance, with our method improving the performance by 9.7 in $AP$, 12.0 in $AP_{50}$, and 11.3 in $AP_{75}$, and achieving SOTA results. 
These results highlight the effectiveness of our approach on the FSVOD-500 benchmark.
Fig.~\ref{fsvod_visualization} displays the detection results for novel categories within the FSVOD-500 dataset, comparing the use of DE-ViT and our method. This visual representation demonstrates the superior object classification accuracy afforded by our method.




\subsection{Visual Analysis by Saliency Maps}

We now use the saliency maps to analyze the effect of each component, the base model is TFA++. We use Grad-CAM~\cite{selvaraju2020grad} to calculate the saliency maps. Fig.~\ref{cam} presents the saliency maps of the novel categories on the PASCAL VOC dataset split 2 under the 10-shot setup. We incrementally add components of prototype clustering, knowledge matrix in CCL, and counterfactual data augmentation to the base model to observe the model's effect. For fairness, we select the images of the novel categories that are accurately classified and localized by each model.

\begin{table}[!t]
    \centering
    \caption{Few-shot object detection evaluation results on FSVOD-500~\cite{fan2022few_video} under 5-shot setup. The evaluation metric adopts $AP$, $AP_{50}$, and $AP_{75}$.}
    \resizebox{0.48 \textwidth}{!}{
        \begin{tabular}{cccc|ccc}
        \toprule
        Method & Paper Year & Backbone & Base Detector & $AP$ & $AP_{50}$ & $AP_{75}$ \\ \midrule
        FRCN~\cite{ren2016faster}  & TPAMI 16 & ResNet & Faster R-CNN & 18.2 & 26.4 & 19.6 \\
        FSOD~\cite{fan2020few} & CVPR 20 & ResNet & Faster R-CNN & 21.1 & 31.3 & 22.6 \\ 
        MEGA~\cite{chen2020memory} & CVPR 20 & ResNet & Faster R-CNN & 16.8 & 26.4 & 17.7 \\ 
        RDN~\cite{deng2019relation} & ICCV 19 & ResNet & Faster R-CNN & 18.2 & 27.9 & 19.7 \\ 
        CTracker~\cite{peng2020chained} & ECCV 20 & ResNet & Faster R-CNN & 20.1 & 30.6 & 21.0  \\ 
        FairMOT~\cite{zhang2021fairmot} & IJCV 21 & ResNet & CenterNet & 20.3 & 31.0 & 21.2 \\ 
        CenterTrack~\cite{zhou2020tracking} & ECCV 20 & DLA & CenterNet & 20.6 & 30.5 & 21.9 \\ 
        FSVOD~\cite{fan2022few_video} & ECCV 22 & ResNet & - & 25.1 & 36.8 & 26.2 \\ 
        DeFRCN~\cite{qiao2021defrcn} & ICCV 21 & ResNet & Faster R-CNN & 42.5 & 63.6 & 47.8 \\
        \rowcolor{gray!20}
        DeFRCN w/ ours & Our Method & ResNet & Faster R-CNN & $\mathbf{44.2}$ & $\mathbf{64.5}$ & $\mathbf{49.2}$ \\ 
        \midrule \midrule
        DE-ViT~\cite{zhang2023detect} & CoRL 24 & ViT-L/14 & Faster R-CNN & 37.7 & 54.9 & 39.8 \\ \rowcolor{gray!20}
        DE-ViT w/ ours & Our Method & ViT-L/14 & Faster R-CNN & $\mathbf{47.4}$ & $\mathbf{66.9}$ & $\mathbf{51.1}$ \\
        \bottomrule
        \end{tabular}
    }
    \label{FSVOD_benchmark}
\end{table}

\begin{figure}[!t]
    \centering
    \setlength{\abovecaptionskip}{0.cm}
    \begin{overpic}[width=0.48\textwidth,tics=8]{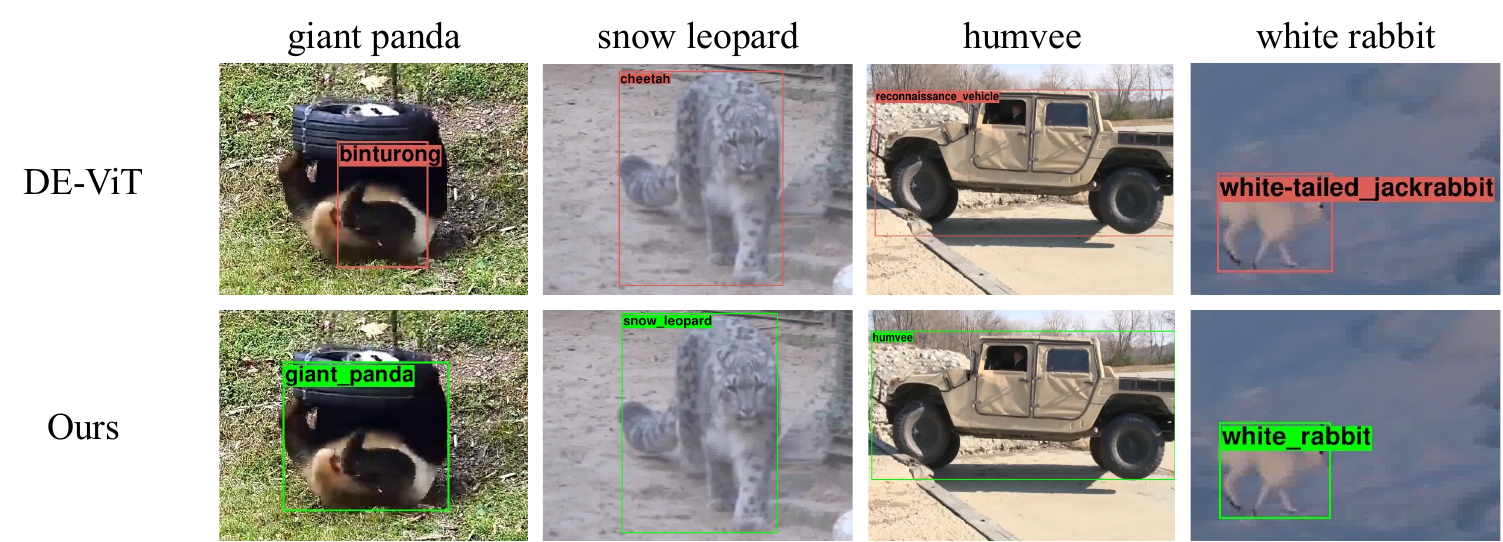}  
        \put (9.6,23.5) {\tiny \cite{zhang2023detect}}
    \end{overpic} 
    \caption{Detection results of DE-ViT and our method in FSVOD-500 dataset. The red bounding boxes indicate incorrect detection results, while the green bounding boxes indicate correct detection results.}
    \label{fsvod_visualization}
\end{figure}

In Fig.~\ref{cam}, we find that after adding the prototype clustering component, the salient regions of the model's decision are more focused on objects. For instance, upon adding the prototype clustering component, the model emphasizes the region of the airplane and reduces its focus on the sky. Similarly, it focuses on the head area of the cow rather than just the limbs, and the horse's body rather than the saddle or human legs. This can be attributed to the fact that clustering the prototypes enhances their representative characteristics and mitigates the bias introduced by a single prototype, as seen in the case of the airplane and the sky. Therefore, it can be inferred that the prototype clustering component facilitates the estimation of unbiased prototypical representations. After introducing a knowledge matrix in CCL to enhance the distinction of similar categories, the model is capable of focusing on more informative regions. For instance, the model is more attentive to the airplane's fuselage, as well as the heads of cows and horses. After introducing counterfactual data augmentation, the model focuses on regions more accurately. These observations provide further evidence of the efficacy of the individual components of our approach.

\begin{figure}[!t]
    \centering
    \setlength{\abovecaptionskip}{0.cm}
    \includegraphics[width = 0.48 \textwidth]{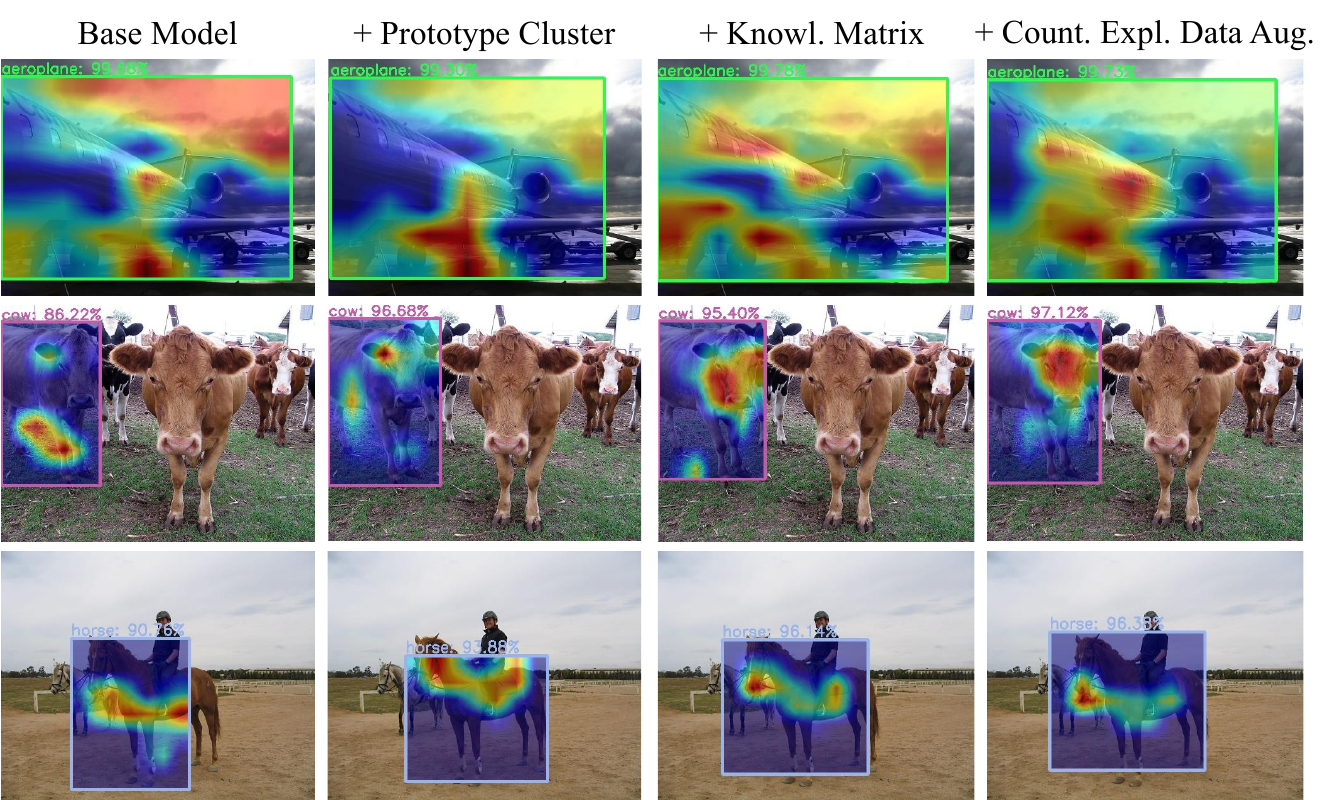}
    \caption{Saliency map visualization of different model components in PASCAL VOC dataset split 2 under the 10-shot setup. Please zoom in for better visualization.}
    \label{cam}
\end{figure}

\subsection{Visualization of CounterFactual Explanation}

We visualize saliency maps of counterfactual interpretability computed by our method during training, with images from the few-shot training set, the base model is TFA++. As shown in Fig. \ref{counter_cam}, we find that the saliency maps assigned to different negative classes focus on different regions. For example, the current model's discriminative region for distinguishing between birds and dogs is the wings, whereas for distinguishing between birds and sheep, it's the paws.
By utilizing this interpretable method, we can gain insights into the key features that the model utilizes to differentiate between the two categories. This not only enhances interpretability but also boosts the model's performance.

\begin{figure}[!t]
    \centering
    \setlength{\abovecaptionskip}{0.cm}
    \includegraphics[width = 0.48 \textwidth]{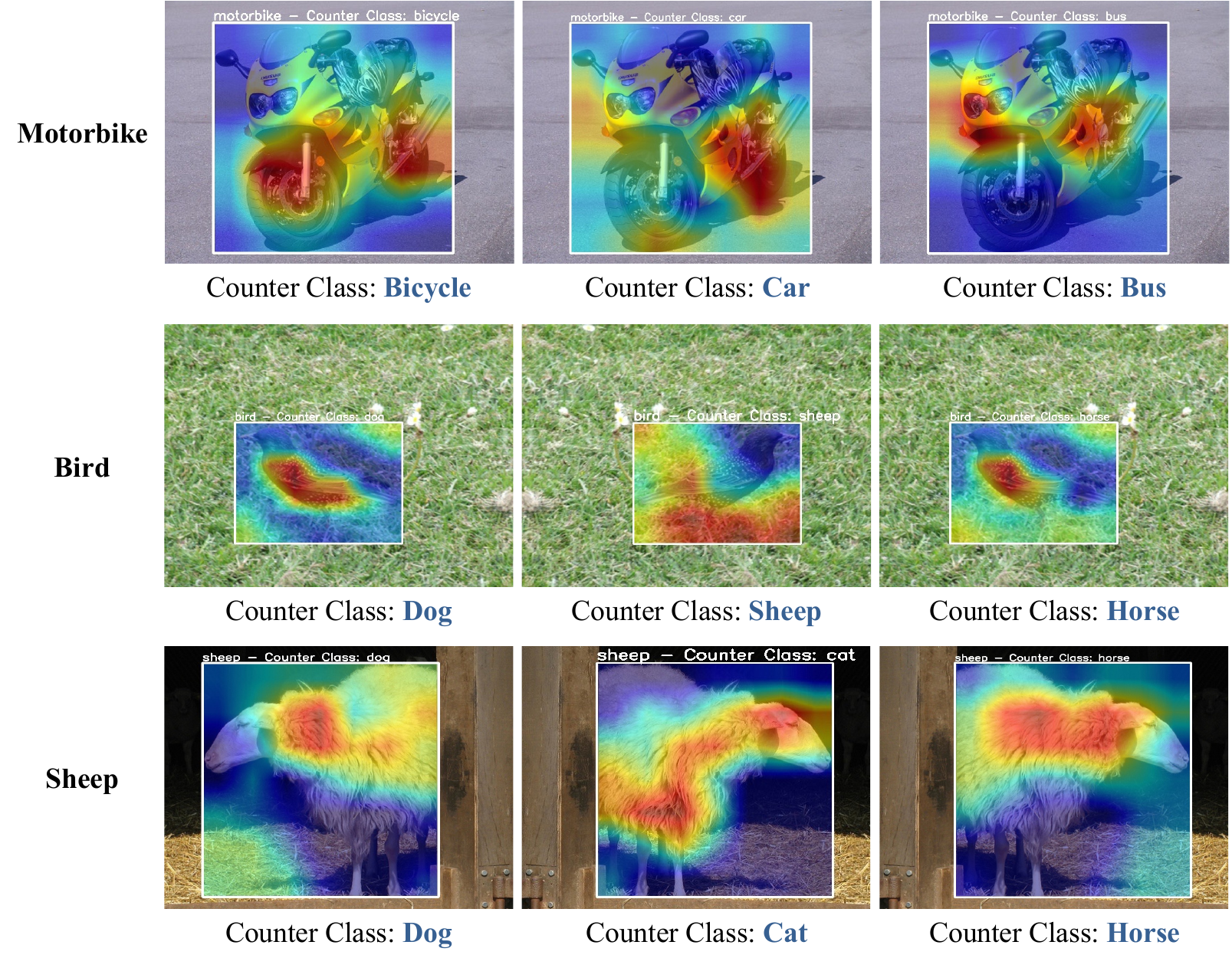}
    \caption{Visualization saliency maps of counterfactual explanation during training. The image flipping in the figure is the default weak data augmentation.}
    \label{counter_cam}
\end{figure}

\section{Conclusion} \label{conclusion}
In this paper, we propose a novel few-shot object detection with side information. Specifically, we introduce a novel contextual semantic supervised contrastive learning module that encodes the visual attribute semantic relationship between base categories and the novel categories. The detector can be guided to strengthen the discrimination between semantically similar categories and improve the separability of the feature space. Furthermore, we proposed a side information guided counterfactual data augmentation method, which can transparently identify the discriminative region of different objects during model training, and can also reduce the generalization error.
Extensive experiments using ResNet and ViT backbones conducted on PASCAL VOC, MS-COCO, LVIS V1, FSOD-1K, and FSVOD-500 benchmarks demonstrate that our model achieves new state-of-the-art performance, significantly improving the ability of FSOD in most of the shot/split.

\appendices
%
%
%
%
%

\ifCLASSOPTIONcaptionsoff
  \newpage
\fi



%
\bibliographystyle{IEEEtran}      
\footnotesize
\bibliography{egbib}

%
\vspace{-10 mm}

\begin{IEEEbiography}[{\includegraphics[width=1in]{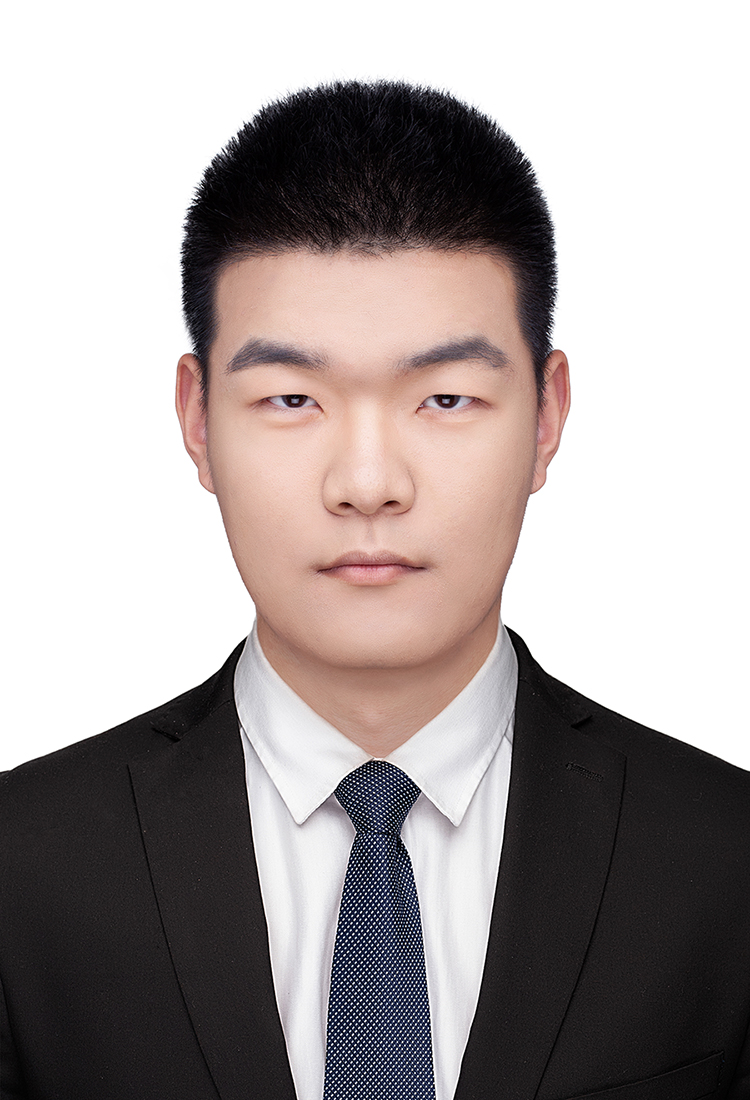}}]{Ruoyu Chen}
    is currently working toward the Ph.D. degree in the School of Cyber Security, University of Chinese Academy of Sciences, China. He received his B.E. degree in measurement \& control technology and instrument from Northeastern University, China in 2021. He has published multiple top journals and conference papers, such as ICLR and CVPR. He has served as a reviewer for several top journals and conferences such as T-PAMI, ECCV, CVPR, ICML, ICCV, ICLR, and NeurIPS. His research interests mainly include computer vision, object detection, and interpretable AI.
\end{IEEEbiography}

\vspace{-10 mm}

\begin{IEEEbiography}[{\includegraphics[width=1in]{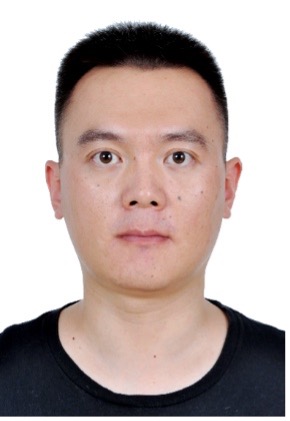}}]{Hua Zhang} is a professor with the Institute of Information Engineering, Chinese Academy of Sciences. He received the Ph.D. degree in computer science from the School of Computer Science and Technology, Tianjin University, Tianjin, China in 2015. His research interests include computer vision, multimedia, and machine learning.
\end{IEEEbiography}

\vspace{-10 mm}

\begin{IEEEbiography}[{\includegraphics[width=1in]{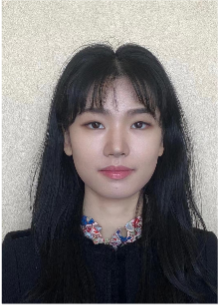}}]{Jingzhi Li} is currently an associate professor with the Institute of Information Engineering, Chinese Academy of Sciences, Beijing, China. She received the Ph.D. degree in cyberspace security from the University of Chinese Academy of Sciences, Beijing, China. Her current research interests include image processing, face recognition security, and multimedia privacy.
\end{IEEEbiography}

\vspace{-10 mm}

\begin{IEEEbiography}[{\includegraphics[width=1in]{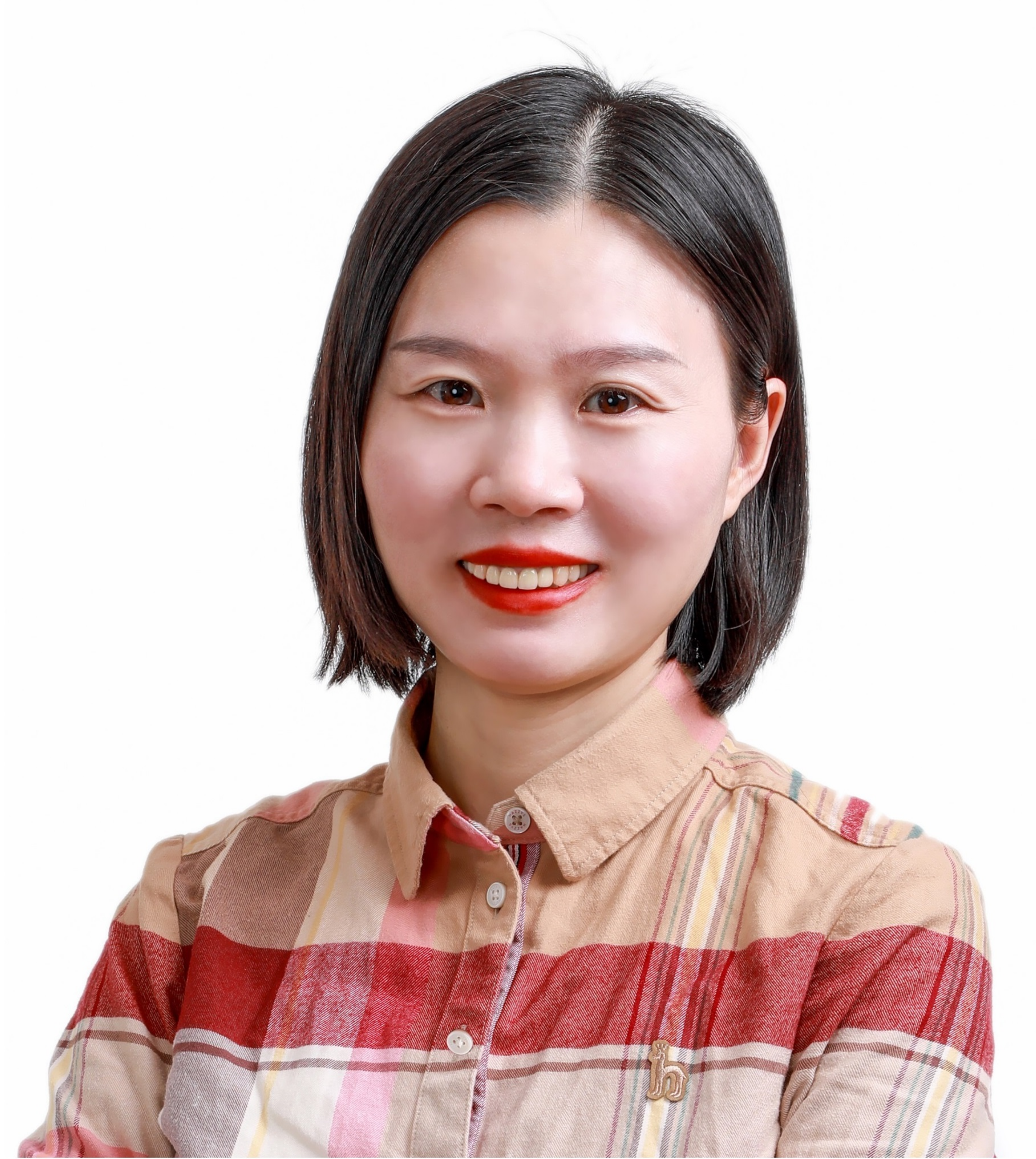}}]{Li Liu}(Senior Member, IEEE) received the Ph.D. degree in information and communication engineering from the National University of Defense Technology (NUDT), China, in 2012.
    She is currently a Full Professor with the College of System Engineering, National University of Defense Technology.
     During her Ph.D. study, she spent more than two years as a Visiting Student at the University of Waterloo, Canada, from 2008 to 2010. From 2015 to 2016, she spent ten months visiting the Multimedia Laboratory at the Chinese University of Hong Kong. From 2016.12 to 2018.11, she worked as a senior researcher at the Machine Vision Group at the University of Oulu, Finland. Her current research interests include Computer Vision, Machine Learning, Artificial Intelligence, Trustworthy AI, and Synthetic Aperture Radar. Her papers have currently over 13,100+ citations in Google Scholar.
\end{IEEEbiography}

\vspace{-10 mm}

\begin{IEEEbiography}[{\includegraphics[width=1in]{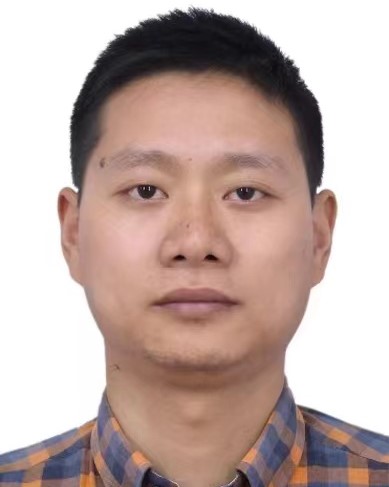}}]{Zhen Huang} received the B.S. and Ph.D. degrees from the National University of Defense Technology (NUDT), Changsha, China, in 2006
and 2012, respectively. He is currently a professor with the National Key Laboratory of Parallel and Distributed Computing, NUDT. His research interests include machine learning algorithms, intelligent systems, and knowledge mining. He has published over 100 papers including top international journals, such as IEEE/ACM Transactions, and international conferences including AAAI, IJCAI, MM, ACL, SIGIR, WWW and others.
\end{IEEEbiography}

\vspace{-10 mm}

\begin{IEEEbiography}[{\includegraphics[width=1in]{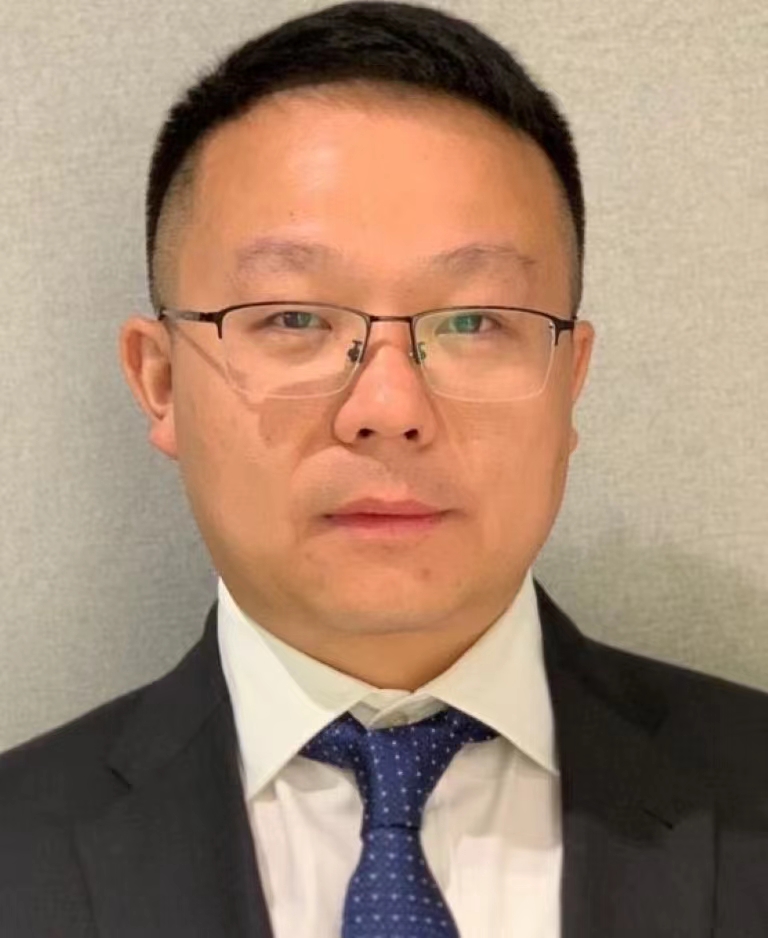}}]{Xiaochun Cao}(Senior Member, IEEE)
    is a Professor and Dean of School of Cyber Science and Technology, Shenzhen Campus of Sun Yat-sen University. He received the B.E. and M.E. degrees both in computer science from Beihang University (BUAA), China, and the Ph.D. degree in computer science from the University of Central Florida, USA, with his dissertation nominated for the university level Outstanding Dissertation Award. After graduation, he spent about three years at ObjectVideo Inc. as a Research Scientist. From 2008 to 2012, he was a professor at Tianjin University. Before joining SYSU, he was a professor at Institute of Information Engineering, Chinese Academy of Sciences. He has authored and coauthored over 200 journal and conference papers. In 2004 and 2010, he was the recipients of the Piero Zamperoni best student paper award at the International Conference on Pattern Recognition. He is on the editorial boards of IEEE \textsc{Transactions on Pattern Analysis and Machine Intelligence} and IEEE \textsc{Transactions on Image Processing}, and was on the editorial boards of IEEE \textsc{Transactions on Circuits and Systems for Video Technology} and IEEE \textsc{Transactions on Multimedia}.
\end{IEEEbiography}




\newpage

\begin{appendices}

\section{Generalization Error Bound Analysis} \label{error_bound}

We perform a theoretical analysis of our method in terms of generalization error bounds. The expected risk $\mathcal{R}(\cdot)$ and empirical risk $\hat{\mathcal{R}}(\cdot)$ of the few-shot object detector's classifier $f$ on dataset $\mathcal{D}$ are defined as:
\begin{gather}
    \mathcal{R}(f) = \mathbb{E}_{(\mathbf{x},y) \sim \mathcal{D}} \left [ \mathcal{L}_{\mathrm{cls}}\left (f(\mathbf{x}) , y \right) \right ],\\
    \hat{\mathcal{R}}(f) = \frac{1}{N}\sum_{i=1}^N \mathcal{L}_{\mathrm{cls}} \left (f(\mathbf{x}) , y \right) ,
\end{gather}
where $N$ denotes the number of samples in the dataset $\mathcal{D}$, and $\mathcal{L}_{\mathrm{cls}}$ denotes the classification loss function. Let $\mathcal{F}$ be a classifier function, few-shot object detection model $\hat{f}_{\mathrm{base}}$ is trained on base set $\mathcal{D}_{\mathrm{base}}$:
\begin{equation}
    \hat{f}_{\mathrm{base}} = \underset{f \in \mathcal{F}}{\arg \min} \hat{\mathcal{R}}_{b} (f),
\end{equation}
and the conventional fine-tuning model $\hat{f}_{\mathrm{novel}}$ is trained on novel set $\mathcal{D}_{\mathrm{novel}}$:
\begin{equation}
    \hat{f}_{\mathrm{novel}} = \underset{f \in \mathcal{F}}{\arg \min} \hat{\mathcal{R}}_{n} (f).
\end{equation}

Since our method introduces a memory prototype bank and CCL module, the prototype bank is continuously stored in the base category prototype, and the base category features are kept aligned with the prototype through CCL. The learning rate used to fine-tune the network is small. The fine-tuned base category feature distribution is basically consistent with the pre-trained base category feature distribution. Therefore, the detector model with CCL $\hat{f}_{\mathrm{CCL}}$ can be seen as joint fine-tuning in the base domain $\mathcal{D}_{\mathrm{base}}$ and novel domain $\mathcal{D}_{\mathrm{novel}}$:
\begin{equation}
    \hat{f}_{\mathrm{CCL}} = \underset{f \in \mathcal{F}}{\arg \min} \hat{\mathcal{R}}_{(b+n)} (f),
\end{equation}
and empirically, for $\hat{\mathcal{R}}_{(b+n)}$, given a hyperparameter $\lambda_c \in [0,1)$,  $\hat{\mathcal{R}}_{(b+n)}$ can be a convex combination of base risk and novel risk:
\begin{equation}
    \hat{\mathcal{R}}_{(b+n)}(\hat{f}_{\mathrm{CCL}}) = (1-\lambda_c) \hat{\mathcal{R}}_{b}(\hat{f}_{\mathrm{CCL}}) + \lambda_c \hat{\mathcal{R}}_{n}(\hat{f}_{\mathrm{CCL}}),
\end{equation}
thus, given the test set $\mathcal{D}_{\mathrm{test}}$, the generalization error is:
\begin{equation}
    \mathcal{R}_{t}(\hat{f}_{\mathrm{CCL}}) - \hat{\mathcal{R}}_{(b+n)}(\hat{f}_{\mathrm{CCL}}),
\end{equation}
then, based on the neural network assumption of Zhang \textit{et al.} \cite{zhang2012generalization} and the theorem of Yang \textit{et al.} \cite{yang2021bridging}, we can bound the generalization error of a network with the following lemma:

\begin{lemma}[\cite{yang2021bridging}] \label{bound}
    For any sample $\mathbf{x}$, its upper bound is $B$, i.e., $\|\mathbf{x}\|\le B$. Let $\mathbf{X}^{N_b} = \{\mathbf{x}_i^{b}\}_{i=1}^{N_b}$ and $\mathbf{X}^{N_n} = \{\mathbf{x}_i^{n}\}_{i=1}^{N_n}$ be two set of i.i.d. samples from the base set $\mathcal{D}_{\mathrm{base}}$ and the novel set $\mathcal{D}_{\mathrm{novel}}$. Assume that the detector has multiple layers, with each layer $l$ having parameters matrices $W_l$ and its Frobenius norm are at most $M_l$. The activation functions be 1-Lipschitz continuous, positive-homogeneous, and applied element-wise. Let $\mathcal{F}$ be a function class ranging from $[0, 1]$. Then, given a $\lambda_c \in [0,1)$, for any $\delta > 0$, we have at least $1-\delta$ probability that,
    \begin{equation} \label{bound_eq}
        \begin{aligned}
            \mathcal{R}_{t}\left(\hat{f}_{\mathrm{CCL}} \right) \leq & \; \hat{\mathcal{R}}_{(b+n)}\left(\hat{f}_{\mathrm{CCL}}\right)+(1-\lambda_c) \gamma_{\mathcal{F}}(\mathcal{D}_{\mathrm{base}}, \mathcal{D}_{\mathrm{novel}})   \\
            & + 2(1-\lambda_c)\hat{\mathfrak{R}}_{\mathrm{base}}(\mathcal{F}) + 3(1-\lambda_c) \sqrt{\frac{\ln (4 / \delta)}{2 N_{b}}}  \\
            & + 2 \lambda_c \hat{\mathfrak{R}}_{\mathrm{novel}}(\mathcal{F}) + 3 \lambda_c \sqrt{\frac{\ln (4 / \delta)}{2 N_{n}}} \\
            & + \sqrt{\frac{\ln (4 / \delta)}{2}\left ( \frac{(1-\lambda_c)^2}{N_b} + \frac{\lambda_c^2}{N_n} \right)},
        \end{aligned}
    \end{equation}
    where $\gamma_{\mathcal{F}}(\cdot, \cdot)$ is the integral probability metric \cite{muller1997integral} that measures the distance between two distributions. $N_b$ and $N_n$ are the sample sizes of $\mathcal{D}_{\mathrm{base}}$ and $\mathcal{D}_{\mathrm{novel}}$ respectively. $\hat{\mathfrak{R}}(\cdot) $ denotes the Rademacher Complexity \cite{bartlett2002rademacher}.
\end{lemma}

Based on Lemma \ref{bound}, we have:
\begin{theorem}
    When using the CCL module, the generalization error of the detector $\hat{f}_{\mathrm{CCL}}$ is approximately:
    \begin{equation} \label{bound_eq_appro}
        \begin{aligned}
            \mathcal{R}_{t}(\hat{f}_{\mathrm{CCL}}) \leq & \; (1-\lambda_c) \gamma_{\mathcal{F}}(\mathcal{D}_{\mathrm{base}}, \mathcal{D}_{\mathrm{novel}})\\
            &  + 2 \lambda_c \hat{\mathfrak{R}}_{\mathrm{novel}}(\mathcal{F})  + 4 \lambda_c \sqrt{\frac{\ln (4 / \delta)}{2 N_{n}}},
        \end{aligned}
    \end{equation}
    the generalization error of the model will be reduced, and the generalization error can be further reduced by introducing the knowledge matrix.
\end{theorem}

\begin{proof}
    Equation \ref{bound_eq} shows the generalization error $\mathcal{R}_{t}\left(\hat{f}_{\mathrm{CCL}} \right)$ is bounded by the empirical training risk $\hat{\mathcal{R}}_{(b+n)}\left(\hat{f}_{\mathrm{CCL}}\right)$, the domain gap $\gamma_{\mathcal{F}}(\mathcal{D}_{\mathrm{base}}, \mathcal{D}_{\mathrm{novel}}) $ and the estimation error. Among them, the empirical training risk can be minimized by the detector to arbitrary small value, and the number of samples of $N_b$ is very large, so $1/N_b$ can be approximately 0. Therefore, we can deduce Equation \ref{bound_eq_appro}. According to Lemma \ref{bound}, when CCL is not used, the detector is equivalent to optimizing only on the novel set and its generalization error is:
    \begin{equation} \label{}
        \mathcal{R}_{t}(\hat{f}_{\mathrm{novel}}) \precsim  2 \hat{\mathfrak{R}}_{\mathrm{novel}}(\mathcal{F})  + 4 \sqrt{\frac{\ln (4 / \delta)}{2 N_{n}}},
    \end{equation}
    since the few-shot object detection model maintains the performance of recognizing base categories, base set and novel set come from the same dataset, under this assumption, the domain gap $\gamma_{\mathcal{F}}(\mathcal{D}_{\mathrm{base}}, \mathcal{D}_{\mathrm{novel}})  \triangleq \sup_{f \in \mathcal{F}} |\mathcal{R}_b(f_{\mathrm{CCL}}) - \mathcal{R}_n(f_{\mathrm{CCL}})|$ is small, which can be represented by a small value $\epsilon$. Then:
    \begin{equation} \label{compare_eq}
        \mathcal{R}_{t}(\hat{f}_{\mathrm{CCL}}) \precsim \epsilon + \lambda_c \sup_{f \in \mathcal{F}} \left( \mathcal{R}_{t}(\hat{f}_{\mathrm{novel}}) \right) < \sup_{f \in \mathcal{F}} \left( \mathcal{R}_{t}(\hat{f}_{\mathrm{novel}}) \right ),
    \end{equation}
    therefore, after introducing CCL, the generalization error of the model will decrease. When the knowledge matrix is introduced, the model will focus on extremely similar base categories and novel categories. This process can be seen as strengthening the discrimination and preventing extremely similar novel classes from being recognized as base classes. It is equivalent to increasing the weight $(1-\lambda_c)$ of base empirical risk $\hat{\mathcal{R}}_{\mathrm{base}}$, the value of $\lambda_c$ will be smaller. According to Equation \ref{compare_eq}, $\mathcal{R}_{t}(\hat{f}_{\mathrm{CCL}})$ is positively correlated with $\lambda_c$, so when $\lambda_c$ decreases, the error bound of the model will be lower.
\end{proof}

We next analyze the proposed explainable data augmentation method. To this end, we introduce notation for the domains of the real data used for training and the data augmented by the counterfactual method, denoting them as $\mathcal{D}_r$ and $\mathcal{D}_e$, respectively. The empirical risk of training on real data and augmented data are denoted as $\hat{\mathcal{R}}_r$ and $\hat{\mathcal{R}}_e$, respectively. The process of detector optimizing data $\mathcal{D}_r$ and $\mathcal{D}_e$ can be regarded as:
\begin{equation}
    \hat{f}_{e} = \underset{f \in \mathcal{F}}{\arg \min} \hat{\mathcal{R}}_{(r+e)} (f),
\end{equation}
when given an hyper parameter $\lambda_g \in [0, 1)$, $ \hat{\mathcal{R}}_{(r+e)} (\hat{f}_e)$ can be a convex combination of risks $\hat{\mathcal{R}}_{r} (\hat{f}_e)$ and $\hat{\mathcal{R}}_{e} (\hat{f}_e)$:
\begin{equation}
    \hat{\mathcal{R}}_{(r+e)} (\hat{f}_e) = (1-\lambda_g) \hat{\mathcal{R}}_{r} (\hat{f}_e) + \lambda_g \hat{\mathcal{R}}_{e} (\hat{f}_e),
\end{equation}
then based on Lemma \ref{bound}, we have:
\begin{proposition}
    When using the counterfactual explanation method for data augmentation, for any $\delta > 0$, we have at least $1-\delta$ probability that,
    \begin{equation} \label{proposition_e}
        \begin{aligned}
            \mathcal{R}_{t}(\hat{f}_{e}) \leq & \; \hat{\mathcal{R}}_{(r+e)}(\hat{f}_{e}) + (1-\lambda_g) \gamma_{\mathcal{F}}(\mathcal{D}_{r}, \mathcal{D}_{e})\\
            & + 2 (1 - \lambda_g) \hat{\mathfrak{R}}_{r}(\mathcal{F})  + 3 (1 - \lambda_g) \sqrt{\frac{\ln (4 / \delta)}{2 N_{r}}}\\
            &  + 2 \lambda_g \hat{\mathfrak{R}}_{e}(\mathcal{F})  + 3 \lambda_g \sqrt{\frac{\ln (4 / \delta)}{2 N_{e}}} \\
            & + \sqrt{\frac{\ln (4 / \delta)}{2}\left ( \frac{(1-\lambda_g)^2}{N_r} + \frac{\lambda_g^2}{N_e} \right)},
        \end{aligned}
    \end{equation}
    where $N_r$ and $N_e$ are the sample sizes of $\mathcal{D}_r$ and $\mathcal{D}_e$ respectively.
\end{proposition}

\begin{theorem} \label{EXAU_sup}
    Assume that the augmented set $\mathbf{X}^{N_e} = \{\mathbf{x}_i^e\}_{i=1}^{N_e}$ generated by the counterfactual explanation method is not out-of-distribution samples, and the number of augmented samples $N_e=k_eN_r$, where $k_e > 1$. For the detector $\hat{f}_e$ trained with augmented set and $\hat{f}$ trained without augmented set, we have:
    \begin{equation} \label{}
        \sup \left( \mathcal{R}_{t}\left( \hat{f}_e \right ) \right ) < \sup \left( \mathcal{R}_{t}\left( \hat{f} \right ) \right ).
    \end{equation}
\end{theorem}

\begin{proof}
    Based on Lemma \ref{bound}, we can get the generalization error of the detector $\hat{f}$ trained without augmented data:
    \begin{equation} \label{}
        \begin{aligned}
            \mathcal{R}_{t} \left ( \hat{f} \right) & \le \mathcal{R}_{r}\left( \hat{f} \right ) + 2 \hat{\mathfrak{R}}_{r}(\mathcal{F}) + 4  \sqrt{\frac{\ln (4 / \delta)}{2 N_{r}}},
        \end{aligned}
    \end{equation}
    the empirical risk $\mathcal{R}_{r}\left( \hat{f} \right )$ can be minimized to arbitrary small, close to 0. The empirical risk $\hat{\mathcal{R}}_{(r+e)}(\hat{f}_{e})$ in Equation \ref{proposition_e} can be minimized tending to 0. $N_e = k_e N_r, k_e > 1, N_r < N_e$.  For the domain gap $\gamma_{\mathcal{F}}(\mathcal{D}_{r}, \mathcal{D}_{e})$, it has been assumed in Theorem \ref{EXAU_sup} that the generated samples are not out-of-distribution samples, and in our detector model training, we force the augmented sample features to be close to the prototype of the real category, so $\gamma_{\mathcal{F}}(\mathcal{D}_{r}, \mathcal{D}_{e})$ is very small and tends to 0. While,
    \begin{equation*} \label{}
        \hat{\mathfrak{R}}(\mathcal{F}):=\mathrm{E}_{\sigma} \sup _{f \in \mathcal{F}}\left[\frac{1}{N} \sum_{n=1}^{N} \sigma_{n} f\left(\mathbf{x}_{n}\right)\right],
    \end{equation*}
    because the sample distribution of $\mathcal{D}_{r}$ and $\mathcal{D}_{e}$ has a small gap when the model converges, the main difference between $\hat{\mathfrak{R}_r}(\mathcal{F})$ and $\hat{\mathfrak{R}_e}(\mathcal{F})$ depends on the sample size $N$, so $\hat{\mathfrak{R}_e} < \hat{\mathfrak{R}_r}$.
    Thus:
    \begin{equation} \label{proof2_1}
        \begin{aligned}
            \mathcal{R}_{t}(\hat{f}_{e})  \precsim & \; 2 (1 - \lambda_g) \hat{\mathfrak{R}}_{r}(\mathcal{F})  + 3 (1 - \lambda_g) \sqrt{\frac{\ln (4 / \delta)}{2 N_{r}}}\\
            &  + 2 \lambda_g \hat{\mathfrak{R}}_{e}(\mathcal{F})  + 3 \lambda_g \sqrt{\frac{\ln (4 / \delta)}{2 N_{e}}} \\
            & + \sqrt{\frac{\ln (4 / \delta)}{2}\left ( \frac{(1-\lambda_g)^2}{N_r} + \frac{\lambda_g^2}{N_e} \right)} \\
            < & \; 2 \hat{\mathfrak{R}}_{r}(\mathcal{F}) + 3 \sqrt{\frac{\ln (4 / \delta)}{2 N_{r}}} \\
            & + \sqrt{\frac{\ln (4 / \delta)}{2}\left ( \frac{(1-\lambda_g)^2 + \lambda_g^2 / k_e}{N_r} \right)}
            ,
        \end{aligned}
    \end{equation}
    for $(1-\lambda_g)^2 + \lambda_g^2 / k_e$ in the last term of Equation \ref{proof2_1}, where $0\le \lambda_g < 1$, have a minimum value when $\lambda_g=\frac{1}{1/k_e + 1}>1$. Therefore, when $\lambda_g=0$, this item obtains the maximum value $\sqrt{\frac{\ln (4 / \delta)}{2 N_r}}$. Thus:
    \begin{equation} \label{}
        \begin{aligned}
            \mathcal{R}_{t}(\hat{f}_{e}) < & \; 2 \hat{\mathfrak{R}}_{r}(\mathcal{F}) + 4 \sqrt{\frac{\ln (4 / \delta)}{2 N_{r}}},
        \end{aligned}
    \end{equation}
    therefore we can prove that:
    \begin{equation*}
        \sup \left( \mathcal{R}_{t}\left( \hat{f}_e \right ) \right ) < \sup \left( \mathcal{R}_{t}\left( \hat{f} \right ) \right ).
    \end{equation*}
\end{proof}

\section{More Experiments}\label{sec:more_experiment}

\begin{figure*}[!t]
    \centering
    \setlength{\abovecaptionskip}{0.cm}
    \includegraphics[width =\textwidth]{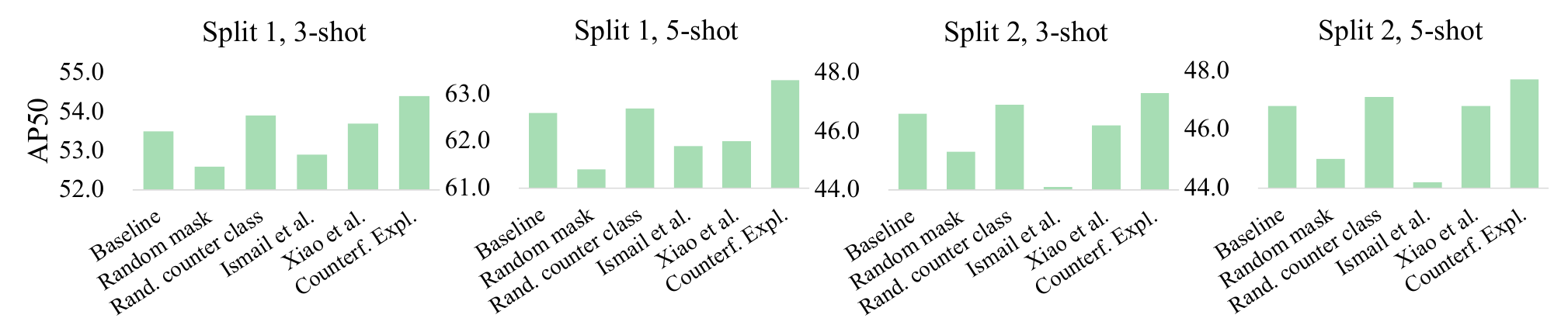}
    \caption{Comparison with different data augmentation strategies. Counterfactual data augmentation guided by a knowledge matrix achieves the best performance.}
    \label{counterfactual_experiment}
    \vspace{-16pt}
\end{figure*}

In this section, we provide a detailed report of additional experimental results, encompassing both ablation study and hyperparameter validation.

\subsection{Ablation for the Counterfactual Data Augmentation}

Compared with the common data augmentation tricks, our proposed counterfactual augmentation method is specifically tailored for FSOD tasks and intricately linked with CCL. It serves to prevent the overfitting of CCL during the process of distinguishing between base and novel categories. The counterfactual data augmentation module utilizes the knowledge matrix to identify and select similar counter classes and generate explainable counterfactual saliency maps to identify the current discriminative region of the model. Based on the counterfactual saliency maps, we use the masking mechanism to augment the few-shot data so that the mined sample features are closer to the decision boundary and jointly trained to improve the generalization of the model. The augmented samples are forwarded to CCL for contrastive learning alongside real features, aiming to broaden the feature distribution of novel categories.

We compare our method with random mask augmentation~\cite{zhong2020random}, which is a common augmentation trick for general object detection. We analyze the impact of integrating a knowledge matrix into counterfactual enhancement, focusing on its role in refining the discriminative boundaries between similar base and novel categories.
Fig.~\ref{counterfactual_experiment} shows the results, where we find that in the setting with multiple splits or shots, traditional random mask augmentation leads to performance degradation. This suggests that it is not well-suited for few-shot object detection tasks. When the knowledge matrix is not introduced, the performance improvement is marginal. Performance can be effectively improved by introducing a knowledge matrix, which enhances the discriminative space between specific base and novel classes, and by calibrating the decision boundary.

We also compared methods based on explanation-guided data augmentation (Ismail \textit{et al.}\cite{ismail2021improving} and Xiao \textit{et al.}\cite{xiao2023masked}). We find that the important region enhancement method (Ismail \textit{et al.}) degrades baseline performance across all settings, while the key region masking method (Xiao \textit{et al.}) provides only slight improvements in some cases and occasionally degrades the baseline model. Therefore, these approaches may not be well-suited for few-shot scenarios.

Thus, our proposed counterfactual data augmentation method effectively enhances the distinction between base and novel categories, which is related to the topic of this paper. Additionally, it can also be used to explain the model, making the training process transparent and more interpretable.

\subsection{Effect of Hyper-parameters}

To evaluate the effectiveness of our model, we validate some hyperparameter values and use $AP_{50}$ as the evaluation metric. We employ TFA++~\cite{sun2021fsce} as the baseline model and search to identify the optimal hyperparameters.

\textbf{Effect of temperature parameter $\tau$.} We investigate the effect of the temperature hyper-parameter in the contrastive learning loss on the model. Table \ref{tau} shows the results from 10-shot of PASCAL VOC split 1. We found that the results of the model are sensitive to the value of $\tau$, and the best result is achieved when the value of $\tau$ is 0.2. Thus, in this paper, we set $\tau$ to 0.2.

\begin{table}[h]
    \centering
    \caption{Ablation for contrastive hyper-parameter $\tau$, results of PASCAL VOC Split 1 10-shot. The parameter $\tau$ affects the contrast of the CCL model.}
    \vspace{-10pt}
    \resizebox{0.30 \textwidth}{!}{
        \begin{tabular}{ccccc}
            \toprule
            $\tau$  & 0.07 & 0.2 & 0.5 & 1 \\ \midrule
            $AP_{50}$ & 64.8  & $\mathbf{65.2}$     &  64.6   & 63.2   \\ \bottomrule
        \end{tabular}
    }
    \label{tau}
\end{table}

\textbf{Effect of balance parameter $\lambda$ in loss function.} To validate the impact of the weight $\lambda_1, \lambda_2$ and $\lambda_3$ of the loss functions on the model performance, we manually set different values for hyper-parameters.
Table \ref{lambda3} shows the effect of parameter $\lambda_3$ from 10-shot of PASCAL VOC split 1, where $\lambda_1$ and $ \lambda_2$ are all set to 1. We found that FSOD performance is insensitive to the value of the balance parameter $\lambda_3$, and in this paper, we set the balance parameter $\lambda$ to 1.
Table \ref{lambda1} shows the effect of parameters $\lambda_1$ and $\lambda_2$, where $\lambda_3$ is set to 1. We find that when setting both $\lambda_1$ and $\lambda_2$ to 1, the model achieves the highest performance. Thus, in this paper, we set $\lambda_1$ and $\lambda_2$ to 1.

\begin{table}[h]
    \centering
    \caption{Ablation for balance parameter $\lambda_3$, results of PASCAL VOC Split 1 10-shot, where $\lambda_1$ and $ \lambda_2$ are all set to 1. The parameter $\lambda_3$ controls the weights of the CCL loss function.}
    \vspace{-10pt}
    \begin{tabular}{ccccc}
        \toprule
        $\lambda_3$ & 0.5 & 1 & 1.5 & 2 \\ \midrule
        $AP_{50}$   & 64.6    & $\mathbf{65.2}$   & 64.5    & 64.7   \\ \bottomrule
    \end{tabular}
    \label{lambda3}
\end{table}
\begin{table}[h]
    \centering
    \caption{Ablation for balance parameters $\lambda_1$ and $\lambda_2$, results of PASCAL VOC Split 1 10-shot, where $\lambda_3$ is set to 1.}
    \vspace{-10pt}
    \begin{tabular}{ccccc}
        \toprule
        \multicolumn{2}{c}{\multirow{2}{*}{Hyperparameters}} & \multicolumn{3}{c}{$\lambda_1$} \\ \cmidrule(l){3-5}
        \multicolumn{2}{c}{}                                 & 0.5     & 1       & 2       \\ \midrule
        \multirow{3}{*}{$\lambda_2$}            & 0.5            & 63.6    & 64.4    & 63.7    \\
        & 1              & 63.5    & \textbf{65.2}    & 64.1    \\
        & 2              & 64.1    & 64.3    & 63.8    \\ \bottomrule
    \end{tabular}
    \label{lambda1}
\end{table}

\section{More Visualization}

\subsection{Visualization of Prototype}

We visualize the prototype features at the end of detector training. We use TSNE to reduce prototype features to 2 dimensions and normalize the feature to observe cosine similarity between different categories. In Fig. \ref{prototype}, we visualize the distribution of prototype features with and without knowledge matrix, respectively, on the PASCAL VOC dataset split 2 under 10 shots setup. We observe that some similar categories, such as horse and cow (both are novel categories), sofa (novel category), and dining table (base category), have very small angles of prototype cluster center features. When the detector is equipped with the knowledge matrix, we can expand the angle between features of similar prototype cluster centers to learn more discriminative features.

\begin{figure}[h]
    \centering
    \setlength{\abovecaptionskip}{0.cm}
    \includegraphics[width = 0.47 \textwidth]{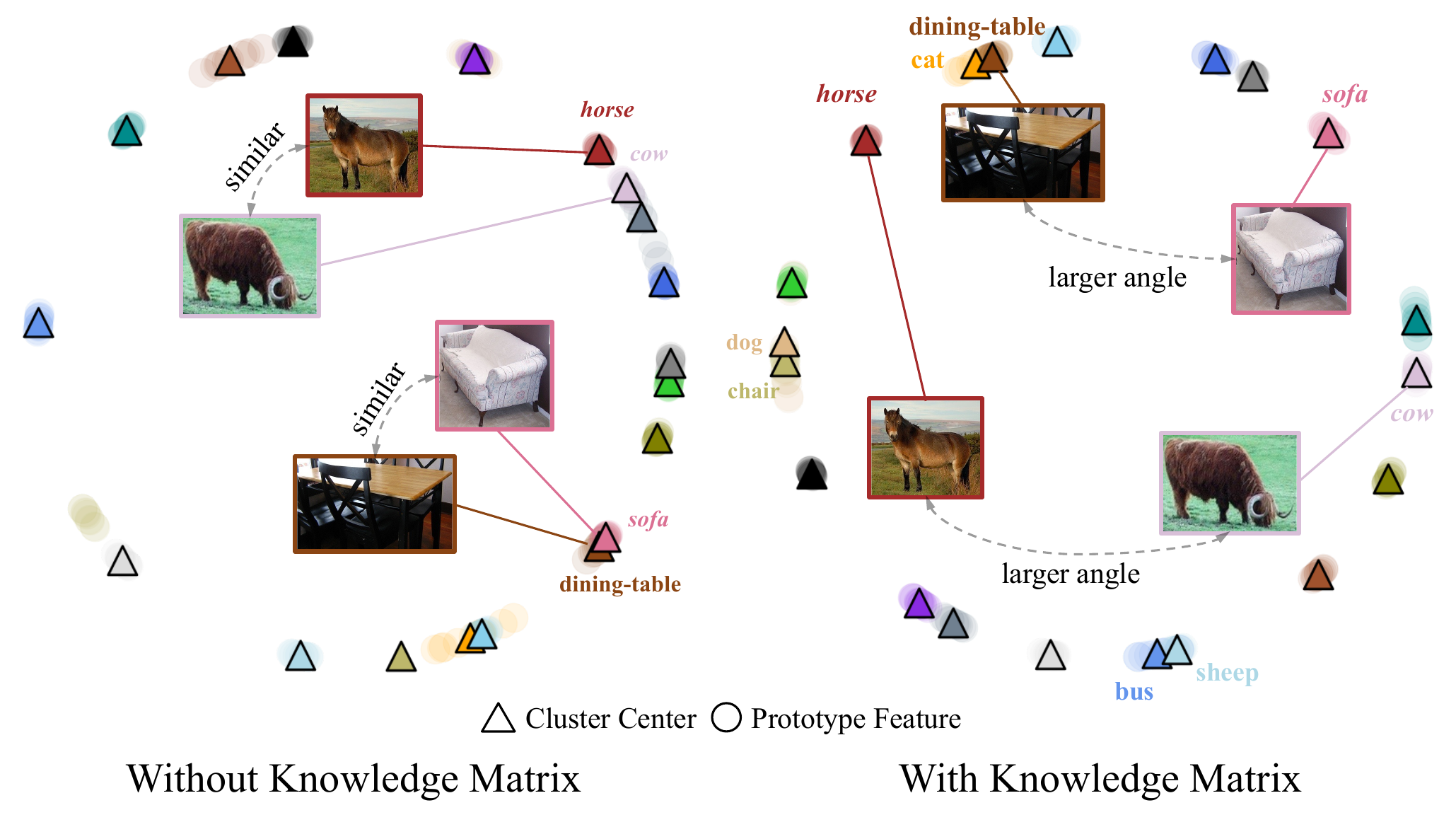}
    \caption{Visualization of prototype cluster centers on PASCAL VOC dataset split 2 under the 10-shot setup. The novel categories are presented in italic font.}
    \label{prototype}
    \vspace{-16pt}
\end{figure}

\section{Discussion and Limitation} \label{limit}

In Fig.~\ref{discuss}, we present several visualizations that demonstrate our method's failure cases, where certain objects are incorrectly recognized as belonging to different categories. Through our analysis, we identified that the smoke and color in the first image, the angle in the second image, and the pose in the third image were factors that caused the model's reasoning to be misled. These are special attributes that the model does not learn but do exist. We consider the following points for improvement: (i) Our method only focuses on overall visual attributes rather than individual visual attributes, but we believe that in future work, we should eliminate the biased attributes to obtain better performance. (ii) Besides, we only consider the relationship between objects and objects, not the relationship between objects and background, which may lead to some biases. Future work should also focus on intervening with context-induced recognition biases.

\begin{figure}[h]
    \centering
    \setlength{\abovecaptionskip}{0.cm}
    \includegraphics[width =  0.48 \textwidth]{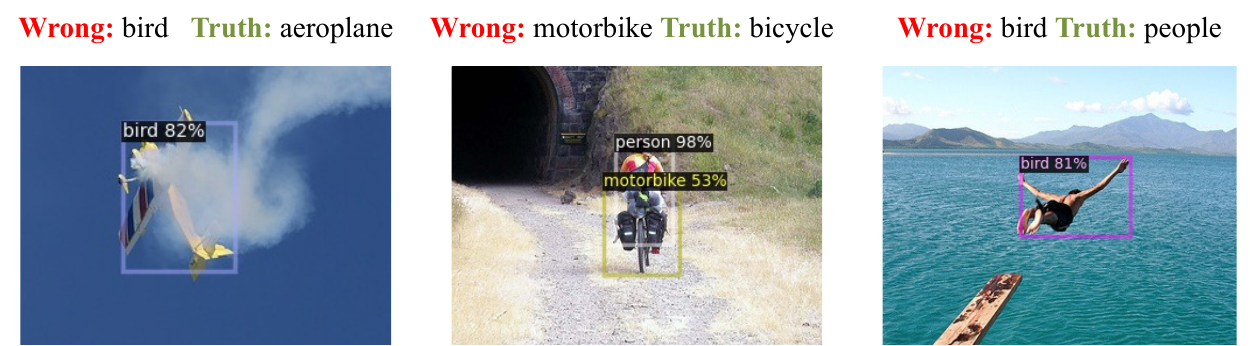}
    \caption{Visualization of some failed cases of our model.}
    \label{discuss}
\end{figure}

\end{appendices}

\end{document}